\documentclass[twocolumn, 10pt]{article}
\usepackage[superscript]{cite}
\usepackage[colorlinks=true, allcolors=blue]{hyperref}
\usepackage{amsmath}
\usepackage{cases}
\usepackage{mathtools}
\usepackage{times}
\usepackage{graphicx}
\usepackage{graphics}
\usepackage{mathptmx}
\usepackage{times}
\usepackage{amssymb}
\usepackage{dcolumn}
\usepackage{bm}
\usepackage{units}
\usepackage{siunitx}
\usepackage{caption}
\usepackage{geometry}
\usepackage{titlesec}
\titleformat*{\section}{\fontsize{12}{12}\selectfont\bfseries}
\titlespacing{\section}{0pc}{1.5ex}{0pc}
\titleformat*{\subsection}{\fontsize{11}{11}\selectfont\bfseries}
\titlespacing{\subsection}{0pc}{1.5ex}{0pc}
\usepackage[switch, columnwise]{lineno}
\usepackage{titling}
\date{\vspace{-5ex}}
\usepackage{microtype}
\usepackage{subfigure}
\usepackage{booktabs} 
\usepackage{color,soul}
\allowdisplaybreaks
\usepackage[noend]{algorithmic}
\usepackage{amsthm}
\newtheorem{thm}{Theorem}[section]

\newtheorem{defn}{Definition}[section]

\usepackage{xpatch}
\makeatletter
\newenvironment{proof*}[1][\proofname]{\par
	\pushQED{\qed}%
	\normalfont \partopsep=\z@skip \topsep=\z@skip
	\trivlist
	\item[\hskip\labelsep
	\itshape
	#1\@addpunct{.}]\ignorespaces
}{%
	\popQED\endtrivlist\@endpefalse
}
\makeatother
\newcommand{\beginsupplement}{%
	\setcounter{table}{0}
	\renewcommand{\thetable}{S\arabic{table}}%
	\setcounter{figure}{0}
	\renewcommand{\thefigure}{S\arabic{figure}}%
	\setcounter{section}{0}
	\renewcommand{\thesection}{S\arabic{section}}%
	\setcounter{equation}{0}
	\renewcommand{\theequation}{S\arabic{equation}}%
}
\usepackage{amssymb}

\usepackage{float}
\newfloat{Algorithm}{t}{lop}
\let\OLDthebibliography\thebibliography
\renewcommand\thebibliography[1]{
	\OLDthebibliography{#1}
	\setlength{\parskip}{0pt}
	\setlength{\itemsep}{0pt plus 0.3ex}
}

\usepackage{fancyhdr}
\fancypagestyle{plain}{
	\fancyhf{} 
	\fancyhead[R]{\textbf{\thepage}} 
	
	}

\title{\Huge \bfseries Optimality of short-term synaptic plasticity in modelling certain dynamic environments}
\author{Timoleon Moraitis\textsuperscript{*}, Abu Sebastian, and Evangelos Eleftheriou \\ IBM Research
	-- Zurich, 8803 R\"{u}schlikon, Switzerland \\\textsuperscript{*}Present address: Huawei -- Zurich Research Center, 8050 Zurich, Switzerland \\ timosmoraitis@gmail.com}

\begin{document}
\frenchspacing \maketitle
\textbf{
 Biological neurons and their \textit{in-silico} emulations for neuromorphic artificial intelligence (AI) use extraordinarily energy-efficient mechanisms, such as spike-based communication and local synaptic plasticity. It remains unclear whether these neuronal mechanisms only offer efficiency or also underlie the superiority of biological intelligence. Here, we prove rigorously that, indeed, the Bayes-optimal prediction and inference of randomly but continuously transforming environments, a common natural setting, relies on short-term spike-timing-dependent plasticity, a hallmark of biological synapses. Further, this dynamic Bayesian inference through plasticity enables circuits of the cerebral cortex in simulations to recognize previously unseen, highly distorted dynamic stimuli. Strikingly, this also introduces a biologically-modelled AI, the first to overcome multiple limitations of deep learning and outperform artificial neural networks in a visual task. The cortical-like network is spiking and event-based, trained only with unsupervised and local plasticity, on a small, narrow, and static training dataset, but achieves recognition of unseen, transformed, and dynamic data better than deep neural networks with continuous activations, trained with supervised backpropagation on the transforming data. These results link short-term plasticity to high-level cortical function, suggest optimality of natural intelligence for natural environments, and repurpose neuromorphic AI from mere efficiency to computational supremacy altogether.
}

\pagestyle{plain}

Fundamental operational principles of neural networks in the central nervous system (CNS) and their computational models \cite{maass1997NN,ponulak2011ANE,gruning2014ESANN} are not part of the functionality of even the most successful artificial neural network (ANN) algorithms. A key aspect of biological neurons is their
communication by use of action potentials, i.e. stereotypical voltage spikes, which carry information in their timing. In addition, to process and learn
from this timing information, synapses, i.e. connections between neurons, are equipped with plasticity mechanisms, which dynamically change the
synaptic efficacy, i.e. strength, depending on the timing of postsynaptic and/or presynaptic spikes. For example, spike-timing-dependent plasticity
(STDP) is a Hebbian, i.e. associative, plasticity, where pairs of pre- and post-synaptic spikes induce changes to the synaptic strength dependent
on the order and latency of the spike pair \cite{markram1997Science,bi1998JN,song2000NatureNeuroscience}. Plastic changes can be long-term or
short-term. Short-term plasticity (STP) \cite{zucker1989ARN,tsodyks1997PNAS,chamberlain2008NP} has been shown for instance to act as a
signal-filtering mechanism \cite{rosenbaum2012PLoS}, to focus learning on selective timescales of the input by interacting with STDP
\cite{moraitis2018IEEE,moraitis2018IJCNN}, to enable long short-term memory \cite{salaj2020spike}, and to act as a tempering mechanism for generative models \cite{leng2018spiking}. Such biophysical mechanisms have been emulated by the physics of electronics to implement
neuromorphic computing systems. Silicon spiking neurons, synapses, and neuromorphic processors are extremely power-efficient \cite{mead1990IEEEProc,merolla2014Science,qiao2015FN,indiveri2018NMEH,Davies2018Micro} and have shown particular promise in tasks such as interfacing with biological
neurons, including chips learning to interpret brain activity \cite{boi2016FN,serb2020SR}.

But the human brain exceeds current ANNs by far not only
in terms of efficiency, but also in its end performance in demanding tasks. Identifying the source of this computational advantage is an important
goal for the neurosciences and could also lead to better AI. Nevertheless, many properties of spiking neurons and synapses in the brain, particularly those mechanisms that are time-based, such as spiking activations themselves, but also STDP, STP, time-varying postsynaptic potentials etc. are studied in a separate class of computational models, i.e. spiking neural networks (SNNs) \cite{maass1997NN}, and are only sparsely and loosely embedded in state-of-the-art AI.

Therefore, it is reasonable to speculate that the brain's spike-based computational mechanisms may also be part of the reason for its performance advantage. Indeed there is evidence that the brain's powerful computations emerge from simple neural circuits with spike-based plasticity. For example, the brain is
well-documented to maintain statistically optimal internal models of the environment \cite{wolpert1995Science,kording2004Nature,ma2006NatureNeuroscience,blaisdell2006Science,griffiths2006PS,doya2007MIT,fiser2010TCS,berkes2011Science,bastos2012Neuron}. Spiking neurons can
give rise to such Bayesian models, whereas STDP can form and update them to account for new observations \cite{nessler2013PLoS}. The structure of
such SNNs is considered a canonical motif, i.e. pattern, in the brain's neocortex \cite{douglas2004ARN}. Their function through STDP is akin to on-line clustering or
expectation-maximization \cite{nessler2013PLoS}, and their models can be applied to common benchmarking tasks of machine learning (ML)
\cite{diehl2015FCN}. Such theoretical modelling and simulation results that link low-level biophysics like synaptic plasticity with high-level cortical function, and that embed biological circuitry in a rigorous framework like Bayesian inference are foundational but also very scarce. This article's first aim is to generate one such insight, revealing a link of STP to cortical Bayesian computation.
\begin{figure*}[h!]
	\centering
	\includegraphics[width = 155mm]{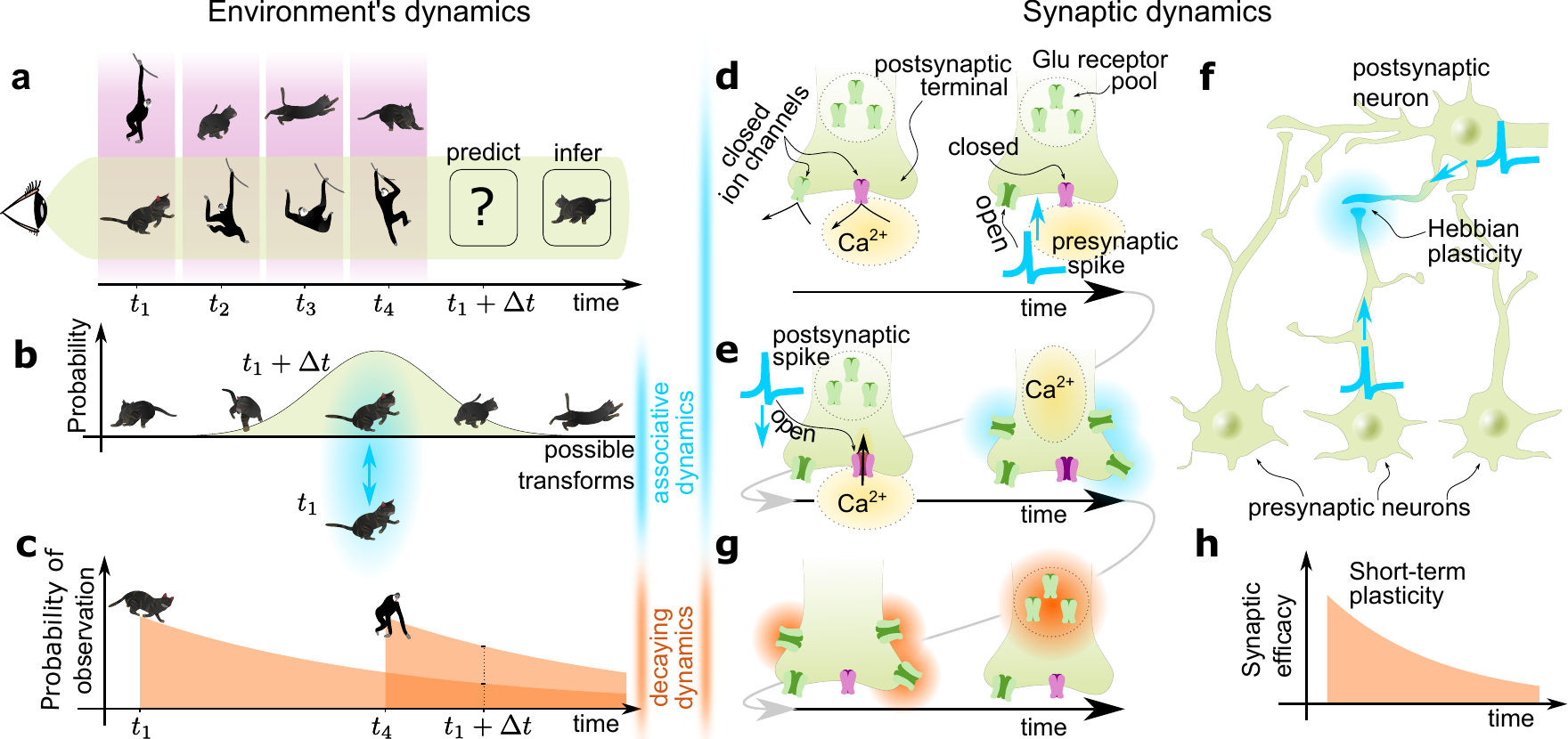}
	\caption{\textbf{Analogous dynamics between natural environments and synapses.} \textbf{a}, An environment (here, a visual one, depicted by the
		magenta field) contains objects (e.g. cat and gibbon) that transform and move randomly but continuously. An observer (left, eye) observes (green
		field) one object at a time.  Objects are replaced at random time instances (e.g. $t_2$). The observer aims to maintain a model of the environment that
		can predict future observations (question mark) and infer their type (frame labelled ``infer"). \textbf{b}, An object's observation ($t_1$, bottom
		cat) implies that the probability distribution of future ($t_1+\Delta t$) transforms of this object (horizontal axis) is centred around its last
		observation. \textbf{c}, An observed object ($t_1$, cat, and $t_4$, gibbon) is likely to be observed again soon after, but less so as time passes
		(decaying curves). \textbf{d}, Left: A synapse's postsynaptic terminal has ion channels that are resting closed, preventing ions, e.g. $Ca^{2+}$,
		from entering the postsynaptic neuron. Right: A presynaptic spike releases neurotransmitters, e.g. glutamate (Glu) that attach to ion channel
		receptors and open some channels, but do not suffice to open $Ca^{2+}$ channels, which are postsynaptic-voltage-gated. The synapse includes a pool of
		additional inactive Glu receptors. \textbf{e}, In the event of an immediately subsequent postsynaptic spike, channels open, allowing $Ca^{2+}$ to enter (left), which
		has a potentiating effect on the activity and the number of Glu receptors on the membrane (right, blue highlight). This increases the efficacy of the
		synapse, and is observed as STDP. \textbf{f}, This establishes a Hebbian link associating only the causally activated pre- and
		post-synaptic neurons. \textbf{g}, As $Ca^{2+}$'s effect on the Glu receptor pool decays with time (orange highlight), then, \textbf{h}, the efficacy
		of the synapse also decays towards its original resting weight, and is observed as STP, so the overall effect amounts to ST-STDP.
		Synaptic dynamics are analogous to the environment's dynamics (associative dynamics: blue elements, compare d-f vs b; decaying dynamics: orange elements, compare g-h vs c). We show analytically that computations performed by such synapses provide the optimal solution to task a of the observer.}
	\label{fig:data_model}
\end{figure*}

The second aim is to demonstrate that models with biological or neuromorphic mechanisms from SNNs can be superior to AI models that do not use these mechanisms, and could thus explain some of the remaining uniqueness of natural intelligence. Evidence suggests that this may be possible in certain scenarios. Proposals for functionality unique to SNNs have been presented \cite{maass1997NN,moraitis2018IJCNN,leng2018spiking}, including models
\cite{poirazi2003Neuron} with very recent experimental confirmation \cite{gidon2020Science} that individual spiking neurons in the primate brain,
even a single compartment thereof, can compute functions that were traditionally considered to require multiple layers of conventional artificial
neural networks (ANNs). Nevertheless, in concrete terms, thus far the accuracy that is achievable by brain-inspired SNNs in tasks of machine
intelligence has trailed that of ANNs \cite{bellec2018NIPS, pfeiffer2018FN, rajendran2019IEEESPMag, wozniak2020NatureMI}, and there is little
theoretical understanding of SNN-native properties in an ML context \cite{nessler2013PLoS, bengio2015ArXiv}. As a result, not only do
the benefits in AI accuracy from SNN mechanisms remain hypothetical, but it is also unclear if these mechanisms are responsible for any of the
brain's superiority to AI.

In this article, we show that, in a common problem setting -- namely, prediction and inference of environments with random
but continuous dynamics -- not only do biological computational mechanisms of SNNs and of cortical models implement a solution, they are indeed the theoretically optimal
model, and can outperform deep-learning-trained ANNs that lack these neuromorphic mechanisms.
As part of this, we reveal a previously undescribed role of STP. We also show cortical circuit motifs can perform Bayesian inference on non-stationary input distributions and can recognize previously unseen dynamically transforming stimuli, under severe distortions.

In the next sections, first, we provide an intuition as to why STP in the form of short-term STDP (ST-STDP) could turn out to be useful in modelling dynamic environments. Subsequently, the theoretically optimal solution for this task is derived, without any reference to synaptic plasticity. The result is a solution for Bayesian inference, with an equivalent interpretation as a ``neural elastic clustering“ algorithm. The theoretical sections conclude by showing the equivalence between this elastic clustering and a cortical microcircuit model with ST-STDP at its synapses, thus confirming the original intuition. Next, in spiking simulations we demonstrate the operation of the cortical-like model with STP, while it is stimulated by transforming images. Finally, we compare its performance to that of ANNs.
\section*{ST-STDP reflects the dynamics of natural environments}
We model the environment as a set of objects, each belonging to one of $K$ classes. Each object can transform in a random or unknown, but time-continuous manner. To predict future transformations and infer their class, an observer observes the environment. Only one object is observed at
each time instance, and the observer switches to different objects at random instances (Fig. \ref{fig:data_model}\textbf{a}). In the absence of
additional prior knowledge about an object's dynamics, a future observation of the potentially transformed object is distributed around its most
recent observation (Fig. \ref{fig:data_model}\textbf{b}; also see Methods section and Supplementary Information, section \ref{sec:pdf}). In addition, time and space continuity imply that an observed
object is likely to be observed again immediately after, but as objects move in space relative to the observer, this likelihood decays to zero with
time, according to a function $f(t)$ (Fig. \ref{fig:data_model}\textbf{c} and Supplementary Information, section \ref{sec:formulation_sup}). The model is rather general, as its main assumption is merely that the environment is continuous with random or unknown dynamics.
Furthermore, there is no assumption of a specific sensory domain, i.e. transforming objects may range from visual patterns in a scene-understanding
task, to proprioceptive representations of bodily states in motor control, to formations of technical indicators in financial forecasting, etc.
\begin{figure*}[h!]
	\centering
	\includegraphics[width = 155mm]{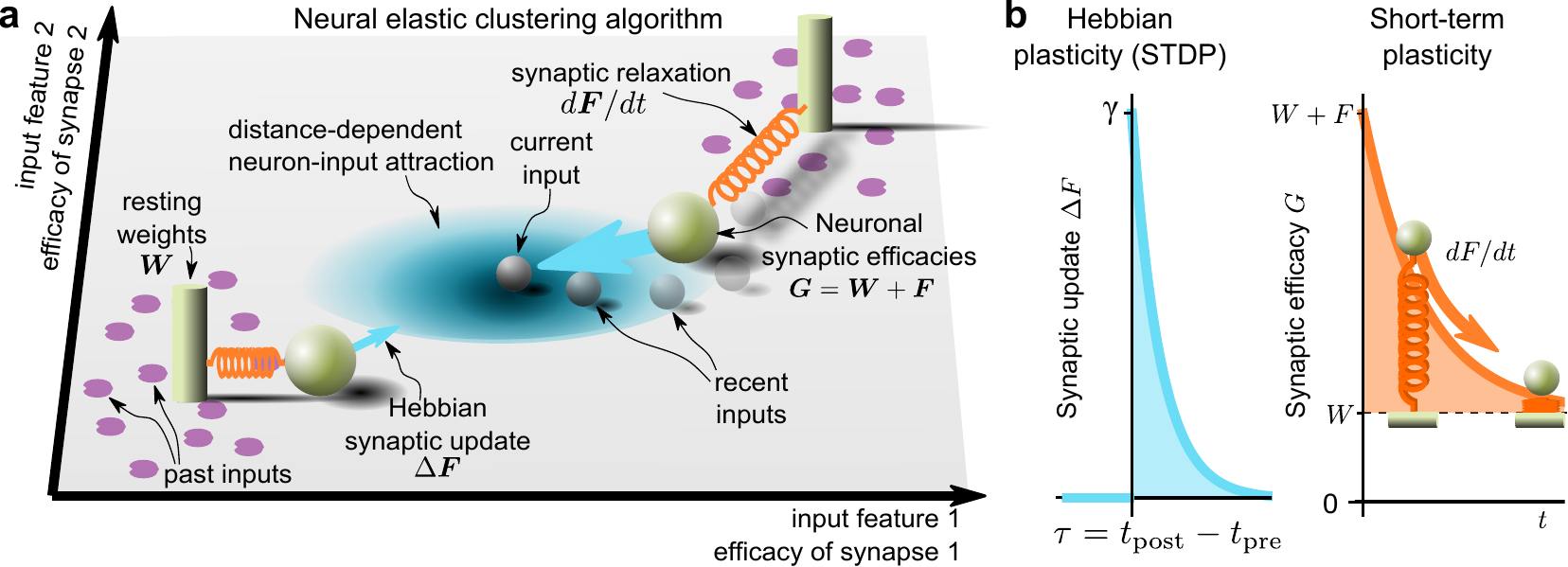}
	\caption{\textbf{Neural elastic clustering and underlying synaptic plasticity.} \textbf{a}, Random continuous environments are optimally modelled by a neurally-realizable clustering algorithm. Each cluster centroid (green spheres) corresponds to a
		neuron, here with two synapses, and is positioned in space according to its vector of synaptic efficacies $\boldsymbol{G}$. Recent inputs
		(transparent grey spheres) have pulled each neuron-centroid away from its fixed resting weight array $\boldsymbol{W}$ (vertical poles), by a short-term
		component $\boldsymbol{F}$ (orange spring). The neuron is constantly pulled back towards its resting state by a certain relaxation dynamics
		($d\boldsymbol{F}/dt$, spring). A new input (grey sphere, centre) pulls each neuron by a $\Delta\boldsymbol{F}$ (blue arrow) that depends on the
		proximity between the input and the neuron (blue field), reflected by the neuron's activation. \textbf{b}, The clustering algorithm is implemented by ST-STDP at the neurons' synapses. The proximity-dependent updates of a neuron-centroid are determined by Hebbian plasticity. In case of spike-based rate-coded inputs, updates depend (blue curve) on the time difference $\tau$ between pre- and post-synaptic spikes. The relaxation of the neuron to its resting position with time is realized by short-term plasticity (orange curve). $G$, $W$, and $F$ denote a representative
		synapse's efficacy, resting weight, and dynamic term.} \label{fig:clustering}
\end{figure*}

It can be seen that these common dynamics of natural environments bear significant analogies to the dynamics involved in short-term STDP (ST-STDP), i.e. a particular type of STP. That is a
realization of STDP with short-term effects \cite{brenowitz2005Neuron,cassenaer2007Nature,erickson2010JCN}, and has been proposed as a mechanism
for working memory in biological neural networks \cite{szatmary2010PLoS,fiebig2017JN}. It is observed in the early phase of long-term potentiation.
A series of paired pre- followed by post-synaptic spikes can lead to a persistent increase in synaptic efficacy \cite{frey1993Science,huang1998CB,baltaci2019NR}. However, when fewer stimuli are given, the induced change is short-term. This short-term increase in synaptic efficacy is mediated by
a series of biophysical events (see Fig. \ref{fig:data_model}, \textbf{d}-\textbf{h}). A presynaptic action potential releases excitatory neurotransmitters such
as glutamate (Glu), which attaches to receptors of ion channels on the postsynaptic terminal, thus opening some of the channels (Fig.
\ref{fig:data_model}\textbf{d}). However, calcium channels with N-methyl-D-aspartate (NMDA) receptors are voltage-gated \cite{schiller2000Nature},
i.e. they only open if the postsynaptic voltage increases while Glu is attached. A postsynaptic spike occurring shortly after the presynaptic one
achieves this, so that calcium does enter the postsynaptic cell. Calcium interacts with protein kinases that increase the activity and the number of
Glu receptors on the postsynaptic membrane (Fig. \ref{fig:data_model}\textbf{e}). This is observed as a Hebbian potentiation (Fig.
\ref{fig:data_model}\textbf{f}). The effect is short-term, as Glu receptors return to their resting state. Time constants for decay rates of short-term efficacy changes in cortex can be as short as tens of milliseconds \cite{wang2006heterogeneity}, and as long as 1.6 minutes \cite{erickson2010JCN}.

Note that the associative memories formed through the Hebbian aspect of ST-STDP resemble the also associative dynamics in
the environment, i.e. the association of the latest form of an object with its most likely future form (Fig. \ref{fig:data_model}, blue elements). In a second analogy, the transience of
this potentiation could be used in computations to reflect the transiently decreasing probability of observing the same object again (Fig. \ref{fig:data_model}, orange elements). Indeed, we performed a formal
analysis that concluded that the Bayesian generative model of the future observations given the past ones requires for its optimality a mechanism
equivalent to ST-STDP.

\section*{The neural elastic clustering algorithm}
Our formal derivation begins with obtaining, from the previous section's assumptions about the environment, a history-dependent description of the probability distribution of future observations. This is followed by defining a parameterized generative model that will be fitted to this distribution data. Ultimately, we find the maximum-likelihood optimal parameters of the model given the past observations, by analytically minimizing the divergence of the model from the data. We provide the full derivation in the Supplementary Information, and a summary in the Methods section.
We show that the derived generative model is equivalent to a typical cortical pattern of neural circuitry, with ST-STDP optimizing it adaptively (Supplementary Information, section \ref{sec:neuro_exp}), and is
geometrically interpretable as a novel and intuitive clustering algorithm based on centroids with elastic positions (Fig.
\ref{fig:clustering}\textbf{a}, and Supplementary Information, section \ref{sec:clustering_sup}). Specifically, if input samples $\boldsymbol{X}_t$
have $n$ features, then each of the $K$ classes is assigned a centroid, represented by a neuron with $n$ input synapses. Thus, at each time instance
$t$, the $k$-th neuron-centroid is lying in the $n$-dimensional space at a position determined by its vector of synaptic efficacies
$\boldsymbol{G}^{(k)}_t$. A function of the sample's proximity $u_t^{(k)}(\boldsymbol{X}_t)$ to this neuron determines the estimated likelihood
$q^{(k)}(\boldsymbol{X}_t)$ of this sample conditional on it belonging to the $k$-th class $C^{(k)}$. Assuming that neurons enforce normalization of
their synaptic efficacy arrays and receive normalized inputs, then the neuron defines its proximity to the input as the total weighted input, i.e.
the cosine similarity $u^{(k)}_t=\frac{ \boldsymbol{G}^{(k)}_t\cdot\boldsymbol{X}_t}{|| \boldsymbol{G}^{(k)}_t||\cdot||\boldsymbol{X}_t||}$. An
additional scalar parameter $G^{(0k)}_{t}$, represented in the neuron's bias, accounts for the prior belief about this class's probability. The bias
parameterizes the neuron's activation, which associates the sample with the $k$-th class $C^{(k)}$ as their joint probability. Ultimately, the
activation of the $k$-th neuron-centroid relative to the other neurons, e.g. the argmax function as in K-means clustering, or the soft-max function,
is the inference of the posterior probability $Q_t^{(k)}$ that the input belongs to $C^{(k)}$. Similarly to\cite{nessler2013PLoS} , we show that if
the chosen relationship is soft-max and the neurons are linear, then the network's output $Q_t^{(k)}$ is precisely the Bayesian inference given the
present input and parameters (Supplementary Information, sections \ref{sec:optimal_mean} and \ref{sec:neuro_exp}). The Bayesian generative model is
the mixture, i.e. the weighted sum, of the $K$ neurons' likelihood functions $q^{(k)}(\boldsymbol{X}_t)$, which in this case are exponential, and is
fully parameterized by the synaptic efficacies and neuronal biases.

The optimization rule of this model's parameters, for the spatiotemporal environments discussed in this article, is the elastic clustering algorithm.
Specifically, we show (Supplementary Information, section \ref{sec:optimal_givenpast}) that, given the past observations, the model's optimal synaptic efficacies, i.e. those that maximize the likelihood of the observations, comprise a fixed vector of resting weights $ \boldsymbol{W}^{(k)}$ and a dynamic term
$\boldsymbol{F}^{(k)}_t$ with a short-term memory, such that
\begin{equation}
	\boldsymbol{G}^{(k)}_t=
	\boldsymbol{F}^{(k)}_t
	+\boldsymbol{W}^{(k)}.
\end{equation}
The neuron-centroid initially lies at the resting position $\boldsymbol{W}^{(k)}$, which is found through conventional techniques such as expectation
maximization. The synaptic efficacies remain optimal if at every posterior inference result $Q^{(k)}_t$ their dynamic term $\boldsymbol{F}^{(k)}_t$
is incremented by
\begin{equation}
	\Delta\boldsymbol{F}^{(k)}_t=
	\gamma\boldsymbol{X}_t
	Q^{(k)}_t, \label{eq:F}
\end{equation}
where $\gamma$ is a positive constant, and if in addition, $\boldsymbol{F}^{(k)}_t$ subsequently decays continuously according to the dynamics
$f(t)$ of the environment (Fig. \ref{fig:data_model}\textbf{c}), such that $ \boldsymbol{G}^{(k)}_t$ relaxes towards the fixed resting point $
\boldsymbol{W}^{(k)}$. If $f(t)$ is exponential with a rate $\lambda$, then
\begin{equation}\frac{d \boldsymbol{F}^{(k)}_t}{dt}=-\lambda \boldsymbol{F}^{(k)}_t.\label{eq:lambda}\end{equation}
The latter two equations describe the attraction of the centroid by the input, as well as its elasticity (Fig. \ref{fig:clustering}\textbf{a}). The
optimal bias $G^{(0k)}_{t}$ is also derived as consisting of a fixed term and a dynamic term. This dynamic term too is incremented at every inference
step and relaxes with $f(t)$.

\section*{Equivalence to biological mechanisms}
It is remarkable that the solution (Eq. 1-3) that emerges from machine-learning-theoretic derivations in fact is an ST-STDP rule for the synaptic efficacies (see Supplementary Information, section \ref{sec:neuro_exp}), as we explain next. In
this direct equivalence, the equations dictate changes to the synaptic efficacy vector and thus describe a synaptic plasticity rule. Specifically, the rule dictates an attraction of the efficacy by the (presynaptic) input but also proportionally to the (postsynaptic) output $Q^{(k)}_t$ (Eq. \ref{eq:F}), and it is therefore a biologically plausible Hebbian plasticity rule. We also show that if input variables are encoded as binary Poisson processes, e.g. in the context of rate-coding -- a
principal strategy of input encoding in SNNs and of communication between biological neurons \cite{hubel1959JP,nessler2013PLoS,gerstner2014Cambridge,brette2015FSN,diehl2015FCN} -- a timing
dependence in the synaptic update rule emerges as part of the Bayesian algorithm (Supplementary Information, section \ref{sec:stoch}). The resulting
solution then specifies the Hebbian rule further as a spike-timing-dependent rule, i.e. STDP, dependent on the time interval $\tau$ between the pre- and post-synaptic spikes. Finally, the relaxation dynamics
$f(t)$ of the synaptic efficacies towards the resting weights $\boldsymbol{W}^{(k)}$, e.g. as in Eq. \ref{eq:lambda}, indicate that the plasticity is short-term, i.e. STP,
thereby the equations fully describing ST-STDP (see Fig. \ref{fig:clustering}\textbf{b}). This shows that, for these dynamic environments, the optimal model relies on the presence of a specific plasticity mechanism, revealing a possible direct link between synaptic dynamics and internal models of the environment in the brain. Importantly, this synaptic mechanism has not been associated with conventional ML or neurons with continuous activations before. Eq. 1-3 and our overall theoretical foundation integrate STP in those frameworks. Note that the dynamics of the bias parameter $G^{(0k)}_{t}$ capture the
evolution of the neuron's intrinsic excitability, a type of use-dependent plasticity that has been observed in numerous experiments
\cite{abbott2000NatureNeuroscience,daoudal2003LM,cudmore2004JN,turrigiano2011ARN}.

\begin{figure*}[h!]
	\centering
	\includegraphics[width = 155mm]{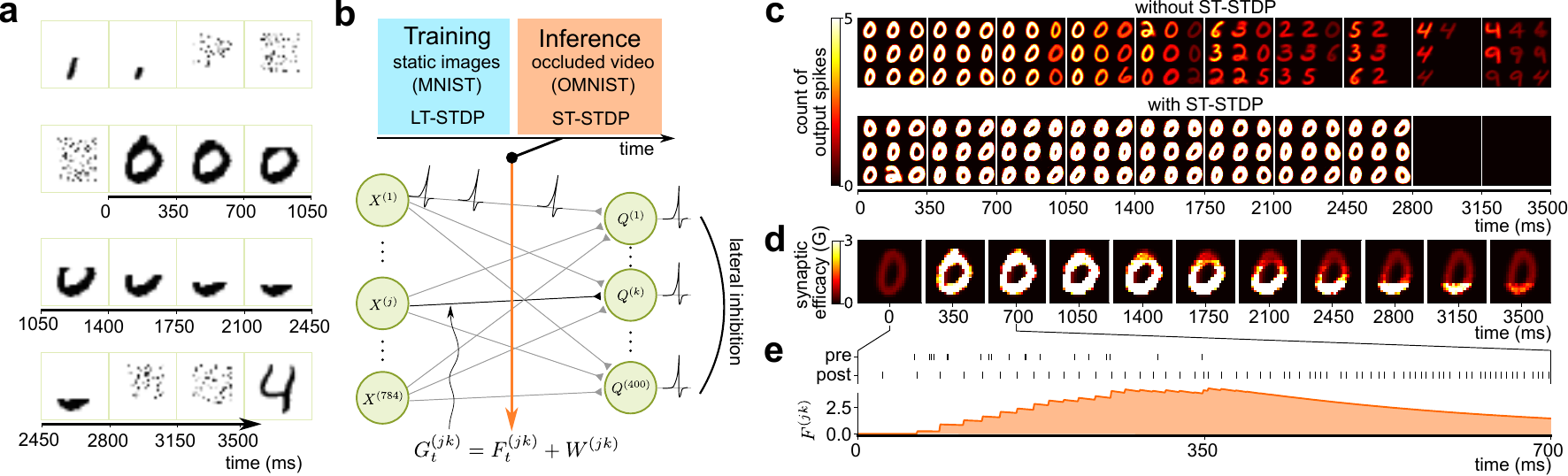}
	\caption{\textbf{SNN with ST-STDP during stimulation by transforming visual inputs.} \textbf{a}, A representative sequence of frames from the OMNIST set. The subsequence labelled as \unit[0-3500]{ms} is the input used in the rest of the figure panels. \textbf{b}, Schematic illustration of
		the neural network architecture used for the classification task. Each of the 784 input neurons, $X^{(j)}$, corresponds to a pixel and is connected to each of the 400 output
		neurons, $Q^{(k)}$, by a synapse with efficacy, $G^{(jk)}$. A lateral inhibition mechanism is implemented, whereby each output neuron's spike inhibits the
		rest of the neurons. During training with static unoccluded images (MNIST), standard, i.e. long-term, STDP
		(LT-STDP) is employed to obtain the fixed weights, $\boldsymbol{W}$, in an unsupervised manner. During inference on occluded video (OMNIST), synaptic
		plasticity switches to ST-STDP. \textbf{c}, Comparison between the SNN with and without ST-STDP in terms of the activated output neurons. The 9 most active output neurons over each period of \unit[350]{ms} are shown. Each neuron is presented as a 2D-map of its fixed weights $\boldsymbol{W}$, showing the digit pattern it has been trained to respond to. Colour-coding corresponds to the neuron's count of recorded output spikes over each period. It can be seen that, only when ST-STDP is enabled, a recognized input digit "zero" continues being recognized as such even when it is highly occluded (\unit[350-2800]{ms}, cf. a), and not when it is replaced by noise (\unit[2800-3500]{ms}). \textbf{d}, Instantaneous snapshots of the synaptic efficacies of one neuron with ST-STDP from c are shown at every 350-ms instance. \textbf{e},
		The pre- and post-synaptic spikes corresponding to one synapse for the first \unit[700]{ms}. The synapse receives input from the upper-most handwritten pixel
		of this digit. Pre-synaptic spikes cease after \unit[350]{ms} due to the occlusion. The evolution of its efficacy's short-term component $F$ is also
		shown in orange.} \label{fig:wta}
\end{figure*}
Remarkably, we show that the model we derived as optimal can be implemented, albeit stochastically, by a network structure that is common in cortex, including neurons that are also
biologically plausible (Supplementary Information, section \ref{sec:snn_exp}). In particular, the soft-max probabilistic inference outputs are
computed by stochastic exponential spiking neurons, assembled in a powerful microcircuit structure that is common within neocortical layers
\cite{douglas2004ARN}, i.e. a soft winner-take-all (WTA) circuit with lateral inhibition \cite{neftci2013PNAS,nessler2013PLoS,diehl2015FCN}.
Stochastic exponential spiking behaviour is in good agreement with the recorded activity of real neurons \cite{jolivet2006JCN}. Using techniques
similar to \cite{nessler2013PLoS}, the output spikes of this network can be shown to be samples from the posterior distribution, conditional on the
input and its past. The parameter updates of Eq. \ref{eq:F} in this spiking implementation are event-based, they occur at the time of each output
spike, and they depend on its latency $\tau$ from preceding input spikes. Finally, to complete the biological realism of the model, even the initial,
resting weights $\boldsymbol{W}^{(k)}$ can be obtained through standard, long-term STDP \cite{nessler2013PLoS} before the activation of ST-STDP.

An alternative model to the stochastic exponential neuron, and a common choice both in computational neuroscience simulations and in neuromorphic hardware circuits,
owing to its effective simplicity, is the leaky integrate-and-fire (LIF) model. We show for a WTA circuit constructed by LIF neurons, that if it
implements elastic clustering through the same ST-STDP rule of equations 1-3, this network maintains an approximate model of the random continuous
environment (Supplementary Information, sections \ref{sec:model_lin}, \ref{sec:neuro_equiv}).

Overall, thus far we have derived rigorously the optimal solution to modelling a class of common dynamic environments as a novel ML and theoretical neuroscience algorithm, namely neural elastic clustering. We have specified that it requires Hebbian STP for its optimality, and we have shown its plausible equivalence to cortical computations. As a consequence, through these computations, a typical cortical neural circuit described mathematically can maintain a Bayesian generative model that predicts and infers dynamic environments.

\section*{Application to recognition of transforming visual inputs}
We will now demonstrate the impact of ST-STDP on cortical-like circuit functionality also in simulations. Importantly, we will show that our simulation of a spiking model of a cortical neural microcircuit also achieves surprising practical advantages for AI as well, that can overcome key limitations of other neural AI approaches. We apply this SNN on the task of recognizing the frames of a sequence of gradually occluded MNIST (OMNIST) handwritten digits (see Methods section). Crucially, the network is tested on this task without any prior exposure to temporal visual sequences, or to occluded digits, but only to the static and untransformed MNIST digits. Such potential functionality of cortical microcircuit motifs for processing dynamically transforming and distorted sensory inputs has not been previously demonstrated. In addition, new AI benchmarks and applications, better suited for neuromorphic algorithms, have been highly sought but thus far missing \cite{pfeiffer2018FN}. Towards that goal, the task's design here is chosen indeed to specifically demonstrate the strengths of ST-STDP, while being attainable by a single-layer SNN with unsupervised learning and high biological realism.

In the OMNIST sequence, an
occluding object progressively hides each digit, and after a random number of frames the digit is fully occluded and replaced by random noise, before
the next digit appears in the visual sequence (see Fig. \ref{fig:wta}\textbf{a}). The task is to classify each frame into one of ten digit classes (0-9) or one
eleventh class of noisy frames. The MNIST classification task is a standard benchmarking task in which ANNs easily achieve almost perfect accuracies.
On the other hand, the derived OMNIST classification task used here is significantly more difficult due to the large occlusions, the noise, the
unpredictable duration of each digit's appearance, and the training on a different dataset, i.e. standard MNIST.
In addition, recognition under occlusion is considered a difficult problem that is highly relevant to the computer vision community, partially motivating our choice of task. Here we address the problem in a specific form that has some continuity in time and some random dynamics. This matches the type of sensory environment that in previous sections we showed demands ST-STDP for its optimal solution. ST-STDP is expected to help by exploiting the temporal continuity of the input to maintain a memory of previously unoccluded pixels also after their occlusion, whereas its synapse-specificity and its short-term nature are expected to help with the randomness. Our mathematical analysis supports this formally, but an intuitive explanation is the following. Continuity implies that recent observations can act as predictions about the future, and randomness implies that further refinement of that prediction would offer little help. Because of this, ST-STDP's short-term synaptic memories of recent unoccluded pixels of a specific object, combined with the long-term weights representing its object category, are the optimal model of future pixels of this object category.

We use an SNN of 400 output neurons in a soft-WTA structure consistent with the canonical local
connectivity observed in the neocortex \cite{douglas2004ARN}, and with the neural implementation we derived for the elastic clustering algorithm. Each output neuron is
stimulated by rate-coded inputs from 784 pixels through plastic excitatory afferent synapses (see Fig. \ref{fig:wta}\textbf{b}), and inhibited by each of the output neurons through an inhibitory interneuron. The 400 elastic
cluster neuron-centroids find their resting positions through unsupervised learning that emerges from conventional, long-term STDP
\cite{diehl2015FCN}, and driven by stimulation with the unoccluded MNIST training images. In cortex, such WTA circuits with strong lateral inhibition result in oscillatory dynamics, e.g. gamma oscillations as produced by the pyramidal-interneuron gamma (PING) circuit in cortex. PING oscillations can coexist with specialization of different sub-populations of the excitatory population \cite{cannon2014neurosystems} to different input patterns, similar to the specialization that is required here from the spiking neuron-centroids. Here too, within the duration of each single input pattern, the winning neurons of the WTA that specialize in one input digit will likely also oscillate, but we did not examine if that is the case, as it is out of our scope, and as it did not prevent the SNN from learning its task.

The 400 categories discovered by the 400 neuron-centroids through unsupervised learning in the SNN account for different styles of writing each of the 10 digits. Thus, each neuron corresponds to a subclass in the data, and each digit class corresponds to an ensemble of neurons. The winning ensemble determines the inferred label, and the winning ensemble is determined by its total firing rate. On the MNIST test set, the network achieves a
classification accuracy of $88.49\%\pm0.12\%$ (mean$\pm$ standard deviation over 5 runs with different random initialization), similar to prior results for this network size \cite{diehl2015FCN}. This network structure, and the
	training and testing protocol are adapted with small implementation differences from \cite{diehl2015FCN} and from our theoretically derived optimal implementation (see Methods section and Supplementary
	Information, section \ref{sec:sim_diffs}). When tested on the OMNIST test set without using ST-STDP or retraining, the performance drops to $61.10\%\pm0.48\%$.
However, when ST-STDP is enabled at the synapses, we observe that neurons that recognize an input digit continue recognizing it even when it is
partly occluded, and without significantly affecting the recognition of subsequent digits or noisy objects (Fig. \ref{fig:wta}\textbf{c}), reaching a
performance of $87.90\%\pm0.29\%$. This performance improvement by ST-STDP in this simulation with high biological realism suggests that ST-STDP in synapses may enable microcircuits of the neocortex
	to perform crucial computations robustly despite severe input distortions, by exploiting temporal relationships in the stimuli. This is achieved as ST-STDP strengthens specifically the synapses that contribute to the recognition of a recent frame (Fig.
\ref{fig:wta}, \textbf{d}-\textbf{e}). The spiking neurons' synaptic vectors follow the input vector and relax to their resting position (Fig. \ref{fig:pca}) confirming in practice that ST-STDP realizes the theoretical elastic clustering algorithm. This SNN with ST-STDP serves primarily as a demonstration of plausible processing in biological networks, embodying the theoretical model of our derivations. However, it is also a neuromorphic ML and inference algorithm for computer vision and other domains, directly suitable for efficient neuromorphic hardware.

\begin{figure*}[h!]
	\centering
	\begin{tabular}{c}
		\includegraphics[width = 155mm]{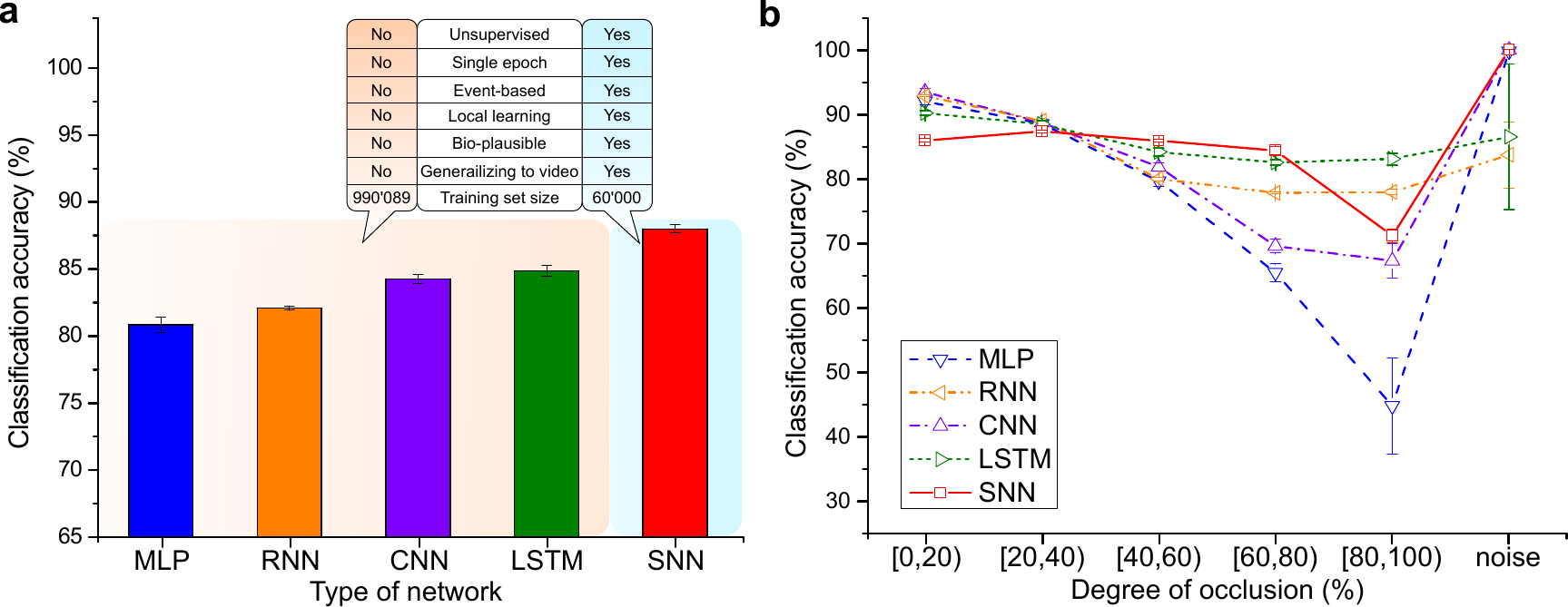}
	\end{tabular}
	\caption{\textbf{Comparison of SNN and ANN classification results.} \textbf{a}, Classification accuracy of SNN with ST-STDP on the OMNIST test set
		compared with that of an MLP, a CNN, an RNN and an LSTM. The SNN is trained on the MNIST training set while the other networks are trained on the OMNIST
		training set. The bar graph shows the mean and standard deviation over ten runs starting from different random weight initializations. The training
		set sizes of MNIST and OMNIST are 60,000 and 990,089, respectively. It can be seen that the SNN with ST-STDP achieves higher classification accuracy
		than all the other networks. In addition, the SNN compares positively across a number of qualities, namely its unsupervised training, in just one
		epoch, with event-based operation, synaptically-local learning, use of largely biologically-plausible mechanisms, generalization to transforming
		video from the standard static MNIST images, and from a significantly smaller training set. \textbf{b}, The performance of the five models in
		classifying OMNIST frames with different degrees of occlusion. Also shown is the classification accuracy when a noise frame is presented. Each data
		point indicates the mean and standard deviation over ten runs starting from different random weight initializations.} \label{fig:ANNs}
\end{figure*}

\section*{Comparison with ANNs}
Neuromorphic algorithms commonly trade off some accuracy compared to more conventional ML and inference, but we will now show that this SNN is in fact a quite competitive algorithm in a broad ML context, and in this particular task it outperforms powerful ANN models not only in accuracy, but in all relevant aspects. We compare it to ANNs that we train with backpropagation of errors, specifically a
multilayer perceptron (MLP), a convolutional neural network (CNN), a recurrent neural network (RNN), and long short-term memory (LSTM) units. The performance of each network depends on its size, so we used ANNs with at least as many trainable parameters as the SNN (see Methods section). The MLP and the CNN generally are well-equipped to tackle tasks of recognizing frames of static
images. Indeed, after training on the MNIST training set, they correctly recognize respectively
$96.89\%\pm0.17\%$ and $98.92\%\pm0.21\%$ of the MNIST testing frames. However, their performance on the OMNIST testing set is significantly lower, i.e. $56.89\%\pm0.78\%$ and $65.91\%\pm2.68\%$ respectively, after augmenting the
MNIST training set with noisy frames as an additional, 11th class of training images. Their accuracy increases to $80.76\%\pm0.58\%$ and $84.19\%\pm0.32\%$
respectively when they are trained on the 990,089 frames of the OMNIST training set (Fig. \ref{fig:ANNs}\textbf{a}), which is still lower than the performance of the SNN
with ST-STDP. Even the training access to OMNIST and supervision did not suffice the ANNs to outperform the unsupervised and MNIST-trained SNN. This
advantage stems from the fact that ST-STDP installs a short-term memory in the SNN, which allows it to make use of the relationships between frames,
while the CNN and the MLP are not equipped to deal with temporal aspects of data.

On the other hand, the RNN and the LSTM are expected to also exploit the slight temporal regularity present in the OMNIST data, if trained on the OMNIST sequence. The task of these recurrent networks is an easier task than the task of the SNN that is only trained on static MNIST data. However, even these networks and even though they are trained on the full OMNIST training sequence over multiple epochs with backpropagation through time, they achieve an accuracy of $82.01\%\pm0.17\%$ and $84.78\%\pm4.21\%$ which is still lower than the SNN's performance, even though the SNN is only
trained on MNIST's 60,000 static images in a single epoch, and in an unsupervised and local manner (Fig. \ref{fig:ANNs}\textbf{a}).

The key reason for the superiority of the SNN against the RNN and LSTM is randomness in the data. The strengths of recurrent ANNs, particularly through supervised training, become important in tasks where the temporal structure of the training data is repeatable and regular, such that it can be learned, and then applied to the testing data. However, here, the OMNIST data sequences have a degree of randomness, particularly in the duration of each digit's sequence, in the randomly noisy frames, and in the random order of digits, such that these sequences cannot be learned precisely.
The SNN achieves to navigate the random changes through the on-line adaptation of synapses and the short-term memory of STP. It achieves to do deal with these changes better than the recurrent ANNs because each neuron is equipped not with one but with multiple and independent short-term memories in its synapses, as opposed to the single short-term memory per neuron in RNN and LSTM networks. Thus, STP realizes an input-feature-specific, i.e. in this case pixel-specific, memory (see Supplementary Information, section \ref{sec:reshapes}). As a result, a recently active neuron with ST-STDP propagates forward through time in its synaptic states not only the information that input similar to its preferred stimulus was just observed but also the specific features that this input just had.
This specificity allows each neuron to keep recognizing objects that persist across frames even if these objects transform into pixels different from those stored in the long-term weights (see Supplementary Information, Fig. \ref{fig:weirdtail}). In addition, it prevents false recognition of a suddenly introduced different object, such as a noisy frame, as if it were continuation of the prior object that is still in short-term memory.
It is important to note that the time constant of STP in the SNN was a single hyperparameter for all neurons and synapses, whereas the RNN and LSTM have multiple time constants, one per recurrent connection. In addition, while the SNN's time constant was only hand-tuned, the RNN and LSTM time constants were objectively optimized by backpropagation through time, showing that the SNN's advantage does not lie in the choice of its time constant parameter, but in the unique model structure.
In Supplementary Information, section \ref{sec:LSTM} we provide some further elaboration on the contrast between short-term memory in LSTM and in ST-STDP.

Noise is a key aspect of the dataset that makes this SNN more suitable than the LSTM for the task. In OMNIST, we have relaxed the theoretical assumptions, and little randomness other than noise is involved in object shapes and their transformations, as the occlusion-introduced transformations are repeatable and rather small. The role of noise is to introduce some of the randomness in input object shapes that is assumed by our theoretical problem setting, and it is essential to the results, along with the random duration and the continuity of each digit.
This lack of randomness in transformations makes OMNIST one of the easier problems that ST-STDP can be applied to, and other models could perform well too if noise were not present, e.g. as shown by LSTM's accuracy in non-noisy frames (Fig. \ref{fig:ANNs}\textbf{b}). This problem of occlusions was still chosen, for its visual simplicity, for its significance to the field of computer vision, and for its stronger realism than a hypothetical dataset with digits that truly morph randomly. 
In more general and more natural scenarios where objects, e.g. digits, actually transform and morph randomly, the assumptions that make the problem best addressed by STP would already be met without the addition of noise. Similarly, in tasks where the possible transformations are unknown, such as in problems where temporal data are unavailable during training, the accuracy advantages of ST-STDP become even stronger. This is demonstrated in our results by the fact that the SNN is the most accurate model on OMNIST despite training only on MNIST, and by its further-increased advantage when the MLP and CNN are also only trained on MNIST. In fact, the accuracy comparison to the LSTM, trained on the OMNIST training sequence, is arguably not fair for the SNN, as the SNN is asked to solve a different and harder task, namely the task of generalizing from the static MNIST set. Still, it performs best.

Fig. \ref{fig:ANNs} summarizes the results on ML-related benchmarking. However, it should be recalled that merely the accuracy advantage of our algorithm in a ML benchmark does not summarize the full scope of this article. The qualitative advantages we described are perhaps more interesting, and are summarized in the inset of Fig. \ref{fig:ANNs}\textbf{a}. Equally importantly, results reported here pertain also to theoretical ML, to neuroscience, and to neuromorphic computing, and we discuss these further in the next section.

\section*{Discussion}
Our results are centred on a single mechanism, but have multiple implications. They span from theoretical derivations of the optimal solution for a category of ML problems, to a new role of synaptic plasticity, a newly derived theory for dynamic Bayesian inference in circuit models of the cerebral cortex, simulations showing processing of distorted and changing sensory inputs by such biological microcircuit models, a new dataset addressing demands of neuromorphic computing, a simple ML solution to visual recognition under some types of occlusion, and an implementation of an event-based algorithm suitable for neuromorphic hardware that demonstrates in practice and in combination many of the previously hypothesized advantages of neuromorphic computing.

More specifically, we have demonstrated theoretically and experimentally that certain properties so far unique to SNNs are required for the optimal processing of randomly but
continuously changing environments, as natural environments commonly are. This indicates that spike-based operations in the brain can provide it with statistically optimal internal models of dynamic environments and may be partly responsible for the brain's
computational superiority to machines in interacting with natural settings. It should be noted that while spiking activity itself may have its own advantages for the brain and ML, the optimality in our specific case does not arise from spikes per se, but rather by the consequent synaptic dynamics. Such temporal responses triggered by perisynaptic events are a hallmark of SNNs and have been absent from conventional ANN models, but as we show, they can be modelled independently from spikes. We have also shown that cortical neural microcircuit models are able to generalize from static, to transforming and novel sensory inputs through Bayesian computation, and that this could make such circuits robust to significant input distortions.
We identified Hebbian short-term plasticity as a key mechanism for this theoretical and practical optimality. This 
suggests a previously unexplored link from low-level biophysics of neurons to high-level cortical functionality, and reveals a new functional role for STP.
Moreover, we showed in practice for the first time that the performance of cortically-modelled networks can
surpass that of usually powerful conventional ANNs in a certain task, using only spikes and local unsupervised plasticity.

In addition to its biological implications, this algorithm is a new neuromorphic ML approach to modelling dynamic environments. In conjunction with the theoretical problem setting of environments with continuous and random dynamics that we emphasized here, it suggests a category of applications as best addressed by neuromorphic computing. Moreover, our rigorous proof formalizes and validates the common view in the neuromorphic community that neuromorphic computing could be best applied in problems with natural dynamics. Importantly, our neuromorphic approach differs from the currently dominant deep learning techniques for neural AI. Elastic clustering and ST-STDP operate by online adaptation to temporally local information, as opposed to the usual global training with an all-encompassing dataset in a separate prior phase. This partly circumvents a notorious
limitation of deep learning, namely that its effectiveness is limited by the breadth and size of the training set. This is demonstrated in our
experiments by the SNN's training on only static and untransformed images, in just one epoch, and with fewer examples. Furthermore, the SNN's training phase
not only is short, but also it does not include examples of the visual dynamics. These qualitative advantages, i.e. strong generalization, learning with local learning rules and without supervision, reliance on smaller training datasets, event-based operation, and on-line adaptivity, have long been among the principal aims of neuromorphic computing, and are also in the focus of conventional AI and ML research. Here we demonstrate that at least in certain cases, these long-standing neuromorphic aims are actually all possible simultaneously, and without trading off accuracy.

We must note that a broad range of deep-learning techniques that were not used here can improve the performance of ANNs on this task, so that they could eventually outperform this SNN, assuming no improvements are added to the SNN as well. Another limitation of our work is that the task's chosen specifics are tailored to be difficult to address with conventional ANN schemes. On the one hand, this specificity provides a concretely beneficial neuromorphic computing application, examples of which have been in high demand by neuromorphic researchers. On the other hand, it reconfirms that further research on biologically-plausible AI approaches is necessary to broaden their applicability. However, it is remarkable that such a simple and biologically plausible network can challenge ANNs in such a demanding task at all. It is the first demonstration in which supervision, non-local learning, training on temporal sequences and including occluded examples, non-spiking neural activations, a larger training set, multiple training epochs and neuronal layers, convolutions of CNNs, and gating mechanisms of LSTMs, do not immediately suffice to greatly outperform a simple spiking circuit with STDP and purely neuromorphic properties, that is not even trained on the actual tested OMNIST task, but on standard MNIST images. This surprising result is due to the randomness in input changes. Random dynamics -- frequent in nature and exemplified by the input sequence used in the experiments -- are unlearnable from example sequences in an off-line manner even if supervised backpropagation is applied, but are captured by on-line adaptivity.

Improvements to conventional ML techniques for tasks like OMNIST are not only hypothetically possible, but in fact suggested by our theoretical derivation. Our theory shows that the mechanism to add to conventional algorithms to achieve optimal performance in environments with random continuous dynamics is Hebbian STP itself. Concretely, while here we focused on a biologically-plausible theory and implementation, in ML systems the advantages of online unsupervised elastic clustering can also be combined with a backpropagation-trained deep SNN or ANN. Of course, STP and ST-STDP are commonly associated only with spiking activity. Nevertheless, our theory does dictate a Hebbian STP alternative that implements the neural elastic clustering without spikes. In addition, the initial and resting positions of the elastic neurons (Fig. \ref{fig:clustering}\textbf{a}) do not have to be learned in a biologically-plausible manner through STDP, and the theory does not preclude elastic neurons from the hidden layers of a deep network.
In that case, supervised training of a deep neural network through backpropagation could be used to obtain the initial state of the network, whereas an unsupervised and local ST-STDP-like neuromorphic scheme could then be added to the network's layers during inference, for on-line
adaptation to dynamic data in random but continuous transforming environments. This would likely improve the network's performance, increase its biological relevance, and practically integrate neuromorphic concepts into ANNs. This method might be useful in helping neural networks generalize from static training settings to dynamic environments, such as from images to videos, from short audio snippets to recordings in cocktail-party \cite{haykin2005cocktail} auditory settings, from individual typed sentences to longer evolving texts, etc. These are readily implementable next steps that can apply the strengths of ST-STDP and elastic clustering to larger AI problems than this first demonstrated application that was purposely constrained by biological plausibility. Thus, one outcome of this article is the theoretical introduction of STP, a typical SNN mechanism, to networks with non-spiking activations and to conventional ML as an elastic clustering algorithm. Therefore, even though this SNN's performance on OMNIST may ultimately be surpassed by supervised deep neural networks, the significance of mechanisms from the field of SNNs for both cortical circuits and AI should remain a key message.

Contrary to ML, biological analogues to global supervised training with backprop are possibly absent. Therefore, adaptivity similar to neural elastic clustering, combined with the temporal continuity of natural sensory inputs, may have an active role in enabling the CNS to perform robustly even in the presence of only unsupervised and synaptically-local learning, as in our mathematical model and simulations.

All in all, the results reported here stem from the systematic identification of analogies between biophysics and the spatiotemporal and causal structure of the
world. We propose this as a strategy that could reveal more of the principles that optimize the brain for its environment.

\bibliographystyle{naturemag}
{\footnotesize \bibliography{References_nmi}}

\begin{thebibliography}{10}
\expandafter\ifx\csname url\endcsname\relax
  \def\url#1{\texttt{#1}}\fi
\expandafter\ifx\csname urlprefix\endcsname\relax\def\urlprefix{URL }\fi
\providecommand{\bibinfo}[2]{#2}
\providecommand{\eprint}[2][]{\url{#2}}

\bibitem{maass1997NN}
\bibinfo{author}{Maass, W.}
\newblock \bibinfo{title}{{Networks of spiking neurons: The third generation of
  neural network models}}.
\newblock \emph{\bibinfo{journal}{Neural Networks}}  (\bibinfo{year}{1997}).

\bibitem{ponulak2011ANE}
\bibinfo{author}{Ponulak, F.} \& \bibinfo{author}{Kasinski, A.}
\newblock \bibinfo{title}{Introduction to spiking neural networks: Information
  processing, learning and applications.}
\newblock \emph{\bibinfo{journal}{Acta neurobiologiae experimentalis}}
  \textbf{\bibinfo{volume}{71}}, \bibinfo{pages}{409--433}
  (\bibinfo{year}{2011}).

\bibitem{gruning2014ESANN}
\bibinfo{author}{Gr{\"u}ning, A.} \& \bibinfo{author}{Bohte, S.~M.}
\newblock \bibinfo{title}{Spiking neural networks: Principles and challenges.}
\newblock In \emph{\bibinfo{booktitle}{Proceedings of European Symposium on
  Artificial Neural Networks (ESANN)}} (\bibinfo{year}{2014}).

\bibitem{markram1997Science}
\bibinfo{author}{Markram, H.}, \bibinfo{author}{L{\"u}bke, J.},
  \bibinfo{author}{Frotscher, M.} \& \bibinfo{author}{Sakmann, B.}
\newblock \bibinfo{title}{Regulation of synaptic efficacy by coincidence of
  postsynaptic aps and epsps}.
\newblock \emph{\bibinfo{journal}{Science}} \textbf{\bibinfo{volume}{275}},
  \bibinfo{pages}{213--215} (\bibinfo{year}{1997}).

\bibitem{bi1998JN}
\bibinfo{author}{Bi, G.-q.} \& \bibinfo{author}{Poo, M.-m.}
\newblock \bibinfo{title}{Synaptic modifications in cultured hippocampal
  neurons: dependence on spike timing, synaptic strength, and postsynaptic cell
  type}.
\newblock \emph{\bibinfo{journal}{Journal of neuroscience}}
  \textbf{\bibinfo{volume}{18}}, \bibinfo{pages}{10464--10472}
  (\bibinfo{year}{1998}).

\bibitem{song2000NatureNeuroscience}
\bibinfo{author}{Song, S.}, \bibinfo{author}{Miller, K.~D.} \&
  \bibinfo{author}{Abbott, L.~F.}
\newblock \bibinfo{title}{Competitive hebbian learning through
  spike-timing-dependent synaptic plasticity}.
\newblock \emph{\bibinfo{journal}{Nature neuroscience}}
  \textbf{\bibinfo{volume}{3}}, \bibinfo{pages}{919} (\bibinfo{year}{2000}).

\bibitem{zucker1989ARN}
\bibinfo{author}{Zucker, R.~S.}
\newblock \bibinfo{title}{Short-term synaptic plasticity}.
\newblock \emph{\bibinfo{journal}{Annual review of neuroscience}}
  \textbf{\bibinfo{volume}{12}}, \bibinfo{pages}{13--31}
  (\bibinfo{year}{1989}).

\bibitem{tsodyks1997PNAS}
\bibinfo{author}{Tsodyks, M.~V.} \& \bibinfo{author}{Markram, H.}
\newblock \bibinfo{title}{The neural code between neocortical pyramidal neurons
  depends on neurotransmitter release probability}.
\newblock \emph{\bibinfo{journal}{Proceedings of the National Academy of
  Sciences}}  (\bibinfo{year}{1997}).

\bibitem{chamberlain2008NP}
\bibinfo{author}{Chamberlain, S.~E.}, \bibinfo{author}{Yang, J.} \&
  \bibinfo{author}{Jones, R.~S.}
\newblock \bibinfo{title}{The role of nmda receptor subtypes in short-term
  plasticity in the rat entorhinal cortex}.
\newblock \emph{\bibinfo{journal}{Neural plasticity}}
  \textbf{\bibinfo{volume}{2008}} (\bibinfo{year}{2008}).

\bibitem{rosenbaum2012PLoS}
\bibinfo{author}{Rosenbaum, R.}, \bibinfo{author}{Rubin, J.} \&
  \bibinfo{author}{Doiron, B.}
\newblock \bibinfo{title}{Short term synaptic depression imposes a frequency
  dependent filter on synaptic information transfer}.
\newblock \emph{\bibinfo{journal}{PLoS Computational Biology}}
  (\bibinfo{year}{2012}).

\bibitem{moraitis2018IEEE}
\bibinfo{author}{Moraitis, T.}, \bibinfo{author}{Sebastian, A.} \&
  \bibinfo{author}{Eleftheriou, E.}
\newblock \bibinfo{title}{The role of short-term plasticity in neuromorphic
  learning: Learning from the timing of rate-varying events with fatiguing
  spike-timing-dependent plasticity}.
\newblock \emph{\bibinfo{journal}{IEEE Nanotechnology Magazine}}
  (\bibinfo{year}{2018}).

\bibitem{moraitis2018IJCNN}
\bibinfo{author}{Moraitis, T.}, \bibinfo{author}{Sebastian, A.} \&
  \bibinfo{author}{Eleftheriou, E.}
\newblock \bibinfo{title}{Spiking neural networks enable two-dimensional
  neurons and unsupervised multi-timescale learning}.
\newblock In \emph{\bibinfo{booktitle}{Proceedings of the International Joint
  Conference on Neural Networks}} (\bibinfo{year}{2018}).

\bibitem{salaj2020spike}
\bibinfo{author}{Salaj, D.} \emph{et~al.}
\newblock \bibinfo{title}{Spike-frequency adaptation provides a long short-term
  memory to networks of spiking neurons}.
\newblock \emph{\bibinfo{journal}{bioRxiv}}  (\bibinfo{year}{2020}).

\bibitem{leng2018spiking}
\bibinfo{author}{Leng, L.} \emph{et~al.}
\newblock \bibinfo{title}{Spiking neurons with short-term synaptic plasticity
  form superior generative networks}.
\newblock \emph{\bibinfo{journal}{Scientific reports}}
  \textbf{\bibinfo{volume}{8}}, \bibinfo{pages}{1--11} (\bibinfo{year}{2018}).

\bibitem{mead1990IEEEProc}
\bibinfo{author}{Mead, C.}
\newblock \bibinfo{title}{Neuromorphic electronic systems}.
\newblock \emph{\bibinfo{journal}{Proceedings of the IEEE}}
  \textbf{\bibinfo{volume}{78}}, \bibinfo{pages}{1629--1636}
  (\bibinfo{year}{1990}).

\bibitem{merolla2014Science}
\bibinfo{author}{Merolla, P.~A.} \emph{et~al.}
\newblock \bibinfo{title}{A million spiking-neuron integrated circuit with a
  scalable communication network and interface}.
\newblock \emph{\bibinfo{journal}{Science}} \textbf{\bibinfo{volume}{345}},
  \bibinfo{pages}{668--673} (\bibinfo{year}{2014}).

\bibitem{qiao2015FN}
\bibinfo{author}{Qiao, N.} \emph{et~al.}
\newblock \bibinfo{title}{A reconfigurable on-line learning spiking
  neuromorphic processor comprising 256 neurons and 128k synapses}.
\newblock \emph{\bibinfo{journal}{Frontiers in neuroscience}}
  \textbf{\bibinfo{volume}{9}}, \bibinfo{pages}{141} (\bibinfo{year}{2015}).

\bibitem{indiveri2018NMEH}
\bibinfo{author}{Indiveri, G.} \& \bibinfo{author}{Douglas, R.}
\newblock \bibinfo{title}{Neuromorphic networks of spiking neurons}.
\newblock \emph{\bibinfo{journal}{Nano and Molecular Electronics Handbook}}
  \textbf{\bibinfo{volume}{10}} (\bibinfo{year}{2018}).

\bibitem{Davies2018Micro}
\bibinfo{author}{{Davies}, M.} \emph{et~al.}
\newblock \bibinfo{title}{Loihi: A neuromorphic manycore processor with on-chip
  learning}.
\newblock \emph{\bibinfo{journal}{IEEE Micro}} \textbf{\bibinfo{volume}{38}},
  \bibinfo{pages}{82--99} (\bibinfo{year}{2018}).

\bibitem{boi2016FN}
\bibinfo{author}{Boi, F.} \emph{et~al.}
\newblock \bibinfo{title}{A bidirectional brain-machine interface featuring a
  neuromorphic hardware decoder}.
\newblock \emph{\bibinfo{journal}{Frontiers in neuroscience}}
  \textbf{\bibinfo{volume}{10}}, \bibinfo{pages}{563} (\bibinfo{year}{2016}).

\bibitem{serb2020SR}
\bibinfo{author}{Serb, A.} \emph{et~al.}
\newblock \bibinfo{title}{Memristive synapses connect brain and silicon spiking
  neurons}.
\newblock \emph{\bibinfo{journal}{Scientific Reports}}
  \textbf{\bibinfo{volume}{10}}, \bibinfo{pages}{1--7} (\bibinfo{year}{2020}).

\bibitem{wolpert1995Science}
\bibinfo{author}{Wolpert, D.~M.}, \bibinfo{author}{Ghahramani, Z.} \&
  \bibinfo{author}{Jordan, M.~I.}
\newblock \bibinfo{title}{An internal model for sensorimotor integration}.
\newblock \emph{\bibinfo{journal}{Science}} \textbf{\bibinfo{volume}{269}},
  \bibinfo{pages}{1880--1882} (\bibinfo{year}{1995}).

\bibitem{kording2004Nature}
\bibinfo{author}{K{\"o}rding, K.~P.} \& \bibinfo{author}{Wolpert, D.~M.}
\newblock \bibinfo{title}{Bayesian integration in sensorimotor learning}.
\newblock \emph{\bibinfo{journal}{Nature}} \textbf{\bibinfo{volume}{427}},
  \bibinfo{pages}{244--247} (\bibinfo{year}{2004}).

\bibitem{ma2006NatureNeuroscience}
\bibinfo{author}{Ma, W.~J.}, \bibinfo{author}{Beck, J.~M.},
  \bibinfo{author}{Latham, P.~E.} \& \bibinfo{author}{Pouget, A.}
\newblock \bibinfo{title}{Bayesian inference with probabilistic population
  codes}.
\newblock \emph{\bibinfo{journal}{Nature neuroscience}}
  \textbf{\bibinfo{volume}{9}}, \bibinfo{pages}{1432--1438}
  (\bibinfo{year}{2006}).

\bibitem{blaisdell2006Science}
\bibinfo{author}{Blaisdell, A.~P.}, \bibinfo{author}{Sawa, K.},
  \bibinfo{author}{Leising, K.~J.} \& \bibinfo{author}{Waldmann, M.~R.}
\newblock \bibinfo{title}{Causal reasoning in rats}.
\newblock \emph{\bibinfo{journal}{Science}} \textbf{\bibinfo{volume}{311}},
  \bibinfo{pages}{1020--1022} (\bibinfo{year}{2006}).

\bibitem{griffiths2006PS}
\bibinfo{author}{Griffiths, T.~L.} \& \bibinfo{author}{Tenenbaum, J.~B.}
\newblock \bibinfo{title}{Optimal predictions in everyday cognition}.
\newblock \emph{\bibinfo{journal}{Psychological science}}
  \textbf{\bibinfo{volume}{17}}, \bibinfo{pages}{767--773}
  (\bibinfo{year}{2006}).

\bibitem{doya2007MIT}
\bibinfo{author}{Doya, K.}, \bibinfo{author}{Ishii, S.},
  \bibinfo{author}{Pouget, A.} \& \bibinfo{author}{Rao, R.~P.}
\newblock \emph{\bibinfo{title}{Bayesian brain: Probabilistic approaches to
  neural coding}} (\bibinfo{publisher}{MIT press}, \bibinfo{year}{2007}).

\bibitem{fiser2010TCS}
\bibinfo{author}{Fiser, J.}, \bibinfo{author}{Berkes, P.},
  \bibinfo{author}{Orb{\'a}n, G.} \& \bibinfo{author}{Lengyel, M.}
\newblock \bibinfo{title}{Statistically optimal perception and learning: from
  behavior to neural representations}.
\newblock \emph{\bibinfo{journal}{Trends in cognitive sciences}}
  \textbf{\bibinfo{volume}{14}}, \bibinfo{pages}{119--130}
  (\bibinfo{year}{2010}).

\bibitem{berkes2011Science}
\bibinfo{author}{Berkes, P.}, \bibinfo{author}{Orb{\'a}n, G.},
  \bibinfo{author}{Lengyel, M.} \& \bibinfo{author}{Fiser, J.}
\newblock \bibinfo{title}{Spontaneous cortical activity reveals hallmarks of an
  optimal internal model of the environment}.
\newblock \emph{\bibinfo{journal}{Science}} \textbf{\bibinfo{volume}{331}},
  \bibinfo{pages}{83--87} (\bibinfo{year}{2011}).

\bibitem{bastos2012Neuron}
\bibinfo{author}{Bastos, A.~M.} \emph{et~al.}
\newblock \bibinfo{title}{Canonical microcircuits for predictive coding}.
\newblock \emph{\bibinfo{journal}{Neuron}} \textbf{\bibinfo{volume}{76}},
  \bibinfo{pages}{695--711} (\bibinfo{year}{2012}).

\bibitem{nessler2013PLoS}
\bibinfo{author}{Nessler, B.}, \bibinfo{author}{Pfeiffer, M.},
  \bibinfo{author}{Buesing, L.} \& \bibinfo{author}{Maass, W.}
\newblock \bibinfo{title}{Bayesian computation emerges in generic cortical
  microcircuits through spike-timing-dependent plasticity}.
\newblock \emph{\bibinfo{journal}{PLoS computational biology}}
  \textbf{\bibinfo{volume}{9}}, \bibinfo{pages}{e1003037}
  (\bibinfo{year}{2013}).

\bibitem{douglas2004ARN}
\bibinfo{author}{Douglas, R.~J.} \& \bibinfo{author}{Martin, K.~A.}
\newblock \bibinfo{title}{Neuronal circuits of the neocortex}.
\newblock \emph{\bibinfo{journal}{Annual Reviews in Neuroscience}}
  \textbf{\bibinfo{volume}{27}}, \bibinfo{pages}{419--451}
  (\bibinfo{year}{2004}).

\bibitem{diehl2015FCN}
\bibinfo{author}{Diehl, P.} \& \bibinfo{author}{Cook, M.}
\newblock \bibinfo{title}{{Unsupervised learning of digit recognition using
  spike-timing-dependent plasticity}}.
\newblock \emph{\bibinfo{journal}{Frontiers in Computational Neuroscience}}
  \textbf{\bibinfo{volume}{9}}, \bibinfo{pages}{99} (\bibinfo{year}{2015}).

\bibitem{poirazi2003Neuron}
\bibinfo{author}{Poirazi, P.}, \bibinfo{author}{Brannon, T.} \&
  \bibinfo{author}{Mel, B.~W.}
\newblock \bibinfo{title}{Pyramidal neuron as two-layer neural network}.
\newblock \emph{\bibinfo{journal}{Neuron}} \textbf{\bibinfo{volume}{37}},
  \bibinfo{pages}{989--999} (\bibinfo{year}{2003}).

\bibitem{gidon2020Science}
\bibinfo{author}{Gidon, A.} \emph{et~al.}
\newblock \bibinfo{title}{Dendritic action potentials and computation in human
  layer 2/3 cortical neurons}.
\newblock \emph{\bibinfo{journal}{Science}} \textbf{\bibinfo{volume}{367}},
  \bibinfo{pages}{83--87} (\bibinfo{year}{2020}).

\bibitem{bellec2018NIPS}
\bibinfo{author}{Bellec, G.}, \bibinfo{author}{Salaj, D.},
  \bibinfo{author}{Subramoney, A.}, \bibinfo{author}{Legenstein, R.} \&
  \bibinfo{author}{Maass, W.}
\newblock \bibinfo{title}{Long short-term memory and learning-to-learn in
  networks of spiking neurons}.
\newblock In \bibinfo{editor}{Bengio, S.} \emph{et~al.} (eds.)
  \emph{\bibinfo{booktitle}{Advances in Neural Information Processing
  Systems}}, \bibinfo{pages}{787--797} (\bibinfo{publisher}{Curran Associates,
  Inc.}, \bibinfo{year}{2018}).

\bibitem{pfeiffer2018FN}
\bibinfo{author}{Pfeiffer, M.} \& \bibinfo{author}{Pfeil, T.}
\newblock \bibinfo{title}{Deep learning with spiking neurons: opportunities and
  challenges}.
\newblock \emph{\bibinfo{journal}{Frontiers in neuroscience}}
  \textbf{\bibinfo{volume}{12}}, \bibinfo{pages}{774} (\bibinfo{year}{2018}).

\bibitem{rajendran2019IEEESPMag}
\bibinfo{author}{Rajendran, B.}, \bibinfo{author}{Sebastian, A.},
  \bibinfo{author}{Schmuker, M.}, \bibinfo{author}{Srinivasa, N.} \&
  \bibinfo{author}{Eleftheriou, E.}
\newblock \bibinfo{title}{Low-power neuromorphic hardware for signal processing
  applications: A review of architectural and system-level design approaches}.
\newblock \emph{\bibinfo{journal}{IEEE Signal Processing Magazine}}
  \textbf{\bibinfo{volume}{36}}, \bibinfo{pages}{97--110}
  (\bibinfo{year}{2019}).

\bibitem{wozniak2020NatureMI}
\bibinfo{author}{Wo{\'z}niak, S.}, \bibinfo{author}{Pantazi, A.},
  \bibinfo{author}{Bohnstingl, T.} \& \bibinfo{author}{Eleftheriou, E.}
\newblock \bibinfo{title}{Deep learning incorporating biologically inspired
  neural dynamics and in-memory computing}.
\newblock \emph{\bibinfo{journal}{Nature Machine Intelligence}}
  \textbf{\bibinfo{volume}{2}}, \bibinfo{pages}{325--336}
  (\bibinfo{year}{2020}).

\bibitem{bengio2015ArXiv}
\bibinfo{author}{Bengio, Y.}, \bibinfo{author}{Lee, D.-H.},
  \bibinfo{author}{Bornschein, J.}, \bibinfo{author}{Mesnard, T.} \&
  \bibinfo{author}{Lin, Z.}
\newblock \bibinfo{title}{Towards biologically plausible deep learning}.
\newblock \emph{\bibinfo{journal}{arXiv preprint arXiv:1502.04156}}
  (\bibinfo{year}{2015}).

\bibitem{brenowitz2005Neuron}
\bibinfo{author}{Brenowitz, S.~D.} \& \bibinfo{author}{Regehr, W.~G.}
\newblock \bibinfo{title}{Associative short-term synaptic plasticity mediated
  by endocannabinoids}.
\newblock \emph{\bibinfo{journal}{Neuron}} \textbf{\bibinfo{volume}{45}},
  \bibinfo{pages}{419--431} (\bibinfo{year}{2005}).

\bibitem{cassenaer2007Nature}
\bibinfo{author}{Cassenaer, S.} \& \bibinfo{author}{Laurent, G.}
\newblock \bibinfo{title}{Hebbian {STDP} in mushroom bodies facilitates the
  synchronous flow of olfactory information in locusts}.
\newblock \emph{\bibinfo{journal}{Nature}} \textbf{\bibinfo{volume}{448}},
  \bibinfo{pages}{709} (\bibinfo{year}{2007}).

\bibitem{erickson2010JCN}
\bibinfo{author}{Erickson, M.~A.}, \bibinfo{author}{Maramara, L.~A.} \&
  \bibinfo{author}{Lisman, J.}
\newblock \bibinfo{title}{A single brief burst induces glur1-dependent
  associative short-term potentiation: a potential mechanism for short-term
  memory}.
\newblock \emph{\bibinfo{journal}{Journal of cognitive neuroscience}}
  \textbf{\bibinfo{volume}{22}}, \bibinfo{pages}{2530--2540}
  (\bibinfo{year}{2010}).

\bibitem{szatmary2010PLoS}
\bibinfo{author}{Szatm{\'a}ry, B.} \& \bibinfo{author}{Izhikevich, E.~M.}
\newblock \bibinfo{title}{Spike-timing theory of working memory}.
\newblock \emph{\bibinfo{journal}{PLoS computational biology}}
  \textbf{\bibinfo{volume}{6}}, \bibinfo{pages}{e1000879}
  (\bibinfo{year}{2010}).

\bibitem{fiebig2017JN}
\bibinfo{author}{Fiebig, F.} \& \bibinfo{author}{Lansner, A.}
\newblock \bibinfo{title}{A spiking working memory model based on hebbian
  short-term potentiation}.
\newblock \emph{\bibinfo{journal}{Journal of Neuroscience}}
  \textbf{\bibinfo{volume}{37}}, \bibinfo{pages}{83--96}
  (\bibinfo{year}{2017}).

\bibitem{frey1993Science}
\bibinfo{author}{Frey, U.}, \bibinfo{author}{Huang, Y.} \&
  \bibinfo{author}{Kandel, E.}
\newblock \bibinfo{title}{Effects of camp simulate a late stage of ltp in
  hippocampal ca1 neurons}.
\newblock \emph{\bibinfo{journal}{Science}} \textbf{\bibinfo{volume}{260}},
  \bibinfo{pages}{1661--1664} (\bibinfo{year}{1993}).

\bibitem{huang1998CB}
\bibinfo{author}{Huang, E.~P.}
\newblock \bibinfo{title}{Synaptic plasticity: going through phases with ltp}.
\newblock \emph{\bibinfo{journal}{Current Biology}}
  \textbf{\bibinfo{volume}{8}}, \bibinfo{pages}{R350--R352}
  (\bibinfo{year}{1998}).

\bibitem{baltaci2019NR}
\bibinfo{author}{Baltaci, S.~B.}, \bibinfo{author}{Mogulkoc, R.} \&
  \bibinfo{author}{Baltaci, A.~K.}
\newblock \bibinfo{title}{Molecular mechanisms of early and late ltp}.
\newblock \emph{\bibinfo{journal}{Neurochemical research}}
  \textbf{\bibinfo{volume}{44}}, \bibinfo{pages}{281--296}
  (\bibinfo{year}{2019}).

\bibitem{schiller2000Nature}
\bibinfo{author}{Schiller, J.}, \bibinfo{author}{Major, G.},
  \bibinfo{author}{Koester, H.~J.} \& \bibinfo{author}{Schiller, Y.}
\newblock \bibinfo{title}{Nmda spikes in basal dendrites of cortical pyramidal
  neurons}.
\newblock \emph{\bibinfo{journal}{Nature}} \textbf{\bibinfo{volume}{404}},
  \bibinfo{pages}{285} (\bibinfo{year}{2000}).

\bibitem{wang2006heterogeneity}
\bibinfo{author}{Wang, Y.} \emph{et~al.}
\newblock \bibinfo{title}{Heterogeneity in the pyramidal network of the medial
  prefrontal cortex}.
\newblock \emph{\bibinfo{journal}{Nature neuroscience}}
  \textbf{\bibinfo{volume}{9}}, \bibinfo{pages}{534--542}
  (\bibinfo{year}{2006}).

\bibitem{hubel1959JP}
\bibinfo{author}{Hubel, D.~H.} \& \bibinfo{author}{Wiesel, T.~N.}
\newblock \bibinfo{title}{Receptive fields of single neurones in the cat's
  striate cortex}.
\newblock \emph{\bibinfo{journal}{The Journal of physiology}}
  \textbf{\bibinfo{volume}{148}}, \bibinfo{pages}{574--591}
  (\bibinfo{year}{1959}).

\bibitem{gerstner2014Cambridge}
\bibinfo{author}{Gerstner, W.}, \bibinfo{author}{Kistler, W.~M.},
  \bibinfo{author}{Naud, R.} \& \bibinfo{author}{Paninski, L.}
\newblock \emph{\bibinfo{title}{Neuronal dynamics: From single neurons to
  networks and models of cognition}} (\bibinfo{publisher}{Cambridge University
  Press}, \bibinfo{year}{2014}).

\bibitem{brette2015FSN}
\bibinfo{author}{Brette, R.}
\newblock \bibinfo{title}{Philosophy of the spike: rate-based vs. spike-based
  theories of the brain}.
\newblock \emph{\bibinfo{journal}{Frontiers in systems neuroscience}}
  \textbf{\bibinfo{volume}{9}}, \bibinfo{pages}{151} (\bibinfo{year}{2015}).

\bibitem{abbott2000NatureNeuroscience}
\bibinfo{author}{Abbott, L.~F.} \& \bibinfo{author}{Nelson, S.~B.}
\newblock \bibinfo{title}{Synaptic plasticity: taming the beast}.
\newblock \emph{\bibinfo{journal}{Nature neuroscience}}
  \textbf{\bibinfo{volume}{3}}, \bibinfo{pages}{1178} (\bibinfo{year}{2000}).

\bibitem{daoudal2003LM}
\bibinfo{author}{Daoudal, G.} \& \bibinfo{author}{Debanne, D.}
\newblock \bibinfo{title}{Long-term plasticity of intrinsic excitability:
  learning rules and mechanisms}.
\newblock \emph{\bibinfo{journal}{Learning \& memory}}
  \textbf{\bibinfo{volume}{10}}, \bibinfo{pages}{456--465}
  (\bibinfo{year}{2003}).

\bibitem{cudmore2004JN}
\bibinfo{author}{Cudmore, R.~H.} \& \bibinfo{author}{Turrigiano, G.~G.}
\newblock \bibinfo{title}{Long-term potentiation of intrinsic excitability in
  lv visual cortical neurons}.
\newblock \emph{\bibinfo{journal}{Journal of neurophysiology}}
  \textbf{\bibinfo{volume}{92}}, \bibinfo{pages}{341--348}
  (\bibinfo{year}{2004}).

\bibitem{turrigiano2011ARN}
\bibinfo{author}{Turrigiano, G.}
\newblock \bibinfo{title}{Too many cooks? intrinsic and synaptic homeostatic
  mechanisms in cortical circuit refinement}.
\newblock \emph{\bibinfo{journal}{Annual Review of Neuroscience}}
  \textbf{\bibinfo{volume}{34}}, \bibinfo{pages}{89--103}
  (\bibinfo{year}{2011}).

\bibitem{neftci2013PNAS}
\bibinfo{author}{Neftci, E.} \emph{et~al.}
\newblock \bibinfo{title}{Synthesizing cognition in neuromorphic electronic
  systems}.
\newblock \emph{\bibinfo{journal}{Proceedings of the National Academy of
  Sciences}} \textbf{\bibinfo{volume}{110}}, \bibinfo{pages}{E3468--E3476}
  (\bibinfo{year}{2013}).

\bibitem{jolivet2006JCN}
\bibinfo{author}{Jolivet, R.}, \bibinfo{author}{Rauch, A.},
  \bibinfo{author}{L{\"u}scher, H.-R.} \& \bibinfo{author}{Gerstner, W.}
\newblock \bibinfo{title}{Predicting spike timing of neocortical pyramidal
  neurons by simple threshold models}.
\newblock \emph{\bibinfo{journal}{Journal of computational neuroscience}}
  \textbf{\bibinfo{volume}{21}}, \bibinfo{pages}{35--49}
  (\bibinfo{year}{2006}).

\bibitem{cannon2014neurosystems}
\bibinfo{author}{Cannon, J.} \emph{et~al.}
\newblock \bibinfo{title}{Neurosystems: brain rhythms and cognitive
  processing}.
\newblock \emph{\bibinfo{journal}{European Journal of Neuroscience}}
  \textbf{\bibinfo{volume}{39}}, \bibinfo{pages}{705--719}
  (\bibinfo{year}{2014}).

\bibitem{haykin2005cocktail}
\bibinfo{author}{Haykin, S.} \& \bibinfo{author}{Chen, Z.}
\newblock \bibinfo{title}{The cocktail party problem}.
\newblock \emph{\bibinfo{journal}{Neural computation}}
  \textbf{\bibinfo{volume}{17}}, \bibinfo{pages}{1875--1902}
  (\bibinfo{year}{2005}).

\bibitem{cover2012Wiley}
\bibinfo{author}{Cover, T.~M.} \& \bibinfo{author}{Thomas, J.~A.}
\newblock \emph{\bibinfo{title}{Elements of information theory}}
  (\bibinfo{publisher}{John Wiley \& Sons}, \bibinfo{year}{2012}).

\bibitem{diehl2015Github}
\bibinfo{author}{Diehl, P.} \& \bibinfo{author}{Cook, M.}
\newblock \bibinfo{title}{stdp-mnist}.
\newblock
  \bibinfo{howpublished}{\url{https://github.com/peter-u-diehl/stdp-mnist}}
  (\bibinfo{year}{2015}).

\bibitem{mitchell2003Neuron}
\bibinfo{author}{Mitchell, S.~J.} \& \bibinfo{author}{Silver, R.~A.}
\newblock \bibinfo{title}{Shunting inhibition modulates neuronal gain during
  synaptic excitation}.
\newblock \emph{\bibinfo{journal}{Neuron}} \textbf{\bibinfo{volume}{38}},
  \bibinfo{pages}{433--445} (\bibinfo{year}{2003}).

\bibitem{ayaz2009JN}
\bibinfo{author}{Ayaz, A.} \& \bibinfo{author}{Chance, F.~S.}
\newblock \bibinfo{title}{Gain modulation of neuronal responses by subtractive
  and divisive mechanisms of inhibition}.
\newblock \emph{\bibinfo{journal}{Journal of neurophysiology}}
  \textbf{\bibinfo{volume}{101}}, \bibinfo{pages}{958--968}
  (\bibinfo{year}{2009}).

\bibitem{wilson2012Nature}
\bibinfo{author}{Wilson, N.~R.}, \bibinfo{author}{Runyan, C.~A.},
  \bibinfo{author}{Wang, F.~L.} \& \bibinfo{author}{Sur, M.}
\newblock \bibinfo{title}{Division and subtraction by distinct cortical
  inhibitory networks in vivo}.
\newblock \emph{\bibinfo{journal}{Nature}} \textbf{\bibinfo{volume}{488}},
  \bibinfo{pages}{343} (\bibinfo{year}{2012}).

\bibitem{amit1991Network}
\bibinfo{author}{Amit, D.~J.} \& \bibinfo{author}{Tsodyks, M.}
\newblock \bibinfo{title}{Quantitative study of attractor neural network
  retrieving at low spike rates: I. substrate -- spikes, rates and neuronal
  gain}.
\newblock \emph{\bibinfo{journal}{Network: Computation in neural systems}}
  \textbf{\bibinfo{volume}{2}}, \bibinfo{pages}{259--273}
  (\bibinfo{year}{1991}).

\end{thebibliography}
\pagebreak

\clearpage
\pagebreak
\section*{Methods}
\subsection*{Synopsis of the theoretical analysis}\label{sec:formulation}
Formal definitions and assumptions (see Supplementary Information, section \ref{sec:formulation_sup}), and the remainder of the complete derivation are
provided in the Supplementary Information. Here we provide its outline. First, we derive the probability distribution of the data, i.e. of future
observations given those of the past, and given the assumptions for continuity and randomness. For this, based on information theory
\cite{cover2012Wiley}, we show that the distribution of an object's future observations is a Gaussian centred at its last observation (Fig.
\ref{fig:data_model}\textbf{b}, and Supplementary Information, section \ref{sec:pdf}). The combination of this with the stationary unconditional-on-history probability
distribution that accounts for the possibility of novel object introductions, composes the probability density function (PDF) of an observed sample
conditional on its past. Using the unconditional PDF and Bayes' theorem, the conditional
PDF is expressed as a mixture of likelihood functions each attributed to a possible "hidden cause", i.e. label of the observation. Second, we model these functions
using distributions that (a) suit the assumptions of the input, (b) are parameterized such that analytical derivation of the optimal parameters is
possible by minimizing the Kullback-Leibler divergence of the model from the data, and (c) can be described as neuron activation functions. The
derived optimal parameter vectors are determined by the mean values of the likelihood functions. By minimizing the Kullback-Leibler divergence of the
model from the data PDF, we show that its components' optimal parameters are given by the means of the data distribution's components (Supplementary
Information sections \ref{sec:model}, \ref{sec:model_lin}). Third, we show that these parameters, as means of distributions, can be regarded as
centroids of clusters. In addition, we show that, to optimize the model, these change with each incoming observation, with each inferred posterior,
and with time, rendering the parameter optimization interpretable as an online clustering algorithm (Supplementary Information, section
\ref{sec:clustering_sup}). Fourth, we extend the model to operate with stochastic measurements of the observations, estimating the value of the
underlying variable as a weighted average of past stochastic samples. The weighting of the sample history is derived from the random dynamics of
object transformations, and of observed object replacement. This allows the model's optimization algorithm to use spike-based input encoding
(Supplementary Information, section \ref{sec:stoch}). Fifth, we show that the dynamics of the cluster centroids have in fact the form of short-term
Hebbian plasticity, and we show how a neural network with this plasticity can realize this model and its optimization (Supplementary Information
section \ref{sec:neuro_exp}). Last, we conclude by deriving a fully spike-based machine learning model with spiking outputs, which maintains the sought generative
model of the data (Supplementary Information, section \ref{sec:snn_exp}). We derive the results for two models, one based on stochastic exponential
and the other based on leaky integrate-and-fire (LIF) neurons.

\subsection*{The OMNIST dataset}\label{sec:OMNIST}
The OMNIST testing dataset is a single continuous sequence derived from the 10,000 images of the MNIST handwritten digit testing dataset. It consists of subsequences, each
generated based on an MNIST image. An example of this type of subsequence is shown in Fig. \ref{fig:wta}\textbf{a}, albeit shortened for the figure, compared
to the subsequences of the actual testing and training OMNIST sets. Each MNIST image is shown  while an occluding square begins to move vertically,
from the top towards the bottom of the frame, and decelerates until it stops at a certain height of the frame. After a random number of frames, the
digit and the occlusion disappear and are replaced by a random noisy object that varies between frames. The noise is generated randomly at each
frame, and repeated for four frames. Subsequently, a different MNIST frame is presented, and the process is repeated for all MNIST testing digits,
resulting in a video sequence of 164,915 total frames. Similarly, an OMNIST training video, used to train the recurrent ANNs, is derived from the separate set of 60,000 images of the
MNIST training dataset, and results in a length of 990,089 frames.

The specific geometry and dynamics of the occlusion and the noise were arbitrarily chosen as follows. Each MNIST frame is repeated for $11\leq
n_f\leq 14$ times, where $n_f$ is randomly and uniformly chosen for each frame. At the first frame for each digit, the square occluding object is
located immediately outside the visible part of the square image, so that the first frame of each digit's subsequence is identical to the original
MNIST frame. Subsequently, the occluding object begins moving downwards into the frame with a speed of 3 pixels per frame. The occlusion's speed is
reduced to 1 pixel per frame when the occlusion is 18 pixels inside the image. When the occlusion reaches 19 pixels, it stops moving until the
completion of the sequence of $n_f$ frames for this digit. The noisy object is a rectangle with a width of 15 and a height of 12 pixels, chosen to
roughly match the central frame region that the MNIST digits cover. It is generated at each frame as follows. From the 784 pixels of the whole frame,
200 are randomly selected, and, from those, the ones outside the limits of the rectangle are excluded. Each selected pixel is assigned a uniformly
random brightness value between 0 and 255.

The dataset's design was chosen to best manifest the strengths of the ST-STDP-enabled SNN, compared to other algorithms, under the constraints of a
very simple architecture and training, like our unsupervised WTA. The untransformed dataset must be able to be handled by a single-trained-layer,
unsupervised network, thus MNIST was chosen because it fulfils this requirement. While hand-written character recognition using the derived OMNIST dataset remains,
as demonstrated in the main text, a difficult task for the compared ANNs, this design makes the task achievable by the SNN. Even more difficult, and
less arbitrary, datasets could likely be handled by networks combining ST-STDP with more complex architectures. For example, multilayer networks
could be trained in a supervised manner, to extract more abstract features than individual pixels, and to recognize patterns in harder training sets
than MNIST. On-line adaptivity through ST-STDP could be added subsequently to such a network, during inference on the transforming testing data.

\subsection*{Simulations}
\textbf{SNN structure and operation}: The SNN was structured and operated during training according to \cite{diehl2015FCN} and the associated code in
\cite{diehl2015Github}. Each pixel of the 28x28 input image corresponded to an input neuron firing Poisson spike trains with a rate proportional to
the brightness of the pixel. The 784 input neurons were connected through excitatory synapses with 400 output neurons in an all-to-all fashion. In
addition, each output neuron was connected to a different inhibitory neuron, and each inhibitory neuron was connected to all other output neurons
through inhibitory synapses, implementing lateral inhibition among output neurons. Synapses were conductance-based as opposed to current-based, so
each excitatory input spike to a neuron caused the excitatory conductance of the neuron to increase by the corresponding synapse's efficacy, while
inhibitory spikes acted on an inhibitory conductance. The conductances decayed exponentially. Each conductance, multiplied by the difference of the
neuron's membrane potential from the synapse's resting potential, was added to an excitatory or inhibitory current respectively for excitatory and
inhibitory synapses, and the total current into the neuron changed its membrane potential linearly with time, while the membrane potential also
decayed exponentially towards the neuron's resting potential, at all times. Once the neuron's firing threshold was reached, the neuron fired an
output spike, and was reset to its reset potential. A refractory period prevented spikes from being produced for a few milliseconds after each output
spike.

\textbf{Deviations of SNN simulations from theory}: To test the model, we chose to simulate it by using LIF neurons, which are a convenient model
suitable for simulations or implementations with efficient electronic circuits. The theoretically optimal model requires normalized inputs,
continuous normalization of the synaptic efficacies, an additional intrinsic neuronal plasticity continuously updating the neuronal threshold, and
divisive inhibition in the case of LIF neurons, as opposed to subtractive. Nevertheless, in our simulations we used a simplified model without these
particular idealities, using only conventional neuromorphic primitives with the addition of ST-STDP. This demonstrated that ST-STDP itself is
powerful, can be robust to the absence of the theoretical idealities, and is thus suitable for simple neuromorphic hardware or simulations. In
particular, omitting the additional neuronal memory mechanism that the intrinsic neuronal plasticity would introduce, not only simplified the
simulation, but also allowed us to perform a more direct contrast of ST-STDP with other models such as RNNs or LSTMs. These networks also include a
decaying neuronal memory, implemented by the self-recurrency, and analogous to an intrinsic neuronal plasticity. Therefore, not including this aspect
in the simulated SNN allows us to attribute the demonstrated benefits unambiguously to the plasticity of synapses itself, and to distinguish them
from potential advantages due to a sheer diversity of short-term memory mechanisms (neuronal and synaptic).

\textbf{SNN training}: During training, short-term plasticity was not introduced, so the component $F$ of the synaptic efficacy $G$ was inactive, and
fixed at zero, and the efficacies were equivalent to the weights $W$, which were initialized with random values. Each of the 60,000 images from the
MNIST training set was input into the network as a \unit[350]{ms} spike train, followed by a \unit[150]{ms} resting time. The training set was
presented over a single epoch, i.e. only once. Long-term plasticity in the form of STDP was active during training. Weights were depressed when a
presynaptic spike followed a postsynaptic one, by an amount proportional to a trace decaying exponentially from the time of the single last
postsynaptic spike. Potentiation events were governed by a triplet rule, where a presynaptic spike followed by two postsynaptic spikes caused an
increase in the weight proportionally to two traces: one trace decaying exponentially from the time of the last presynaptic spike, and also another
trace decaying exponentially from the last postsynaptic spike. At each time step of the training simulation the weights of each neuron were
normalized through division by their sum. In addition, a homeostatic mechanism was changing each neuron's intrinsic excitability, through an adaptive
threshold increasing every time the neuron fires, and exponentially decaying to its resting value at all times except during the \unit[150]{ms}
resting phases between input presentation. If the presentation of an image produced fewer than five output spikes in total, then the example was
repeated with an increased intensity, i.e. by increasing the input firing rates.
The parameters used during training were those of \cite{diehl2015FCN, diehl2015Github}.

\textbf{SNN inference}: To test the performance of the network after its training, we first froze the weights and the firing thresholds. Then we
associated each neuron with a label, which was taken to be the label to which the neuron was most responsive across the last presentation of the
60000 training examples. That was the only part where labels were used, therefore the network's weights were learned in a completely unsupervised way. Subsequently, we tested the recognition of the MNIST and the OMNIST testing
datasets. In the case of MNIST, each tested image was recognized as belonging to the digit class whose corresponding output neurons had the highest
average firing rate during the 350 ms of the image's presentation. Similarly to the training phase, the examples were repeated with increasing
intensity until at least five output spikes were produced. For OMNIST, we followed the same testing protocol, but we removed the resting phase
between frames, and we did not repeat the frames that caused the network to produce few or no output spikes. This made the implementation ready for
future operation with real streaming input, as there was no need to store the data for repetition. The frames that produced no spikes were assigned
an 11th inferred label, corresponding to the noisy object of the OMNIST dataset. First we tested the network with ST-STDP turned off, as during
training, i.e. with the efficacies $G$ fixed to equal the weights $W$. Separately, we tested it after enabling ST-STDP, acting on the short-term
component $F$ of the efficacies. ST-STDP was implemented as short-term facilitation only, by using a trace each keeping track of the recent
presynaptic history. The trace decayed exponentially, and increased with each presynaptic spike. With each postsynaptic spike, the synaptic efficacy
increased by an amount, proportional to the value of the trace at that time, by a parameter $\gamma$ (Eq. \ref{eq:F}), which was fixed per synapse
and was dependent on the synapse's fixed resting weight $W$ (see Supplementary Information, section \ref{sec:weightdependent}). Subsequently, the
short-term component $F$ decayed exponentially towards zero, and therefore the efficacy $G$ towards its equilibrium value equal to the weight $W$. The parameters used during inference are given in Table 1.

\begin{table}[htbp]
	\centering
	\begin{tabular}{lc}
		\multicolumn{1}{c}{\textbf{Parameter}} & \textbf{Value} \\
		\midrule
		\midrule
		\textbf{ST-STDP} &  \\
		\midrule
		ST-STDP: STP time constant  ($1/\lambda$ in Eq. \ref{eq:lambda} \& Fig. \ref{fig:clustering}c) & 300 ms \\
		ST-STDP: STDP kernel's max ($\gamma$ in Eq. \ref{eq:F} \& Fig. \ref{fig:clustering}b) & 0.7 \\
		ST-STDP: STDP kernel time const. (Fig. \ref{fig:clustering}b) & 20 ms \\
		&  \\
		\textbf{Neuronal soma firing} &  \\
		\midrule
		Refractory period of excitatory neurons  & 5 ms \\
		Refractory period of inhibitory neurons & 2 ms \\
		Resting potential of excitatory neurons & -65 mV \\
		Resting potential of inhibitory neurons & -60 mV \\
		Membrane potential time const. of exc. neurons & 100 ms \\
		Membrane potential time const. of inh. neurons & 10 ms \\
		Reset potential of excitatory neurons & -65 mV \\
		Reset potential of inhibitory neurons & -45 mV \\
		&  \\
		\textbf{Synaptic conductance} &  \\
		\midrule
		Resting excitatory conductance & 0 mS \\
		Resting inhibitory conduct. of exc. neurons & -100 mS \\
		Resting inhibitory conduct. of inh. neurons & -85 mS \\
		Excitatory conductance time constant & 1 ms \\
		Inhibitory conductance time constant & 2 ms \\
	\end{tabular}%
\caption{\textbf{SNN simulation parameters during inference}}
	\label{tab:addlabel}%
\end{table}%

\textbf{ANN training and testing}: 11 output units were used in each ANN for OMNIST, to cover the ten digits and the additional noisy class of inputs. The multilayer perceptron (MLP) we used was structured to use one input layer of 784 neurons corresponding to the
image pixels, one fully connected hidden layer of 40 rectified linear units (ReLU), a layer of 11 linear units, and a final soft-max operation. We trained the
network using Adam, using cross-entropy as the loss function, with a learning rate of 0.001, and batches of 100 training examples. The training set
was shuffled between each full presentation of the set, and 20 epochs reached convergence. The structure of the convolutional network (CNN) comprised
a 28x28 input layer, a first convolutional layer computing eight features using a 5x5 filter with ReLU activation, a first max pooling layer with a
2x2 filter and stride of 2, a second convolutional layer of eight features using a 5x5 ReLU filter, a second max pooling layer with a 2x2 filter and
stride of 2, a densely connected layer of 400 ReLU neurons and a final dense layer of 11 outputs, passed through a soft-max operation. The network
was trained to minimize cross-entropy, with gradient descent and dropout (dropout probability of 0.4), with a learning rate of 0.001. The MLP and the
CNN were each trained in two separate manners, namely once on the MNIST and once on the OMNIST training set, before testing on the OMNIST testing
set. When the 60000 MNIST training images were used for training, the training set was augmented with an additional 11th class of 6000 noisy frames
taken from the OMNIST training set. We also trained a fully recurrent neural network (RNN) on the OMNIST training video. The RNN consisted of 784
input neurons, 400 hidden RNN units, fully connected to 11 output units, on which soft-max was applied. To train it, we used gradient descent with a
learning rate of 0.001. Training inputs were batches of 128 sequences of a fixed length. The sequences were subsequences from the OMNIST video, and
their length was equal to the number of frames of the longest digit's observation, including its subsequent noisy frames, within the OMNIST video,
i.e. 21 time steps, i.e. frames. The RNN was stateful, such that the recurrent units' states at the beginning of a batch were those at the end of the
previous batch. Each batch was a continuation of the previous one from the OMNIST video. The sequence of 990089 total training frames was presented
twice during training. The minimized loss function was the cross entropy averaged over the 21 time steps of an input sequence. A long short-term
memory (LSTM) network was also trained on the same task. The network consisted of 784 inputs, and 90 hidden LSTM cells fully connected to 11 output
units on which soft-max was applied. This resulted in a network size of 315990 parameters, slightly larger than the 314000 trained synaptic
efficacies and firing thresholds in the SNN of 400 excitatory neurons. The training procedure was the same as for the RNN.

\pagebreak
\beginsupplement
\setlength{\abovedisplayskip}{3pt}
\setlength{\belowdisplayskip}{3pt}
\title{\Huge \bfseries Supplementary Information}
\author{}
\maketitle
\vspace{3ex}
\section{Problem formulation and outline of our approach}
\label{sec:formulation_sup}
\begin{figure}[h]
	\centering
	\begin{tabular}{c}
		\includegraphics[width = 89mm]{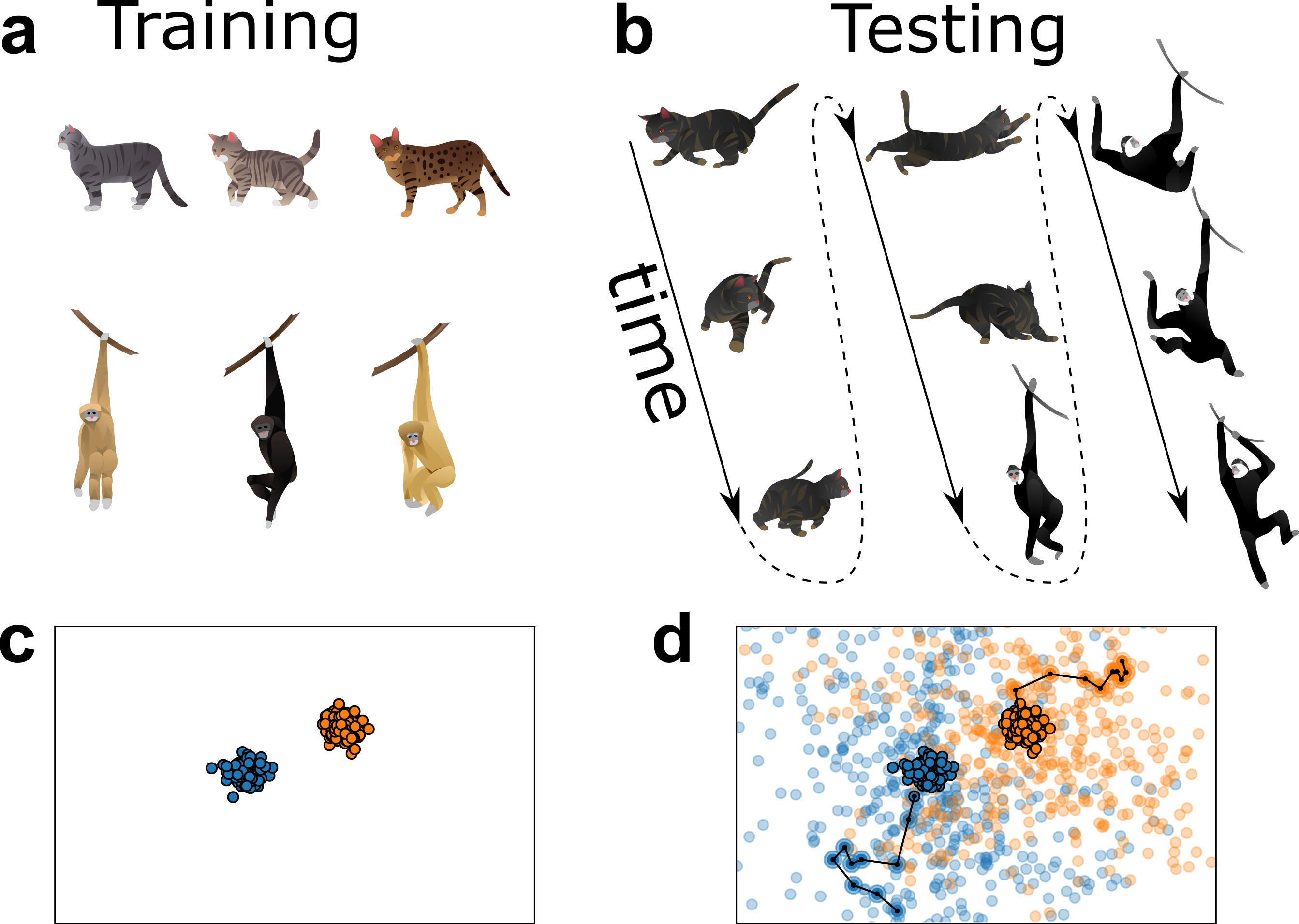}
	\end{tabular}
	\caption{\textbf{Schematic depicting the type of problems addressed here.} \textbf{a}, Training with typical static examples. \textbf{b}, During testing, the environment transforms continuously but randomly. \textbf{c}, Training examples may be easily separable and predictable. \textbf{d}, Testing patterns are highly unpredictable and non-separable, unless the continuity of objects in time (black line paths) is considered.} \label{fig:icml}
\end{figure}
We shall describe a quite generic statistical category of spatiotemporal data. Here we first summarize the reasoning, and then we provide
mathematical definitions and derivations. The employed data model is a formalization of the postulate that in natural input streams, due to the
continuity of natural environments, each instance of the observed input is likely a transformed but similar repetition of a recent instance, unless
the observer turns his attention to a different object. The data at each time instance, for example an image of an animal, is associated with a
category that is unknown, i.e. a hidden label or "cause", e.g. the animal's species. A training set with typical examples of such objects is
available, e.g. frontal images of animals (Fig. \ref{fig:icml}, \textbf{a} and \textbf{c}). After training, each new object that is encountered is typical, from the same
distribution as the training set, but is subsequently characterised by temporal dynamics, e.g. animals moving in a field (Fig. \ref{fig:icml}, \textbf{b} and \textbf{d}, and
Fig. \ref{fig:data_model}\textbf{a}). Because of the continuity of time and space, each observed object is likely to be observed at a point soon after as
well, but unlikely to be the one observed at a point far in the future. Each object's properties can change randomly but are continuous in time.

\subsection{Assumptions}
\label{sec:assumptions}
Specifically, we make the following assumptions.
\begin{enumerate}
	\setlength\itemsep{0.em}
	\item We model the environment as a set $E$ of objects: $E=\{_iO, \forall i\in\mathbb{N}\}$.
	For example, the environment may be the images of the MNIST handwritten digits, in which case the testing set contains $10,000$ examples of the infinite possible objects, or the space of three-dimensional objects, or a dictionary of word vocalizations etc.
	\item Each object $_iO$ in the environment is associated with a hidden label or "cause" or "class" $_iC$ from a finite set of $K$ possible labels, where $K$ is known:
	$_iC \in \{C^{(k)},\, \forall k \leq K\in \mathbb{N}\}.$
	E.g., in the case of the 10 digits of the MNIST dataset, $K=10$.
	\item \label{as:color} Each object $_iO$ in the environment corresponds at time $t$ to an $n$-dimensional vector $_i\boldsymbol{O}_t\in \mathbb{R}^n$. For example, in the MNIST dataset, objects have n=784 dimensions that correspond to the image pixels.
	\item \label{as:wiener} Each object vector $_i\boldsymbol{O}_t$ undergoes a transformation that is random and has continuous but unknown paths in time, e.g. the transformations due to moving occlusions in the OMNIST dataset.
	\item The environment is observed at discrete time points $t\in \{t_1, t_2,...,t_T\}, T \in \mathbb{N}$. The resulting observation at each time instance $t$ is an $n$-dimensional vector $\boldsymbol{X}_t\in \mathbb{R}^n$.
	\item \label{as:poisson1} At each time instance $t$, exactly one object $\prescript{}{i}{O}$ is observed from the environment, be it $O_t=\prescript{}{i}{O}$. A stationary Poisson process with a mean rate $\alpha$ triggers changes of the observed object $\prescript{}{i}{O}$ into a different one, $\prescript{}{j}{O}$, at random times.
	\item \label{as:f} An object $\prescript{}{l}{O}$ last observed at time $t_i$, i.e. $O_{t_i}=\prescript{}{l}{O}$, is likely to also be observed at time $t>t_i$, i.e. $O_t=O_{t_i}$, with a probability that changes with time $t-t_i$ and ultimately decays to zero.
	Let $A_{t_i}$ be the event that the object $\prescript{}{l}{O}$ last observed at time $t_i$ is also observed at time $t$.
	Then, specifically, a decaying function $f(t-t_i)$ determines the odds $\prescript{}{l}{p}$ in favour of the observation of the object $\prescript{}{l}{O}$ such that:
	\begin{equation}
		P\left(A_{t_i}\right)
		=\frac{\prescript{}{l}{p}}{\sum_{m=1}^\infty \prescript{}{m}{p}}=\frac{f(t-t_i)}{\sum_{m=1}^\infty \prescript{}{m}{p}}
	\end{equation}
	\begin{equation}
		\text{and } \lim\limits_{t\rightarrow +\infty}f(t)=0.
	\end{equation}
	\item \label{as:poisson2} At each time instance, the odds that the visible object will be replaced by a novel, previously unseen object, which we call event $B$, are constant and equal to $\beta$, i.e. \begin{equation}
		P(B)=\frac{\beta}{\sum_{m=1}^\infty \prescript{}{m}{p}}.\end{equation}
	
	The events $B$ and $A_{t_i}$ for all $1\leq t_i<t$ cover all possibilities, i.e. either a possible introduction of a novel object or a repetition of one previously observed object. Therefore, from Assumption \ref{as:poisson2} combined with Assumption \ref{as:f} it is \begin{equation}
		\sum_{m=1}^\infty \prescript{}{m}{p}=\sum_{i=1}^{T}f(t-t_i)+\beta.\end{equation}
	\item \label{as:uncond} The probability distribution of the novel objects is unknown, but it is independent from the history:
	\begin{equation}
		p(\boldsymbol{X}_t|B,\boldsymbol{X}_{t_i},C_{t_j})=p(\boldsymbol{X}_t|B), \forall i,j: t_i\neq t\neq t_j,
	\end{equation}and it is the mixture of the history-independent distributions attributed to each of the $K$ labels $C^{(k)}$:
	\begin{equation}
		p(\boldsymbol{X}_t|B)=\sum_{k=1}^{K}p(\boldsymbol{X}_t|C^{(k)}_t,B)P(C^{(k)}_t|B).
		\label{eq:pstar}
	\end{equation}
\end{enumerate}
The task of the observer, i.e. the machine learning problem that is addressed (Fig. \ref{fig:data_model}\textbf{a}, Fig. \ref{fig:icml}), is to maintain an up-to-date generative model of future observations, given the past ones, under the Assumptions 1-9 stated above.
\subsection{Sketch of the derivation}
First, based on the assumptions, we derive the probability distribution of future observations given the past ones, which has two components, one conditional on the past, and one unconditional, i.e. independent, probability distribution component describing introductions of novel objects. The history-dependent part is derived conditional on the history of past observations and of past posterior probabilities that infer the hidden class label. Using this distribution and Bayes' theorem we also get the analytical form of posterior probability for each class given one observation, its history, and past posteriors. These past posterior probabilities, however, are still unknown. To estimate them, we formulate a parametrized generative mixture model of the probability distribution, and we find the analytical form of its optimal parameters given the past inputs, i.e. the parameters that minimize the cross-entropy between the model and the data. The result is that the optimal parameter vector of each component of the mixture model is determined by the mean of each component of the data distribution. These parameters can be therefore regarded as centroids of clusters, with positions that change with each new observation and with time, and we describe this derived novel clustering algorithm. We notice that the dynamics of the cluster centroids are in fact identical to short-term Hebbian plasticity, and we show how a neural network with this plasticity can realize this optimal clustering. Lastly, we extend the model to operate with stochastic measurements of the observations, which allows the algorithm to use spike-based input encoding, and we conclude by deriving a fully spike-based machine-learning model, i.e. with spiking outputs too, that maintains the sought generative model of the data.
\section{The probability distribution of the data}
\label{sec:pdf}
Let $\boldsymbol{X}_t=\boldsymbol{X}_{T+1}$ be the $T+1^{st}$ sample observed at time $t=t_{T+1}>t_T$, and its associated hidden label be $C_t=C_{T+1}$.
Let $\boldsymbol{\mathcal{X}}_t$ be the sequence of the samples preceding $\boldsymbol{X}_t$.
Let $\mathcal{C}_t$ be a sequence of random variables, such that these random variables represent the conditional probabilities $P(C^{(k)}_{t_i}|\boldsymbol{X}_{t_i},\boldsymbol{\mathcal{X}}_{t_i},\mathcal{C}_{t_i}), \, \forall i: t_1<t_i<t,$ of the hidden labels preceding $C_t$, where $C^{(k)}_t$ is the event that the hidden label of the observation at time $t$ is $C^{(k)}$. In case the label of the observation were known to be $C^{(k)}$, then $P(C^{(k)}_{t_i}|\boldsymbol{X}_{t_i},\boldsymbol{\mathcal{X}}_{t_i},\mathcal{C}_{t_i})=1$ and $P(C^{(j)}_i|\boldsymbol{X}_{t_i},\boldsymbol{\mathcal{X}}_{t_i},\mathcal{C}_{t_i})=0, \forall j\neq k$. Importantly, this way of referring to the labels allows us to apply it also to labels that are hidden, due to its probabilistic form. It is also practical for the derivation of the statistical model of the data in the next section.
\begin{thm}
	\label{thm:pdf}
	Given the history $\boldsymbol{\mathcal{X}}_t$ and $\mathcal{C}_t$ of a future sample $\boldsymbol{X}_t$, the probability distribution of this sample is
	\begin{multline}
		p(\boldsymbol{X}_t|\boldsymbol{\mathcal{X}}_t, \mathcal{C}_t)\\
		=\frac{1}{\sum_{i=1}^T f(t-t_i)+\beta}
		\cdot\sum_{k=1}^{K}\Bigg[\sum_{i=1}^{T}\mathcal{N}\left(\boldsymbol{X}_{t_i},\boldsymbol{\Sigma}_{t_i,t}\right)P(C^{(k)}_{t_i}|\boldsymbol{X}_{t_i},\boldsymbol{\mathcal{X}}_{t_i},\mathcal{C}_{t_i})\\ \cdot f(t-t_i) +\beta p(\boldsymbol{X}_t|C^{(k)}_t,B)P(C^{(k)}_t|B)\Bigg],\label{eq:data_generative_model}
	\end{multline}
	where $\mathcal{N}\left(\boldsymbol{X}_{t_i},\boldsymbol{\Sigma}_{t_i,t}\right)$ is a normal distribution centred at $\boldsymbol{X}_{t_i}$, and with a covariance matrix $\boldsymbol{\Sigma}_{t_i,t}$ that depends on the time interval $t-t_i$.
\end{thm}
\begin{proof}
	\begin{multline}
		p(\boldsymbol{X}_t|\boldsymbol{\mathcal{X}}_t,\mathcal{C}_t)   =\sum_{i=1}^{T}p(\boldsymbol{X}_t|\boldsymbol{\mathcal{X}}_t,\mathcal{C}_t,A_{t_i})P(A_{t_i})
		\\+p(\boldsymbol{X}_t|\boldsymbol{\mathcal{X}}_t,\mathcal{C}_t,B)P(B)\\
		=\frac{\sum_{i=1}^{T}p(\boldsymbol{X}_t|\boldsymbol{\mathcal{X}}_t,\mathcal{C}_t,A_{t_i})f(t-t_i)+p(\boldsymbol{X}_t|B)\beta}{\sum_{i=1}^T f(t-t_i)+\beta}  \label{eq:broken_down}
	\end{multline}
	The second term $p(\boldsymbol{X}_t|B)$ is the unconditional-on-history probability which is fixed and independent of time or history, as assumed in Section \ref{sec:assumptions}, Assumption \ref{as:uncond}.
	We aim to expand the first term also as a function of the known assumptions and of the visible aspects of the data. We decompose this distribution into the mixture of its $K$ classes:
	\begin{multline}
		\sum_{i=1}^{T}p(\boldsymbol{X}_t|\boldsymbol{\mathcal{X}}_t,\mathcal{C}_t,A_{t_i})f(t-t_i)=
		\\=\sum_{k=1}^{K}\sum_{i=1}^{T}p(\boldsymbol{X}_t|C^{(k)}_t,\boldsymbol{\mathcal{X}}_t,\mathcal{C}_t,A_{t_i})P(C^{(k)}_t|\boldsymbol{\mathcal{X}}_t,\mathcal{C}_t,A_{t_i})\\\cdot f(t-t_i) .\label{eq:intprod}
	\end{multline}
	The first term of the product in Eq. \ref{eq:intprod} is
	\begin{equation}
		p(\boldsymbol{X}_t|C^{(k)}_t,\boldsymbol{\mathcal{X}}_t,\mathcal{C}_t,A_{t_i})
		=p(\boldsymbol{X}_t|\boldsymbol{\mathcal{X}}_t,A_{t_i}),
		\label{eq:history_only}
	\end{equation}
	because the probability distribution of the new sample in event $A_{t_i}$ depends only on the past samples, not the labels.
	
	In addition, the second term in Eq. \ref{eq:intprod}, due to the definition of $A_{t_i}$, is
	\begin{equation}
		P(C^{(k)}_t|\boldsymbol{\mathcal{X}}_t,\mathcal{C}_t,A_{t_i})=P(C^{(k)}_t|\boldsymbol{\mathcal{X}}_{t_i},\mathcal{C}_{t_i},A_{t_i}). \label{eq:delta_exp}
	\end{equation}
	
	We will be using $P_{t_i}(C^{(k)})\coloneqq P(C^{(k)}_t|\boldsymbol{\mathcal{X}}_{t_i},\mathcal{C}_{t_i},A_{t_i})$ as the shorthand form of $P(C^{(k)}_t|\boldsymbol{\mathcal{X}}_{t_i},\mathcal{C}_{t_i},A_{t_i})$.
	
	In a key step, through assumption \ref{as:wiener} we derive a prior belief about the expected distribution of the increments of an object
	$\prescript{}{j}{} \boldsymbol{O}$ as a function of time, from the object's random but continuous dynamics. In the absence of additional information,
	the increment $\delta \prescript{}{j}{} \boldsymbol{O}_{t_i,t}\coloneqq\prescript{}{j}{} \boldsymbol{O}_{t}-\prescript{}{j}{} \boldsymbol{O}_{t_i}$ is unbiased,
	i.e. the mean of its distribution is zero. The temporal continuity of the objects implies that the variance of this distribution increases with time
	and at each point in time it is specified. Based on this, the maximum-entropy estimate probability distribution of this increment is a Gaussian
	\cite{cover2012Wiley}. Therefore, according to the principle of maximum entropy, this Gaussian is the best estimate for the distribution. Taken
	together, the Gaussian aspect and the increasing variance imply that, in expectation, the dynamics of each object are a Wiener process, which can be used as a Mean Field Approximation of the
	objects' dynamics. The Wiener process that governs the evolution of the objects implies that a Wiener process also describes the evolution of the
	object's observed subset of features $\boldsymbol{X}_{t}$ between two observations. Therefore, $A_{t_i}$, i.e. knowing that the object at time $t$ was last
	observed at time $t_i$, implies:
	\begin{equation}
		A_{t_i}
		\implies \boldsymbol{X}_{t}\sim \mathcal{N}\left(\boldsymbol{X}_{t_i},\boldsymbol{\Sigma}_{t_i,t}\right),\label{eq:Wiener}
	\end{equation}
	where $\mathcal{N}\left(\boldsymbol{X}_{t_i},\boldsymbol{\Sigma}_{t_i,t}\right)$ is a normal distribution centred at $\boldsymbol{X}_{t_i}$, and with a covariance matrix that depends on $t-t_i$.
	By using Eq. \ref{eq:delta_exp} and Eq. \ref{eq:Wiener}, Eq. \ref{eq:intprod} becomes
	\begin{equation}\sum_{i=1}^{T}p(\boldsymbol{X}_t|\boldsymbol{\mathcal{X}}_t,\mathcal{C}_t,A_{t_i})f(t-t_i)
		=\sum_{k=1}^{K}\sum_{i=1}^{T}\mathcal{N}\left(\boldsymbol{X}_{t_i},\boldsymbol{\Sigma}_{t_i,t}\right)  \, P_{t_i}(C^{(k)}) f(t-t_i)  .\label{eq:intprod2}
	\end{equation}
	By using Eq. \ref{eq:intprod2} and Eq. \ref{eq:pstar}, Eq. \ref{eq:broken_down} proves the theorem.
\end{proof}
Theorem \ref{thm:pdf} yields a distribution, which, given past observations $\boldsymbol{X}_{t_i}$, past labels $C_{t_i}$, as well as distributions $\mathcal{N}\left(\boldsymbol{X}_{t_i},\boldsymbol{\Sigma}_{t_i,t}\right)$ and $p(\boldsymbol{X}_t|C^{(k)}_t)$, can generate new samples $\boldsymbol{X}_t$ from the actual data.
On the other hand, through the same distribution, for any given sample and its history one can infer the probabilities associated with each possible hidden label of the sample.
\begin{thm}
	For any given sample $\boldsymbol{X}_t$, and history $\boldsymbol{\mathcal{X}}_t$ and $\mathcal{C}_t$, the probabilities associated with each possible value $C^{(k)}$ of the hidden class label $C_t$ that led to the generation of the observation are inferred as
	\begin{multline}
		P(C^{(k)}_t|\boldsymbol{X}_t,\boldsymbol{\mathcal{X}}_t,\mathcal{C}_t)=\frac{p(\boldsymbol{X}_t,C^{(k)}_t|\boldsymbol{\mathcal{X}}_t,\mathcal{C}_t)}{\sum_{l=1}^{K}p(\boldsymbol{X}_t,C^{(l)}_t|\boldsymbol{\mathcal{X}}_t,\mathcal{C}_t)}, \label{eq:data_inference_model_visible_causes}
	\end{multline}
	\begin{multline}
		\text{where } p(\boldsymbol{X}_t,C^{(k)}_t|\boldsymbol{\mathcal{X}}_t,\mathcal{C}_t)
		\\=\frac{1}{\sum_{i=1}^T f(t-t_i)+\beta} \cdot\sum_{i=1}^{T}\Big[\mathcal{N}\left(\boldsymbol{X}_{t_i},\boldsymbol{\Sigma}_{t_i,t}\right)P_{t_i}(C^{(k)}) f(t-t_i)
		\\ +\beta p(\boldsymbol{X}_t|C^{(k)}_t,B)P(C^{(k)}_t|B)\Big]. \label{eq:data_joint_model_visible_causes}
	\end{multline}
\end{thm}
This is easily shown through Bayes' theorem applied to the distribution of Eq. \ref{eq:data_generative_model}.
But, to compute $p(\boldsymbol{X}_t|\boldsymbol{\mathcal{X}}_t, \mathcal{C}_t)$ and $P(C^{(k)}_t|\boldsymbol{X}_t,\boldsymbol{\mathcal{X}}_t,\mathcal{C}_t)$ with Eq. \ref{eq:data_generative_model} and Eq. \ref{eq:data_inference_model_visible_causes}, the past probabilities $P_{t_i}(C^{(k)})$ are needed and are unknown. To estimate these, we will make a parametrized model of the data distribution, and find its parameters such that it is as similar as possible to the real distribution.
\section{Modelling the data with a mixture of exponentials}
\label{sec:model}
\subsection{The mixture model and its optimal parameters}
\label{sec:optimal_mean}
\begin{defn}
	\label{def:model}
	Consider a mixture model distribution $q$:
	$q(\boldsymbol{X}_t)=\sum_{k=1}^{K}q(\boldsymbol{X}_t|C^{(k)}_t)\,Q(C^{(k)}_t),$
	approximating some data distribution $p$ that is also a mixture of K components.
	We choose a mixture of exponentials and we parametrize $Q(C^{(k)}_t;G^{(k)}_0)$ also as an exponential, specifically:
	\begin{equation}
		q(X_{t}^{(j)}|C^{(k)}_t;G^{(jk)})=e^{G^{(jk)}\cdot \frac{X_{t}^{(j)}}{||\boldsymbol{X}_t||}},\, \forall j>0, k \label{eq:g_param}
	\end{equation}
	\begin{equation}
		Q(C^{(k)}_t;G^{(0k)})=e^{G^{(0k)}},\,\forall k.\label{eq:g0_param}\end{equation}
	
	In addition, the parameter vectors are subject to the normalization constraints:
	$||\boldsymbol{G}^{(k)}||=1,\, \forall k$,
	and
	$
	\sum_{k=1}^{K}e^{G^{(0k)}}=1. \label{eq:norm_0}$
\end{defn}

Assuming that the dimensions $X_{t}^{(j)}$ of the observed variable $\boldsymbol{X}_t$ are conditionally independent from each other, then $p(\boldsymbol{X}_t)=\prod_{j=1}^{n}p(X_{t}^{(j)})$, so the model we have chosen is a reasonable choice because it factorizes similarly:
\newline $q^{(k)}\coloneqq q(\boldsymbol{X}_t|C^{(k)}_t; \boldsymbol{G}^{(k)})=\prod_{j=1}^{n}q(X_{t}^{(j)}|C^{(k)}_t;G^{(jk)})=e^{\sum_{j=1}^{n}G^{ (jk)}\frac{X_{t}^{(j)}}{||\boldsymbol{X}_t||}}=e^{u^{(k)}_t}, \label{eq:multinomial}
$
where $u^{(k)}_t=\frac{ \boldsymbol{G}^{(k)}\cdot \boldsymbol{X}_t}{|| \boldsymbol{G}^{(k)}||\cdot||\boldsymbol{X}_t||}$, i.e. the cosine similarity of the two vectors.
\begin{thm}\label{thm:optimal}
	The optimal parameters of such a mixture model are
	\begin{equation}
		\prescript{}{opt}{}G^{(0k)}=\ln P(C^{(k)}_t)
		\label{eq:G0}
	\end{equation}
	\begin{equation}
		\text{and } \prescript{}{opt}{}\boldsymbol{G}^{*(k)}=\frac{ \prescript{}{opt}{}\boldsymbol{G}^{(k)}}{|| \prescript{}{opt}{}\boldsymbol{G}^{(k)}||}=\frac{E_{p^{(k)}}\left[\boldsymbol{X}_t\right]}{||E_{p^{(k)}}\left[\boldsymbol{X}_t\right]||}, \label{eq:Gstar}
	\end{equation}
	\begin{equation}
		\text{where } \prescript{}{opt}{}\boldsymbol{G}^{(k)}=c\cdot E_{p^{(k)}}\left[\boldsymbol{X}_t\right], c\in\mathbb{R}, \label{eq:G}
	\end{equation}
	\begin{equation}
		\text{and } p^{(k)}\coloneqq p(\boldsymbol{X}_t|C^{(k)}_t)
	\end{equation}
	for every $k$.
\end{thm}
\begin{proof}
	
	The model $q$ is optimal from a minimum cross-entropy perspective if its parameters equal those parameters $\boldsymbol{G}=\prescript{}{opt}{}\boldsymbol{G}$ that minimize the model's Kullback-Leibler divergence with the data distribution $p$. $D_{KL}(p(\boldsymbol{X}_t)||q(\boldsymbol{X}_t;\boldsymbol{G}))$.
	Because $p^{(k)}\coloneqq p(\boldsymbol{X}_t|C^{(k)}_t)$ is independent from $p^{(l)}$, and $q^{(k)}\coloneqq q(\boldsymbol{X}_t|C^{(k)}_t; \boldsymbol{G}^{(k)})$ is independent from $ \boldsymbol{G}^{(l)}$ for every $l\neq k$,
	we can find the set of parameters that minimize the KL divergence of the mixtures, by minimizing the KL divergence of each component $k$:
	$\min D_{KL}(p^{(k)}||q^{(k)}),\, \forall k,
	$ and simultaneously setting
	$P(C^{(k)}_t)=Q(C^{(k)}_t;G^{(k)}_0), \, \forall k.$
	
	From Eq. \ref{eq:g0_param} and this last condition, Eq. \ref{eq:G0} of the Theorem can be proven.
	Further,
	\begin{gather}
		\prescript{}{opt}{}\boldsymbol{G}^{(k)}  \coloneqq \arg \min_{ \boldsymbol{G}^{(k)} } D_{KL}(p^{(k)}||q^{(k)})\nonumber\\
		=\arg \min_{ \boldsymbol{G}^{(k)} }\int_{\boldsymbol{X}_t}  p^{(k)}  \ln \frac{p^{(k)} }{q^{(k)} }d\boldsymbol{X}_t\nonumber\\
		=\arg \max_{ \boldsymbol{G}^{(k)} } E_{p^{(k)}}\left[u^{(k)}_t\right].
		\label{eq:argmin}
	\end{gather}
	%
	This follows from using the definition of $q^{(k)}$, and it is the expected value of the cosine similarity $u^{(k)}_t$.
	By virtue of the symmetry of the cosine similarity, it follows that
	\begin{multline}
		\prescript{}{opt}{}\boldsymbol{G}^{(k)} =\arg \max_{ \boldsymbol{G}^{(k)} } \cos\left( \boldsymbol{G}^{(k)}, E_{p^{(k)}}\left[\boldsymbol{X}_t\right]\right)
		\\=c\cdot E_{p^{(k)}}\left[\boldsymbol{X}_t\right], c\in\mathbb{R}.
	\end{multline}
	Enforcement of the normalization of the parameter vector results in the unique solution $\prescript{}{opt}{}\boldsymbol{G}^{*(k)}$.
\end{proof}

\subsection{The optimal parameters given the history of observations}
\label{sec:optimal_givenpast}
We will now further specify this solution, for the specific distribution $p(\boldsymbol{X}_t|\boldsymbol{\mathcal{X}}_t, \mathcal{C}_t)$ as described in Eq. \ref{eq:data_generative_model}, and its components $p^{(k)}(\boldsymbol{X}_t|\boldsymbol{\mathcal{X}}_t, \mathcal{C}_t)$.

\begin{thm}
	\label{thm:opt_params}
	Let it be $ \boldsymbol{W}\coloneqq E_{p^{(k)}_B}[\boldsymbol{X}_t]$, and $e^{W^{(k)}_0}\coloneqq P(C^{(k)}_t|B)$.
	The optimal parameters of the model of Definition \ref{def:model}, for data following the assumptions of paragraph \ref{sec:assumptions}, are, for each component $k$,
	\begin{equation}
		\prescript{}{opt}{}\boldsymbol{G}^{ (k)}_t
		=\frac{1}{\beta  \, e^{W^{(k)}_0}}\sum_{i=1}^{T}\boldsymbol{X}_{t_i}
		Q_{t_i}(C^{(k)})f(t-t_i) + \boldsymbol{W} \label{eq:F3_dynamics}
	\end{equation}
	\begin{equation}
		e^{G^{ (0k)}_{t}}=\frac{\sum_{i}^{T} Q_{t_i}(C^{(k)})f(t-t_i)+e^{W^{(0k)}}\beta}{\sum_{i}^{T}f(t-t_i)+\beta}, \label{eq:thr}
	\end{equation}
	\begin{equation}
		\text{where } Q_t(C^{(k)})=\frac{e^{u^{(k)}_t+G^{(0k)}_{t}}}{\sum_{l=1}^{K}e^{u^{(l)}_t+G^{(0l)}_{t}}}. \label{eq:Q_model}
	\end{equation}
\end{thm}

\begin{proof*}
	\begin{equation}
		p(\boldsymbol{X}_t|\boldsymbol{\mathcal{X}}_t, \mathcal{C}_t)=\sum_{k=1}^{K}p^{(k)}(\boldsymbol{X}_t|\boldsymbol{\mathcal{X}}_t, \mathcal{C}_t)\cdot P(C^{(k)}_t|\boldsymbol{\mathcal{X}}_t,\mathcal{C}_t),
	\end{equation}
	and in conjunction with the alternative expression of $p(\boldsymbol{X}_t|\boldsymbol{\mathcal{X}}_t, \mathcal{C}_t)$ in Eq. \ref{eq:data_generative_model}, it follows that
	\begin{multline}
		p^{(k)}(\boldsymbol{X}_t|\boldsymbol{\mathcal{X}}_t, \mathcal{C}_t)\cdot P(C^{(k)}_t|\boldsymbol{\mathcal{X}}_t,\mathcal{C}_t)=\frac{1}{\sum_{i=1}^T f(t-t_i)+\beta}
		\\\cdot\sum_{i=1}^{T}\Big[\mathcal{N}\left(\boldsymbol{X}_{t_i},\boldsymbol{\Sigma}_{t_i,t}\right)P_{t_i}(C^{(k)}) f(t-t_i)+\beta p(\boldsymbol{X}_t|C^{(k)}_t,B)P(C^{(k)}_t|B)\Big]
	\end{multline}
	assuming $P(C^{(k)}_t|\boldsymbol{\mathcal{X}}_t,\mathcal{C}_t)\neq 0$.
	In this Eq., $P(C^{(k)}_t|\boldsymbol{\mathcal{X}}_t,\mathcal{C}_t)$ is still unknown. But it is
	\begin{multline}
		P(C^{(k)}_t|\boldsymbol{\mathcal{X}}_t,\mathcal{C}_t)
		\\= \frac{\sum_{i=1}^{T}P(C^{(k)}_{t_i}|\boldsymbol{X}_{t_i},\boldsymbol{\mathcal{X}}_{t_i},\mathcal{C}_{t_i})f(t-t_i) +\beta P(C^{(k)}_t|B)}{\sum_{i}^{T} f(t-t_i)+\beta},
	\end{multline}
	therefore
	\begin{multline}
		p^{(k)}_t=
		\frac{1}{Z_t^{(k)}} \sum_{i=1}^{T}\Big [\mathcal{N}\left(\boldsymbol{X}_{t_i},\boldsymbol{\Sigma}_{t_i,t}\right)P(C^{(k)}_{t_i}|\boldsymbol{X}_{t_i},\boldsymbol{\mathcal{X}}_{t_i},\mathcal{C}_{t_i})f(t-t_i)
		\\+\beta p(\boldsymbol{X}_t|C^{(k)}_t,B)P(C^{(k)}_t|B)\Big], \label{eq:pk}
	\end{multline}
	where $Z_t^{(k)}$ is the appropriate normalization factor. $P(C^{(k)}_{t_i}|\boldsymbol{X}_{t_i},\boldsymbol{\mathcal{X}}_{t_i},\mathcal{C}_{t_i})$ are missing,
	but can be estimated by the model as
	\begin{equation}
		Q_t(C^{(k)})\coloneqq Q\left(C^{(k)}_{t}|\boldsymbol{X}_t;\prescript{n+1}{}{}\boldsymbol{G}_t^{(k)}\right),
	\end{equation}
	where $\prescript{n+1}{}{}\boldsymbol{G}_t^{(k)}=\left(G_t^{(jk)},\, j=0...n\right)$
	if its optimal parameters are known.
	For this, we derive the analytical form of $Q_t(C^{(k)})$.
	It is
	\begin{multline}
		q\left(\boldsymbol{X}_t;\prescript{n+1}{}{}\boldsymbol{G}_t\right)
		\\=\sum_{k=1}^{K}q(\boldsymbol{X}_t|C^{(k)}_t;\boldsymbol{G}^{ (k)}_t)Q(C^{(k)}_t;G^{ (0k)}_{t})
		=\sum_{k=1}^{K}e^{u^{(k)}_t+G^{(0k)}_{t}}
	\end{multline}
	and thus from Bayes' rule we arrive at Eq. \ref{eq:Q_model}.
	Using this estimate $Q_t(C^{(k)})$, and the fact that $E_{\mathcal{N}\left(\boldsymbol{X}_{t_i},\boldsymbol{\Sigma}_{t_i,t}\right)}\left[\boldsymbol{X}_t\right]=\boldsymbol{X}_{t_i}$ we compute $E_{p^{(k)}_t}[\boldsymbol{X}_t]$ from Eq. \ref{eq:pk}, and thus the optimal (un-normalized) parameter vector $ \prescript{}{opt}{}\boldsymbol{G}^{(k)}$ for every $k$ and $t$ with Eq. \ref{eq:G}, proving Eq. \ref{eq:F3_dynamics}.
	Lastly, through Eq. \ref{eq:G0} we find $e^{G^{ (0k)}_{t}}$ as well.
\end{proof*}
For $t=t_1$, the optimal estimate of $P_1(C^{(k)})$ is $Q\left(C^{(k)}_{1}|\boldsymbol{X}_{1};\prescript{n+1}{}{}\boldsymbol{G}_1^{(k)}\right)$, where $G_{1}^{(jk)}=W_{1}^{(jk)}$. Using this first estimate, iteratively we can calculate
the subsequent optimal parameters $\prescript{}{opt}{}\boldsymbol{G}^{ (k)}_{t}$ and $\prescript{}{opt}{}G^{ (0k)}_{t}$, and probabilities $Q_t(C^{(k)})$.
$ \boldsymbol{W}$ and $W^{(0k)}$ themselves can both be estimated with a standard technique such as Expectation-Maximization (EM), using samples from distribution $p(\boldsymbol{X}_t|B)$.
\section{Modelling the data with a linear mixture}
\label{sec:model_lin}
\subsection{The mixture model and its optimal parameters}
\begin{defn}
	\label{def:model_lin}
	Consider a mixture model distribution $q$:
	$q(\boldsymbol{X}_t)=\sum_{k=1}^{K}q(\boldsymbol{X}_t|C^{(k)}_t)\,Q(C^{(k)}_t),$
	approximating some data distribution $p$ that is also a mixture of K components.
	We choose a mixture of piecewise linear likelihood functions, specifically:
	\begin{equation}q^{(k)}\coloneqq q(\boldsymbol{X}_t|C^{(k)}_t;\boldsymbol{G}^{(k)})=\frac{1}{Z}\max\left( u^{(k)}_t ,\,0\right)
	\end{equation}
	where $Z$ is a normalization factor, and
	$u^{(k)}_t=\frac{ \boldsymbol{G}^{(k)}\cdot\boldsymbol{X}_t}{|| \boldsymbol{G}^{(k)}||\cdot||\boldsymbol{X}_t||}$, i.e. the cosine similarity of the two vectors $\boldsymbol{G}^{(k)}$ and $\boldsymbol{X}_t$.
	The joint probability is \begin{equation}q(\boldsymbol{X}_t,C^{(k)}_t)\equiv q(\boldsymbol{X}_t|C^{(k)}_t)Q(C^{(k)}_t)=\max\left(\frac{u^{(k)}_t}{Z}  ,\,0\right)Q(C^{(k)}_t)\end{equation}
	and the parametrized model approximates it as
	\begin{equation}
		q(\boldsymbol{X}_t,C^{(k)}_t;\prescript{n+1}{}{}\boldsymbol{G}_t^{(k)})\approx \max\left(\frac{u^{(k)}_t+G^{(0k)}}{Z},\,0\right).\end{equation}
	
	In addition, the parameter vectors are subject to the normalization constraints:
	$||\boldsymbol{G}^{(k)}||=1,\, \forall k$,
	and
	$
	\sum_{k=1}^{K}e^{G^{(0k)}}=1. \label{eq:norm_0}$
\end{defn}

\begin{thm}
	\label{thm:optcos}
	The optimal parameters of such a mixture model are
	\begin{equation}
		\prescript{}{opt}{}\boldsymbol{G}^{*(k)}=\frac{ \prescript{}{opt}{}\boldsymbol{G}^{(k)}}{|| \prescript{}{opt}{}\boldsymbol{G}^{(k)}||}=\frac{E_{p^{(k)}}\left[\boldsymbol{X}_t\right]}{||E_{p^{(k)}}\left[\boldsymbol{X}_t\right]||}, \label{eq:Gstar}
	\end{equation}
	\begin{equation}
		\text{where } \prescript{}{opt}{}\boldsymbol{G}^{(k)}=c\cdot E_{p^{(k)}}\left[\boldsymbol{X}_t\right], c\in\mathbb{R}, \label{eq:G}
	\end{equation}
	\begin{equation}
		p^{(k)}\coloneqq p(\boldsymbol{X}_t|C^{(k)}_t)
	\end{equation}
	for every $k$
	and the optimal bias parameter of each component $k$ is related to $P(C^{(k)}_t)$ as
	\begin{multline}
		\iff P(C^{(k)}_t)\\
		=\sqrt{1-\left(\prescript{}{opt}{}G^{(0k)}\right)^2}+\prescript{}{opt}{}G^{(0k)}\arccos(-\prescript{}{opt}{}G^{(0k)}). \label{eq:P_G0}
	\end{multline}
	and can be approximated as
	\begin{equation}
		\prescript{}{opt}{}G^{(0k)}\approx P(C^{(k)}_t)-1
	\end{equation}
	for every $k$.
\end{thm}

\begin{proof}
	The model is optimal from a minimum cross-entropy perspective if its parameters equal those parameters $\boldsymbol{G}=\prescript{}{opt}{}\boldsymbol{G}$ that minimize the Kullback-Leibler divergence between the mixture distributions p and q:
	\begin{equation}
		\min D_{KL}(p(\boldsymbol{X}_t)||q(\boldsymbol{X}_t|\boldsymbol{G}^{*})).
	\end{equation}
	
	Because $p^{(k)}\coloneqq p(\boldsymbol{X}_t|C^{(k)}_t)$ is independent from $p^{(l)}$, and $q^{(k)}\coloneqq q(\boldsymbol{X}_t|C^{(k)}_t, \boldsymbol{G}^{(k)})$ is independent from $ \boldsymbol{G}^{(k)}$ for every $l\neq k$,
	we can find the set of parameters that minimize the KL divergence of the mixtures, by minimizing the KL divergence of each component $k$:
	\begin{equation}\min D_{KL}(p^{(k)}||q^{(k)}),\, \forall k.
	\end{equation}
	
	As in Theorem \ref{thm:optimal}, it follows that
	\begin{equation}
		\prescript{n}{opt}{}\boldsymbol{G}^{(k)}=\arg \max E_{p^{(k)}}\left[\ln q^{(k)}\right].
	\end{equation}
	The parameter vectors that bring this expected value to a (local) maximum are found by demanding that its derivative equal zero:
	\begin{equation}
		\frac{\partial  E_{p^{(k)}}\left[\ln q^{(k)}\right]}{\partial  \boldsymbol{G}^{(k)}}=0. \label{eq:partial_devG}
	\end{equation}
	Each vector $ \boldsymbol{G}^{(k)}$ defines a unique angle $\phi^{(k)}_0$ with the fixed vector $E_{p^{(k)}}\left[\boldsymbol{X}_t\right]$: \begin{equation}\phi^{(k)}_0 =\angle\left( \boldsymbol{G}^{(k)},E_{p^{(k)}}\left[\boldsymbol{X}_t\right]\right). \label{eq:phi0}
	\end{equation}
	Therefore, equivalently to Eq. \ref{eq:partial_devG}, but using this uniquely corresponding angle instead of $\boldsymbol{G}^{(k)}$, we can demand that
	\begin{equation}
		\frac{\partial  E_{p^{(k)}}\left[\ln q^{(k)}\right]}{\partial \phi^{(k)}_0}=0.\label{eq:partial}\end{equation}
	But $q^{(k)}$ is defined based on the cosine similarity $u^{(k)}$ that is the cosine of an angle different from $\phi^{(k)}_0$. In particular, it is the cosine of
	\begin{equation}\phi^{(k)}\coloneqq\angle \left( \boldsymbol{G}^{(k)},\boldsymbol{X}_t\right).
	\end{equation}
	To find the solution of \ref{eq:partial}, we will express $q^{(k)}$ with respect to $\phi^{(k)}_0$. Let $\theta^{(k)}$ be such that
	\begin{equation}\phi^{(k)}=\phi^{(k)}_0+\theta^{(k)}. \label{eq:theta}
	\end{equation}
	From the definition of the model, it is
	\begin{multline}q^{(k)}=\frac{1}{Z}\max\left( \cos\phi^{(k)} ,\,0\right)\\
		=\frac{1}{Z}\max\left(\cos\left(\phi^{(k)}_0+\theta^{(k)}\right) ,\,0\right).
	\end{multline}
	The normalization factor $Z$ can be found as follows.
	\begin{gather}\int_{0}^{2\pi} q^{(k)}d\phi^{(k)}=1 \iff\\
		\int_{0}^{2\pi} \frac{1}{Z}\max\left( \cos\phi^{(k)} ,\,0\right)d\phi^{(k)}=1\iff\\
		Z=\int_{0}^{\pi/2} \cos\phi^{(k)} d\phi^{(k)}+\int_{3\pi/2}^{2\pi} \cos\phi^{(k)} d\phi^{(k)}=2,
	\end{gather}
	i.e.
	\begin{multline}q^{(k)}=\frac{1}{2}\max\left( \cos\phi^{(k)} ,\,0\right)\\ =\frac{1}{2}\max\left(\cos\left(\phi^{(k)}_0+\theta^{(k)}\right) ,\,0\right).
	\end{multline}
	Using this form we will find the optimal angle $\prescript{}{opt}{}\phi^{(k)}_0$ from Eq. \ref{eq:partial}. For $\pi/2<\phi^{(k)}<3\pi/2$, Eq. \ref{eq:partial} is true for any $\phi^{(k)}_0$. For $0\leq\phi^{(k)}\leq\pi/2$ or $3\pi/2\leq\phi^{(k)}\leq2\pi$, it is
	\begin{gather}
		E_{p^{(k)}}\left[\frac{1}{q^{(k)}}\frac{\partial q^{(k)}}{\partial \phi^{(k)}_0}\right]=0\iff \\
		E_{p^{(k)}}\left[\frac{1}{q^{(k)}}\frac{\partial \cos\left(\phi^{(k)}\right)}{\partial \phi^{(k)}_0}\right]=0\iff \\
		E_{p^{(k)}}\left[\frac{1}{q^{(k)}}(-\sin\phi^{(k)})\right]=0\iff\\
		E_{p^{(k)}}\left[\sin\phi^{(k)}\right]=0\iff\\
		E_{p^{(k)}}\left[\phi^{(k)}\right]=0. \label{eq:E_phi_k}
	\end{gather}
	We will now show that $\phi^{(k)}_0=0$ is a solution that satisfies this condition.
	If $\phi^{(k)}_0=0$, i.e. $\angle\left( \boldsymbol{G}^{(k)},E_{p^{(k)}}\left[\boldsymbol{X}_t\right]\right)=0$, then \begin{equation}\phi^{(k)}=\angle \left( \boldsymbol{G}^{(k)},\boldsymbol{X}_t\right)=\angle \left( E_{p^{(k)}}\left[\boldsymbol{X}_t\right],\boldsymbol{X}_t\right).\end{equation}
	Also, from Eq. \ref{eq:theta}, it follows that $\theta^{(k)}=\phi^{(k)}$, because $\phi^{(k)}_0=0$.
	Therefore, \begin{gather}\theta^{(k)}=\angle \left( E_{p^{(k)}}\left[\boldsymbol{X}_t\right],\boldsymbol{X}_t\right)\nonumber\\ \implies E_{p^{(k)}}\left[\theta^{(k)}\right]=\angle \left( E_{p^{(k)}}\left[\boldsymbol{X}_t\right],E_{p^{(k)}}\left[\boldsymbol{X}_t\right]\right)=0. \label{eq:E_theta}\end{gather}
	
	From Eq. \ref{eq:theta}, and using Eq. \ref{eq:E_theta} and our assumption that $\phi^{(k)}_0=0$, it follows that $E_{p^{(k)}}\left[\phi^{(k)}\right] =E_{p^{(k)}}\left[\phi^{(k)}_0\right]+E_{p^{(k)}}\left[\theta^{(k)}\right]=0$, and this is Eq. \ref{eq:E_phi_k}. Therefore $\phi^{(k)}_0=0$ is indeed a solution that satisfies the condition of Eq. \ref{eq:E_phi_k}.
	
	Because the optimal parameter vector $\prescript{n}{opt}{}\boldsymbol{G}^{(k)}$ satisfies this condition, then, from the definition of $\phi^{(k)}_0$ in Eq. \ref{eq:phi0}, it follows that
	\begin{equation}
		\prescript{n}{opt}{}\boldsymbol{G}^{(k)}=cE_{p^{(k)}}\left[\boldsymbol{X}_t\right] \label{eq:Gq_cosine}
	\end{equation}
	for any $c>0$.
	Enforcement of the requirement for normalization of the vector leads to the unique solution $ \prescript{}{opt}{}\boldsymbol{G}^{*(k)}$.
	
	Having determined $\prescript{n}{opt}{}\boldsymbol{G}^{(k)}$, it remains to determine the optimal value of the parameters $G^{(0k)}$ for every $k$, to complete the model of Def. \ref{def:model_lin}.
	We use the fact that
	\begin{gather}
		Q(C^{(k)}_t)=\int_{\phi^{(k)}}q(\phi^{(k)},C^{(k)}_t)d\phi^{(k)}
	\end{gather}
	To solve for the parameter $G^{(0k)}$, we introduce it by using, as per the model, the estimate \begin{equation}q(\phi^{(k)},C^{(k)}_t)\approx\widehat{q}(\phi^{(k)},C^{(k)}_t;\prescript{n+1}{}{}\boldsymbol{G}^{(k)})=\frac{\left(u^{(k)}_t+G^{(0k)}\right)^+}{2}.\end{equation}
	$u^{(k)}_t=\cos\phi^{(k)}\leq 0$ for $\frac{\pi}{2}\leq\phi^{(k)}\leq\frac{3\pi}{2}$.
	If $G^{(k)}_0\leq 0$, which is true, as we show at the end of this Theorem's proof, then the curve of $y=u^{(k)}_t+G^{(0k)}$ is the curve of $y=\cos\phi^{(k)}$, shifted lower by $|G^{(0k)}|$. This implies firstly that $u^{(k)}_t+G^{(0k)}=0$ when $\phi^{(k)}=\arccos(-G^{(0k)})$, and secondly that
	$0\leq \arccos(-G^{(0k)})\leq\frac{\pi}{2}$. This curve, i.e. $u^{(k)}_t+G^{(0k)}$, is now non-negative only for $0\leq\phi^{(k)}\leq\arccos(-G^{(0k)})$
	and for $\frac{3\pi}{2}+\arccos(-G^{(0k)}) \leq \phi^{(k)}\leq2\pi$. So,
	\begin{multline} Q(C^{(k)}_t)=\int_{0}^{\arccos(-G^{(0k)})}\frac{\cos\phi^{(k)}+G^{(0k)}}{2}d\phi^{(k)}\\
		+\int_{\frac{3\pi}{2}+\arccos(-G^{(0k)})}^{2\pi}\frac{\cos\phi^{(k)}+G^{(0k)}}{2}d\phi^{(k)}\\
		=2\cdot\frac{1}{2}\left.\left(\sin\phi^{(k)} +  G^{(0k)} \phi^{(k)}\right)\right|_0^{\arccos(-G^{(0k)})}\\
		=\sin\arccos(-G^{(9k)})+G^{(k)}_0\arccos(-G^{(0k)})\end{multline}
	\begin{multline}
		\iff Q(C^{(k)}_t)\\
		=\sqrt{1-\left(G^{(0k)}\right)^2}+G^{(0k)}\arccos(-G^{(0k)}). \label{eq:Q_G0_proof}
	\end{multline}
	by demanding that
	$Q(C^{(k)}_t)=P(C^{(k)}_t) \quad \forall k,
	$, this proves Eq. \ref{eq:P_G0} of the Theorem.
	This equation cannot be solved for $G^{(0k)}$ analytically.
	
	Our experimental results (see main paper, "Application to video recognition" section) show that the model can function well in practice without updating the bias or threshold parameters $G^{(0k)}$. Still, here, we do derive an approximation of $G^{(0k)}$ as a function of $P(C^{(k)}_t)$. We find from Eq. \ref{eq:P_G0} that $G^{(0k)}=-1 \implies P(C^{(k)}_t)=0$ and $G^{(0k)}=0 \implies Q(C^{(k)}_t)=1$. In addition, numerically, we find that $P(C^{(k)}_t)$ is an increasing function of $G^{(0k)}$.
	
	We will now quantify how non-linear this curve of Eq. \ref{eq:P_G0} is, in order to assess how reasonable it is to use a line as an estimate of the curve. The line that passes through the two points that we specified on the $(G^{(0k)}, P(C^{(k)}_t))$ plane, namely the points (-1, 0) and (0, 1), is described by $\widehat{P(C^{(k)}_t)}=\widehat{G^{(0k)}}+1$, i.e. it has a derivative equal to 1. We find that the derivative of the curve of Eq. \ref{eq:Q_G0_proof} is $\frac{\partial Q(C^{(k)}_t)}{\partial G^{(0k)}}=\arccos(-P(C^{(k)}_t))$. In the range $-1\leq G^{(0k)}\leq 0$, the derivative increases from a value of 0 at $G^{(0k)}=-1$, and crosses the value of 1 at $G^{(0k)}=-\cos(1)$. I.e., the curve initially diverges from the straight line, and at the point $G^{(0k)}=-\cos(1)$ it begins to converge. Therefore, at this point the line has its maximum distance from the curve. Specifically, the difference between the value of $P(C^{(k)}_t)$ at this point and its approximation $\widehat{P(C^{(k)}_t)}$ by the line is equal to $cos(1)-\sqrt{\sin^2(1)-\cos(1)}$ which, compared to the real value $P(C^{(k)}_t)=\sqrt{\sin^2(1)-\cos(1)}$ is a 31.9\% overestimation by the line, and that is the maximum divergence of the line, showing that it can be reasonable to use the line
	\begin{equation}
		\prescript{}{opt}{}\widehat{G^{(0k)}}=P(C^{(k)}_t)-1 \quad \forall k \label{eq:G0_cosine}
	\end{equation}
	as an approximation of the optimal value $\prescript{}{opt}{}G^{(0k)}$ of parameter $G^{(0k)}$.
\end{proof}
Our derivation relied on the assumption that $G^{(0k)}\leq 0$. We will show now that this is true.
\begin{equation}\widehat{q}(\phi^{(k)},C^{(k)}_t;\prescript{n+1}{}{}\boldsymbol{G}^{(k)})\approx q(\phi^{(k)},C^{(k)}_t)
\end{equation}
\begin{multline}\implies\left(\cos\phi^{(k)}+G^{(0k)}\right)^+\\
	\approx\left(\cos\phi^{(k)}\right)^+Q(C^{(k)}_t)\leq\left(\cos\phi^{(k)}\right)^+
\end{multline}
\begin{equation}\implies G^{(0k)}\leq 0.
\end{equation}
\subsection{The optimal parameters given the history of observations}
\begin{thm}
	Let it be $ \boldsymbol{W}\coloneqq E_{p^{(k)}_B}[\boldsymbol{X}_t]$, and $W^{(0k)}\coloneqq P(C^{(k)}_t|B)-1$.
	The optimal parameters of the model of Definition \ref{def:model_lin}, for data following the assumptions of paragraph \ref{sec:assumptions}, are, for each component $k$,
	\begin{multline}
		\prescript{n}{opt}{}\boldsymbol{G}^{(k)}\\
		= \frac{1}{\beta\left( W^{(0k)}+1\right)}\sum_{i=1}^{T}\boldsymbol{X}_{t_i}Q_{t_i}(C^{(k)})f(t-t_i) + \boldsymbol{W}^{(k)}, \label{eq:F3_dynamics_cos}
	\end{multline} and
	\begin{multline}
		\prescript{}{opt}{}G^{(0k)}\approx\prescript{}{opt}{}\widehat{G^{(0k)}}\\
		=\frac{\beta}{\sum_{i}^{T} f(t-t_i)+\beta}\\
		\cdot\left(\sum_{i=1}^{T}(Q_{t_i}(C^{(k)})-1)\frac{f(t-t_i)}{\beta} + W^{(0k)}\right), \label{eq:thr_cos}
	\end{multline}
	\begin{equation}
		\text{where }     Q_t(C^{(k)})=\frac{\max\left(u^{(k)}_t+G^{(0k)},\,0\right)}{\sum_{l=1}^{K}\max\left(u^{(l)}_t+G^{(0l)},\,0\right)}. \label{eq:Q_model_}
	\end{equation}
\end{thm}
\begin{proof}
	Theorem \ref{thm:optcos} indicates that
	\begin{equation}
		\prescript{}{opt}{}\boldsymbol{G}^{(k)}=c\cdot E_{p^{(k)}}\left[\boldsymbol{X}_t\right],\, c>0
	\end{equation}
	for every $k$, where $\boldsymbol{X}_t\sim p(\boldsymbol{X}_t)$ from Eq. \ref{eq:data_generative_model}. As in the proof of Theorem \ref{thm:opt_params}, we can use Eq. \ref{eq:pk}, the fact that $E_{\mathcal{N}\left(\boldsymbol{X}_{t_i},\boldsymbol{\Sigma}_{t_i,t}\right)}\left[\boldsymbol{X}_t\right]=\boldsymbol{X}_{t_i}$, and, as an estimate of $P(C^{(k)}_{t_i}|\boldsymbol{X}_{t_i},\boldsymbol{\mathcal{X}}_{t_i},\mathcal{C}_{t_i})$, the model of Definition \ref{def:model_lin}, to find the parameters $\prescript{}{opt}{}\boldsymbol{G}^{(k)}$ for this data distribution $p(\boldsymbol{X}_t)$, given past observations.
	
	For each time instance $t_i$, $P(C^{(k)}_{t_i}|\boldsymbol{X}_{t_i},\boldsymbol{\mathcal{X}}_{t_i},\mathcal{C}_{t_i})$ can be estimated by using
	\begin{equation}
		Q_t(C^{(k)})\coloneqq Q\left(C^{(k)}_{t}|\boldsymbol{X}_t;\prescript{n+1}{}{}\boldsymbol{G}_t\right),
	\end{equation}
	if the optimal parameters are known.
	For this, we will now derive the parametric form of $Q_t(C^{(k)})$.
	From Bayes' rule, it is
	\begin{multline}
		Q_t(C^{(k)})=\frac{q(\boldsymbol{X}_t,C^{(k)}_t;\prescript{n+1}{}{}\boldsymbol{G}^{(k)}_t)}{q(\boldsymbol{X}_t|;\prescript{n+1}{}{}\boldsymbol{G}_t)}\\
		=\frac{q(\boldsymbol{X}_t,C^{(k)}_t;\prescript{n+1}{}{}\boldsymbol{G}^{(k)}_t)}{\sum_{l=1}^{K}q(\boldsymbol{X}_t,C^{(l)}_t;\prescript{n+1}{}{}\boldsymbol{G}^{(l)}_t)},
	\end{multline}
	where $q(\boldsymbol{X}_t,C^{(k)}_t;\prescript{n+1}{}{}\boldsymbol{G}^{(k)}_t)$ is approximated by the model of Definition \ref{def:model_lin} as
	\begin{equation}
		q(\boldsymbol{X}_t,C^{(k)}_t;\prescript{n+1}{}{}\boldsymbol{G}^{(k)})\approx \frac{1}{2} \max\left(u^{(k)}_t+G^{(0k)},\,0\right)
	\end{equation}
	i.e.
	\begin{equation}
		Q_t(C^{(k)})=\frac{\max\left(u^{(k)}_t+G^{(0k)},\,0\right)}{\sum_{l=1}^{K}\max\left(u^{(l)}_t+G^{(0l)},\,0\right)}. \label{eq:Q_model_cos}
	\end{equation}
	
	Let \begin{equation}
		W^{(k)}_0\coloneqq P(C^{(k)}_t|B)-1,
	\end{equation}
	and let
	\begin{equation}
		\boldsymbol{W}^{(k)}\coloneqq E_{p^{(k)}_B}[\boldsymbol{X}_t].
	\end{equation}
	Then, using the definition of Eq. \ref{eq:pk}, we find
	\begin{multline}
		E_{p^{(k)}_t}[\boldsymbol{X}_t]=
		\frac{1}{Z_t^{(k)}}\left( \sum_{i=1}^{T}\boldsymbol{X}_{t_i}P(C^{(k)}_{t_i}|\boldsymbol{X}_{t_i},\boldsymbol{\mathcal{X}}_{t_i},\mathcal{C}_{t_i})f(t-t_i)\right.\\
		+\Bigg.\beta  \boldsymbol{W}^{(k)}\left( W^{(k)}_0+1\right)\Bigg). \label{eq:expectationY}
	\end{multline}
	Therefore, because
	\begin{equation}
		\prescript{n}{opt}{}\boldsymbol{G}^{(k)}=c\cdot E_{p^{(k)}}\left[\boldsymbol{X}_t\right]
	\end{equation}
	for any chosen positive $c$, we can choose $c$ such that
	\begin{multline}
		\prescript{n}{opt}{}\boldsymbol{G}^{(k)}\\
		=\frac{1}{\beta\left(W^{(0k)}+1\right)} \sum_{i=1}^{T}\boldsymbol{X}_{t_i}Q_{t_i}(C^{(k)})f(t-t_i) + \boldsymbol{W}^{(k)}\label{eq:gfw_dynamics_cos},
	\end{multline}
	where we used the model $Q_{t_i}(C^{(k)})$ as an estimate of $P(C^{(k)}_{t_i}|\boldsymbol{X}_{t_i},\boldsymbol{\mathcal{X}}_{t_i},\mathcal{C}_{t_i})$.
	This proves Eq. \ref{eq:F3_dynamics_cos} of the Theorem.
	
	We will now also find the optimal bias parameters $\prescript{}{opt}{}G^{(k)}_0$.
	In Theorem \ref{thm:optcos} we showed that this can be approximated relatively well as
	\begin{gather}
		\prescript{}{opt}{}G^{(0k)}\approx\prescript{}{opt}{}\widehat{G^{(0k)}}=P(C^{(k)}_t)-1.
	\end{gather}
	Therefore, using
	\begin{gather}
		P(C^{(k)}_t) = \sum_{i=1}^{T} P(C^{(k)}_t|A_{t_i})P(A_{t_i})+P(C^{(k)}_t|B)P(B)\iff\\
		P(C^{(k)}_t)= \frac{\sum_{i=1}^{T}P(C^{(k)}_{t_i}|\boldsymbol{X}_{t_i},\boldsymbol{\mathcal{X}}_{t_i},\mathcal{C}_{t_i})f(t-t_i) +\beta\left( W^{(0k)}+1\right)}{\sum_{i}^{T} f(t-t_i)+\beta},
	\end{gather}
	we find that
	\begin{multline}
		\prescript{}{opt}{}G^{(0k)}\approx\frac{\sum_{i=1}^{T}P(C^{(k)}_{t_i}|\boldsymbol{X}_{t_i},\boldsymbol{\mathcal{X}}_{t_i},\mathcal{C}_{t_i})f(t-t_i) +\beta\left( W^{(0k)}+1\right)}{\sum_{i}^{T} f(t-t_i)+\beta}-1\\
		=\frac{\beta}{\sum_{i}^{T} f(t-t_i)+\beta}\\
		\cdot\left[\sum_{i=1}^{T}(P(C^{(k)}_{t_i}|\boldsymbol{X}_{t_i},\boldsymbol{\mathcal{X}}_{t_i},\mathcal{C}_{t_i})-1)\frac{f(t-t_i)}{\beta} + W^{(0k)}\right]. \label{eq:gf0_dynamics_cos}
	\end{multline}
	By making use of the model $Q_t(C^{(k)})$ as an estimate of $P_t(C^{(k)})$ in Eq. \ref{eq:gf0_dynamics_cos}, we prove Eq. \ref{eq:thr_cos} of the Theorem.
\end{proof}

\section{Clustering interpretation of the model and of its optimization: Elastic Clustering}
\label{sec:clustering_sup}
\begin{Algorithm}[h]
	\begin{algorithmic}
		\caption{Elastic Clustering}
		\FORALL {centroids $k$}
		\STATE $\prescript{n+1}{}{}\boldsymbol{F}^{(k)} \gets \boldsymbol{0}$\;
		\STATE $\prescript{n+1}{}{}\boldsymbol{G}^{(k)} \gets \prescript{n+1}{}{}\boldsymbol{W}^{(k)}$\COMMENT{Initialize centroids. eg with EM}
		\ENDFOR
		\FOR {$t = 1$ \textbf{to} $duration$}{
			\STATE $\prescript{n+1}{}{}\boldsymbol{F}^{(k)} \gets \prescript{n+1}{}{}\boldsymbol{F}^{(k)} - \lambda\cdot\prescript{n+1}{}{}\boldsymbol{F}^{(k)}$ \COMMENT{Elastically relax F}
			\STATE $\boldsymbol{G}^{(k)} \gets \boldsymbol{W}^{(k)}+\boldsymbol{F}^{(k)}$ \COMMENT{Update generative model with relaxed parameters}\;
			\STATE $G^{(0k)} \gets (W^{(0k)}+F^{(0k)})/Z^{(0)}_{t}$ \COMMENT{Update prior probabilities of the model. $Z^{(0)}_{t}$ used for possible normalization}
			\STATE \textbf{input} $\boldsymbol{X}$\;
			\IF{$\boldsymbol{X} \neq \boldsymbol{0}$} {
				\FOR {all centroids $k$}{
					\STATE $u^{(k)}\gets u(\boldsymbol{X},\boldsymbol{G}_k)$\COMMENT{Compute input-centroid proximity, i.e. summed weighted input of neuron $k$}
					\STATE $q_{X}^{(k)}\gets q(u^{(k)}, G^{(0k)})$\COMMENT{Compute joint probability of input and cause $k$, i.e. un-normalized activation of neuron, e.g. exponential activation}
				}
				\ENDFOR
				\STATE $q\_all \gets$ set of all $q_{X}^{(k)}$\;
				\FOR{all centroids $k$}{
					\STATE $Q_k\gets Q\left(q_{X}^{(k)},\, q\_all\right)$\COMMENT{Inference of posterior probability. eg $Q=$softmax, i.e. normalized activation of exponential neuron $k$, or $Q=$max}
					\STATE $\boldsymbol{F}^{(k)} \gets \boldsymbol{F}^{(k)}+\gamma\boldsymbol{X}Q^{(k)}$ \COMMENT{Input-centroid attraction. Hebbian aspect of synaptic plasticity}\;
					\STATE $F^{(0k)} \gets F^{(0k)}+\gamma Q^{(k)}$
				}\ENDFOR
			}\ENDIF
		}
		\ENDFOR
		\label{algo:elastic}
	\end{algorithmic}
\end{Algorithm}
Both the exponential and the linear model, i.e. Eq. \ref{eq:F3_dynamics} and Eq. \ref{eq:F3_dynamics_cos} show that $ \prescript{}{opt}{}\boldsymbol{G}^{ (k)}$ is the weighted average of two points in the n-dimensional space, one being the fixed term $\boldsymbol{W}^{(k)}$ and the other being a dynamic term, specifically a discrete convolution of $\boldsymbol{X}_{t_i} Q_{t_i}(C^{(k)})$ with $f(t)$, be it $ \boldsymbol{F}^{(k)}_t$:
\begin{equation}
	\boxed{ \boldsymbol{G}^{(k)}_t=
		\boldsymbol{F}^{(k)}_t
		+
		\boldsymbol{W}^{(k)}}.
\end{equation}
From the same Eq. \ref{eq:F3_dynamics} or Eq. \ref{eq:F3_dynamics_cos} it can be seen that, at every posterior inference result $Q^{(k)}_t$, the dynamic term is incremented as follows:
\begin{equation}
	\boxed{\boldsymbol{F}^{(k)}_t\leftarrow
		\boldsymbol{F}^{(k)}_t+
		\gamma\boldsymbol{X}_t
		Q^{(k)}_t}, \label{eq:F_sup}
\end{equation}
where $\gamma$ is a positive constant, specifically $\gamma\coloneqq \frac{f(0)}{\beta e^{W^{(k)}_0}}$ in the case of the exponential model (Eq.
\ref{eq:F3_dynamics}), and $\gamma\coloneqq \frac{f(0)}{\beta\left( W^{(k)}_0+1\right)}$ in the case of the linear model (Eq.
\ref{eq:F3_dynamics_cos}). In addition, this term $\boldsymbol{F}^{(k)}_t$ subsequently decays continuously according to the dynamics of $f(t)$, such
that $\boldsymbol{G}^{(k)}_t$ relaxes towards the fixed resting point $\boldsymbol{W}^{(k)}$. If $f(t)$ is exponential with a rate $\lambda$, then
\begin{equation} \boxed{\frac{d\boldsymbol{F}^{(k)}_t}{dt}=-\lambda\boldsymbol{F}^{(k)}_t}.\end{equation}
Eq. \ref{eq:F_sup} shows that each input $\boldsymbol{X}_{t_i}$ attracts $\boldsymbol{G}^{ (k)}$ in proportion to the inferred probability $Q_{t_i}(C^{(k)})$ that $\boldsymbol{X}_{t_i}$'s label was $C^{(k)}$. In addition, the attraction by each past input is in proportion to the value of a temporally evolving kernel $f(t)$. $f$ ultimately decays to zero, as assumed in Section \ref{sec:assumptions}, Assumption \ref{as:f}, for example exponentially.
That is, the movement of $\boldsymbol{G}^{ (k)}$ towards the inputs is elastic, as the dynamics is relaxation of $\boldsymbol{G}^{ (k)}$ towards a resting point $\boldsymbol{W}^{(k)}$. Therefore, the model can be described as a sequential clustering algorithm based on elastic centroids (Fig. \ref{fig:clustering}\textbf{a} and Algorithm 1).
Every input is associated by the algorithm with each cluster to a degree $Q_t^{(k)}$ that depends on the proximity of the input vector to the
centroid. If the proximity function for the clustering is specifically chosen to be the cosine similarity, then the Elastic Clustering can compute
the optimal parameters of the two generative models we described in \ref{def:model} and \ref{def:model_lin}. In particular, if $Q_t^{(k)}$ is
implemented as in Eq. \ref{eq:Q_model}, and in turn its bias parameter $G_{0t}^{(k)}$ also adapts to the input elastically with time as per Eq.
\ref{eq:thr}, then the clustering algorithm realizes the exponential model and the associated Bayesian optimization that were described in Section
\ref{sec:model}. On the other hand, if $Q_t^{(k)}$ is instead implemented as in Eq. \ref{eq:Q_model_}, and its bias parameter adapts as per Eq.
\ref{eq:thr_cos}, then the clustering algorithm realizes and approximately optimizes the linear model of Section \ref{sec:model_lin}.

\section{Neurosynaptic equivalence of the optimal model: ST-STDP}
\label{sec:neuro_equiv}
\subsection{Stochastic or spiking input}
\label{sec:stoch} The model remains functional if we assume that the samples $\boldsymbol{X}_{t}$ are not visible to the observer, but what is
visible instead is a stochastic measurement $\boldsymbol{X}^{stoch}_{t}$, drawn randomly from a probability distribution whose mean is the underlying
data point: $E[\boldsymbol{X}^{stoch}_{t}]=\boldsymbol{X}_{t}.$ An example of such inputs are Poisson-distributed rate-coded spiking inputs, a common
method of input coding in models of SNNs \cite{nessler2013PLoS,diehl2015FCN}, as well as a good model of biological rate-coding neurons
\cite{hubel1959JP,gerstner2014Cambridge,brette2015FSN}. To use the Elastic Clustering model, we need an estimate of the -- now hidden -- measured
data point $\boldsymbol{X}_{t}$. If the measurement distribution were stationary, then the underlying data point could be estimated as the uniformly
weighted average of the noisy measurements over the past. However, as each object is transformed, and the object that is visible switches, the
measurement distribution cannot be assumed stationary. Instead, the estimate of the data point is
\begin{multline}
	\boldsymbol{X}_{t}=
	E[\boldsymbol{X}^{stoch}_{t}]
	=\int\boldsymbol{X}^{stoch} p\left(\boldsymbol{X}^{stoch}\right) d\boldsymbol{X}^{stoch}
	\\=\int_{\tau=0}^{t}\boldsymbol{X}^{stoch}_{t-\tau} p\left(\boldsymbol{X}_{t}=\boldsymbol{X}_{t-\tau}\right) d\tau. \label{eq:stoch_sup}
\end{multline}
$p\left(\boldsymbol{X}_{t}=\boldsymbol{X}_{t-\tau}\right)$ can be derived from the dynamics of the objects in the environment.
\begin{multline}
	p\left(\boldsymbol{X}_{t}=\boldsymbol{X}_{t-\tau}\right)
	\\=p\left(\boldsymbol{X}_{t}=\boldsymbol{X}_{t-\tau}|O_t=O_{t-\tau}\right)P(O_t=O_{t-\tau})
	\\+p\left(\boldsymbol{X}_{t}=\boldsymbol{X}_{t-\tau}|O_t\neq O_{t-\tau}\right)P(O_t\neq O_{t-\tau})
	\\=p\left(\boldsymbol{X}_{t}=\boldsymbol{X}_{t-\tau}|O_t=O_{t-\tau}\right)e^{-\alpha\tau},
	\label{eq:stoch_cases}
\end{multline}
because the observed objects are replaced according to a Poisson process of rate $\alpha$ (Section \ref{sec:assumptions}, Assumption \ref{as:poisson1}), and because $p\left(\boldsymbol{X}_{t}=\boldsymbol{X}_{t-\tau}|O_t\neq O_{t-\tau}\right)=0$.
But we have concluded in the proof of Theorem \ref{thm:pdf} that, due to the assumptions of continuity, each object is generated by an underlying Wiener process. So the resulting probability density function is a multivariate Gaussian:
\begin{multline}
	p(\boldsymbol{X}_{t}|\boldsymbol{X}_{t_i}, O_t=O_{t_i})\equiv \mathcal{N}\left(\boldsymbol{X}_{t_i},\boldsymbol{\Sigma}_{t_i,t}\right)\\
	=\frac{\exp\left(-\frac 1 2 ({\boldsymbol{X}_t}-{\boldsymbol{X}_{t_i}})^\mathrm{T}{\boldsymbol\Sigma}_{t_i,t}^{-1}({\boldsymbol{X}_t}-{\boldsymbol{X}_{t_i}})\right)}{\sqrt{(2\pi)^n|\boldsymbol\Sigma_{t_i,t}|}},\label{eq:multinormal}\end{multline}
where $\boldsymbol\Sigma_{t_i,t}$ is the covariance matrix of the distribution, whose determinant $|\boldsymbol\Sigma_{t_i,t}|$ increases linearly with time $\tau=t-t_i$ with some positive rate $\sigma^2$, i.e.
$|\boldsymbol\Sigma_{t_i,t}|=\sigma^2\,\tau.$
Thus, because $\boldsymbol{X}_{t_i}$ is the mean of this distribution, it follows that
\begin{equation}
	p(\boldsymbol{X}_{t}=\boldsymbol{X}_{t_i}|\boldsymbol{X}_{t_i}, O_t=O_{t_i})
	=\frac{1}{\sqrt{(2\pi)^n\sigma^2\,\tau}}.
\end{equation}
So, using this result in Eq. \ref{eq:stoch_cases}, we conclude from Eq. \ref{eq:stoch_sup}, that the underlying data point at time $t$ is estimated
as proportional to the weighted integral of past measurements, where the weightings decay with $\tau^{-1/2}$ and with $e^{-\alpha \tau}$:
\begin{equation}
	\boxed{\boldsymbol{X}_{t}= \frac{1}{(2\pi)^{n/2}\cdot\sigma}
		\int_{\tau=0}^{t}\boldsymbol{X}^{stoch}_{t-\tau} \tau^{-1/2}e^{-\alpha \tau} d\tau}. \label{eq:stochX}
\end{equation}
This estimate of the input, used with Eq. \ref{eq:F3_dynamics}, enables the estimation of the optimal parameters, when inputs are stochastic measurements.

\subsection{Equivalence to Short-Term Hebbian plasticity}
\label{sec:neuro_exp} The cosine similarity between the input vector and each centroid's parameters underpins both the exponential and the linear
model. This similarity is precisely computed by a linear neuron that receives normalized inputs
$\boldsymbol{X}_t^*\coloneqq\frac{\boldsymbol{X}_t}{||\boldsymbol{X}_t||}$, and that normalizes its vector of synaptic efficacies:
$\boldsymbol{G}^{*(k)}_t\coloneqq\frac{\boldsymbol{G}_t}{||\boldsymbol{G}_t||}$.The cosine similarity is then the neuron's summed weighted input:
$u^{(k)}_t= \boldsymbol{G}^{*(k)}_t\cdot\boldsymbol{X}_t^*$ . Then, a set of $K$ such neurons in a soft-max configuration, each with a bias term
$G^{(0k)}_{t}$, computes the functions $Q_t^{(k)}$ of the exponential model (Section \ref{sec:model}, Eq. \ref{eq:Q_model}). The soft-max-configured
neurons implement and optimize the Bayesian generative model and the equivalent clustering model, if their parameters are set according to Eq.
\ref{eq:F3_dynamics}. The function $Q_t^{(k)}$ of the linear model (Section \ref{sec:model_lin}, Eq. \ref{eq:Q_model_} can also be computed neurally.
In particular, a rectified linear unit (ReLU) with normalized inputs and synaptic efficacies can compute the function
$\left(u^{(k)}_t+G^{(0k)}\right)^+=\max\left(u^{(k)}_t+G^{(0k)},\,0\right)$ and thus a set of $k$ such ReLU units can compute the ratio-based
function of Eq. \ref{eq:Q_model_}.

Furthering the neural analogy, Eq. \ref{eq:F3_dynamics} and Eq. \ref{eq:F3_dynamics_cos} show that, in both models, neuron $k$ has a resting weight
vector $\boldsymbol{W}^{(k)}$. It also shows that for a pair of input $\boldsymbol{X}_{t_i}$ and subsequent neuronal output $Q_{t_i}(C^{(k)})$, each
element of the parameter vector, i.e. the overall efficacy $G^{(jk)}=W^{(jk)}+F^{(jk)}$ of the j-th synapse, changes by a quantity proportional to
both the pre-synaptic input and the immediately subsequent post-synaptic neuron's output (also summarized in Eq. \ref{eq:F_sup}). This is a Hebbian
update similar to an STDP rule with a very short window for the timing dependence. The dynamics of the bias term for both models (Eq. \ref{eq:thr},
and, respectively, Eq. \ref{eq:thr_cos}) are a form of Hebbian short-term sensitization of the neuron with each input, that also depends on the
output $Q_{t_i}(C^{(k)})$.

If inputs are stochastic, e.g. rate-coded spike trains, using Eq. \ref{eq:stochX} with Eq. \ref{eq:F3_dynamics} or Eq. \ref{eq:F3_dynamics_cos},
shows that a synaptic efficacy increment occurs for each pair of pre- and post-synaptic activity (Fig. \ref{fig:data_model}\textbf{d}-\textbf{f}), with a timing
dependence (Fig. \ref{fig:clustering}\textbf{b}) such that the increment becomes smaller for longer time intervals $\tau$ elapsed between pre- and post-,
specifically decaying with $v(\tau)=\tau^{-1/2}e^{-\alpha \tau}$. Such a time-dependent change of the synaptic efficacy is STDP, in a generalized
form that can involve non-spiking, i.e. analog pre-synaptic input $X_{t_i-\tau}^{(j)}$ of synapse $j$ and analog post-synaptic activation $Q^{(k)}$
of neuron $k$. Compared to standard STDP that is based on spikes, the effect of the non-binary pre- and post-synaptic activity pair is the rescaling
of the synaptic update not only by the temporal distance within the pair, but also by the analog values $X_{t_i-\tau}^{(j)}$ (or
$X^{stoch(j)}_{t_i-\tau}$) and $Q^{(k)}(t_i)$, as is described in Eq. \ref{eq:F3_dynamics} and Eq. \ref{eq:F3_dynamics_cos} for the exponential and
the linear model respectively. In addition, the same Eq. shows that the synaptic update is transient (Fig. \ref{fig:clustering}\textbf{c}, and Fig.
\ref{fig:data_model}\textbf{g}-\textbf{h}) and decays exponentially with a time constant of $\frac{1}{\beta}$. This, therefore, is a case of STP, and combined with the
STDP effect, it is a case of ST-STDP, in a generalized form that can involve non-spiking inputs and non-spiking neuronal activations.
\subsection{Implementation with spiking neurons}
\label{sec:snn_exp} In this section we describe how, in addition to spiking \textit{input} (Section \ref{sec:stoch}), spiking \textit{output} can
also be used, to derive a stochastic approximation of the optimal model. To achieve this, we arrange spiking neurons in a winner-take-all (WTA) setup
(Fig. \ref{fig:wta}\textbf{b}) --- a powerful SNN architecture \cite{neftci2013PNAS,diehl2015FCN}. Such an SNN, when equipped with STDP, can approximate EM
\cite{nessler2013PLoS}. Thus, regular, i.e. long-term, STDP can be used here as well to learn the initial resting weights $W^{(jk)}$. After the
training, during the testing phase, short-term dynamics are added to the synapses, enabling ST-STDP. Making arguments similar to those in
\cite{nessler2013PLoS}, we will show that in this set-up too, the network does compute the necessary probabilities, here in order to implement the
Elastic Clustering, in the specific form that maintains an optimized Bayesian generative model. We use a stochastic model, in which the input firing
rate is proportional to $\boldsymbol{X}_t$'s analog value, and the output neurons' firing probability is a function of the membrane potential
$u^{(k)}$. Each input spike arriving at synapse $j$ causes an excitatory post-synaptic potential modelled as a step function of amplitude equal to
the synaptic efficacy $G^{(jk)}$ with a short duration. WTA competition between the neurons ensures that no more than one of the output neurons can
fire at a time. This competition is mediated by a lateral inhibition term $I(t)$ that is common to all neurons.

\paragraph{Stochastic exponential spiking neurons}To implement the spiking analogue to the exponential model, for normalized input
$\boldsymbol{X}_t^*$, the total excitation of a neuron k is $u^{(k)}_t= \boldsymbol{G}^{*(k)}\cdot\boldsymbol{X}_t^*+G^{(0k)}-I(t)$, where $G^{(0k)}$ is the neuron's intrinsic excitability. The output neurons are exponential, i.e. each neuron's spiking behaviour is modelled as an inhomogeneous Poisson process with a firing rate
\begin{equation}
	r^{(k)}=e^{u^{(k)}}=\frac{e^{ \boldsymbol{G}^{*(k)}\cdot\boldsymbol{X}_t^*+G^{(0k)}}}{e^{I}}=Q(C^{(k)}_t;G^{(0k)})\frac{q^{(k)}}{e^{I}},
\end{equation}
which uses the model's Definition \ref{def:model}. The combined firing output of the $K$ neurons is a Poisson process with rate $r_{all}=\sum_{l=1}^{K}r^{(l)}$. Therefore, if at one time instance the network produces an output spike, the conditional probability that this spike originated from neuron $k$ is
\begin{equation}\frac{r^{(k)}}{r_{all}}=\frac{e^{ \boldsymbol{G}^{*(k)}\cdot\boldsymbol{X}_t^*+G^{(0k)}}}{\sum_{l=1}^{K}e^{ \boldsymbol{G}^{*(l)}\cdot\boldsymbol{X}_t^*+G^{(0l)}}},
\end{equation}
which is independent of the inhibition term. Importantly, this is in fact exactly $Q_t(C^{(k)})$ as given in Eq. \ref{eq:Q_model}. Thus, each output
spike is a sample from the $Q^{(k)}$ distribution, so that the emission or not of a spike by neuron $k$ at time $t$ can be used in Eq.
\ref{eq:F3_dynamics} as an instantaneous stochastic estimate of $P_t(C^{(k)})$.
\paragraph{Leaky integrate-and-fire neurons}
The rectified linear model can also be extended to spiking outputs, as the exponential model was extended in Paragraph \ref{sec:snn_exp}. Here too,
we use a stochastic model in which the input firing rate is proportional to the input's analog value, but here the output neurons' firing probability
is a ReLU function of the membrane potential $u^{(k)}$. Here too the stochastic spiking ReLU neurons are configured in a WTA connectivity, such that
competition between the neurons ensures that at each time instance no more than one of the output neurons can fire, and this competition too is
mediated by an inhibition term that is common to all neurons. However, in this case inhibition is is divisive, so that the membrane potential of a
neuron $k$ is \begin{equation}u^{(k)}(t)=\frac{ \boldsymbol{G}^{*(k)}\cdot\boldsymbol{X}^*_t+G^{(0k)}}{I(t)}.
\end{equation} This is in accordance with a multitude of studies that
have described and modelled evidence for divisive effects of inhibition on neuronal gains \cite{mitchell2003Neuron,ayaz2009JN,wilson2012Nature}. We
model each output neuron's spiking behaviour as an inhomogeneous Poisson process with a firing rate
\begin{equation}
	R^{(k)}=ReLU(u^{(k)})=\frac{ReLU( \boldsymbol{G}^{*(k)}\cdot\boldsymbol{X}^*_t+G^{(0k)})}{I(t)}.
\end{equation} The combined firing output of the K neurons is a Poisson process with rate $R_{all}=\sum_{l=1}^{K}r^{(l)}$. Therefore, if at some time instance $t$ the network produces an output spike, the conditional probability that this spike originated from neuron $k$ can be expressed as
\begin{equation}\frac{R^{(k)}}{R_{all}}=\frac{ReLU( \boldsymbol{G}^{*(k)}\cdot\boldsymbol{X}^*_t+G^{(0k)})}{\sum_{l=1}^{K}ReLU( \boldsymbol{G}^{*(l)}\cdot\boldsymbol{X}^*_t+G^{(0l)})},
\end{equation}
which is independent of the inhibition term. This is in fact exactly $Q_t(C^{(k)})$ as given in Eq. \ref{eq:Q_model_cos}. Thus, each output spike is
a sample from the $Q^{(k)}$ distribution, so that the emission or not of a spike by neuron $k$ at time $t$ can be used in Eq.
\ref{eq:F3_dynamics_cos} as the value of $Q_t(C^{(k)})$ i.e. as an instantaneous stochastic inference of $P_t(C^{(k)})$ in Eq.
\ref{eq:gfw_dynamics_cos}.

Notably, a LIF spiking neuron's firing rate is linearly dependent on the weighted input for inputs that well surpass the firing threshold
\cite{amit1991Network}, and it is zero below threshold. Assuming stochastic inputs, a LIF neuron's output is stochastic, with a probability of
producing a spike within an infinitesimal time window proportional to its firing rate. Therefore, the stochastic spiking ReLU neuronal model is a
close approximation of a LIF neuron with noisy inputs, which, in turn, is a convenient model commonly used for simulations or emulations of spiking
neurons. This makes our model efficiently testable in practice using standard LIF neurons in simulations.

\paragraph{Short-term STDP}
The spiking outputs combined with the spiking inputs render the plasticity in the spiking realization of the model indeed an STDP rule, with an
additional short-term temporal dependence. The updates are event-based, such that at every post-synaptic spike, the synapses are updated according to
the spikes they received previously, the STDP kernel $v(\tau)=\tau^{-1/2}e^{-\alpha \tau}$, and the time $\tau$ that mediated between the pre- and
post-synaptic spikes (Eq. \ref{eq:F_sup} and Eq. \ref{eq:stochX}). The un-normalized synaptic efficacy $G^{(jk)}$ of synapse $j$ subsequently decays
towards the fixed weight $W^{(jk)}$ according to the short-term plasticity dynamics $f(t)$.

\section{Weight-dependent ST-STDP}
\label{sec:weightdependent} The uninformed prior belief about an object's transformation has characteristics of a Wiener process with no drift and implies a
normal distribution $\mathcal{N}\left(\boldsymbol{X}_{t_i},\boldsymbol{\Sigma}_{t_i,t}\right)$, centred around the most recent observation $\boldsymbol{X}_{t_i}$ of the object. In
a more general formulation of the model, Assumption \ref{as:wiener} from Section \ref{sec:assumptions} may include some prior knowledge about the
temporal evolution of an object. An informed prior would instead lead to a distribution centred around a mean point
$\boldsymbol{X}_{t_i}'=g(\boldsymbol{X}_{t_i}| \boldsymbol{\theta})$, according to a function $g$ of the input, parametrized by $\theta$.

For example, it may be $\boldsymbol{X}_{t_i}'=g(\boldsymbol{X}_{t_i})=\boldsymbol{\theta}\circ \boldsymbol{X}_{t_i}, \quad \boldsymbol{\theta},\boldsymbol{X}_{t_i}\in \mathbb{R}$, where $\circ$ symbolizes the element-wise product, i.e. where each dimension $X_{t_i}^{(j)}$ of $\boldsymbol{X}_{t_i}$ is weighted by a different coefficient $\theta^{(j)}$.

The weighted influence of each input feature on future samples, as represented by $\boldsymbol{\theta}$, is determined by a prior likelihood
distribution of the input feature, and this prior may be according to the history-independent distribution $p(\boldsymbol{X}_t|B)$ of Eq. \ref{eq:pstar}.
This describes situations where a feature of an observation of the $k$-th class is more likely to be repeated if it is generally a likely feature of
observations of that $k$-th class, than if it is not, and this likelihood is reflected in the distribution $p(\boldsymbol{X}_t|C^{(k)},B)$. In this
case, we can formulate this aspect of the data as
\begin{equation}
	\boldsymbol{X}_{t_i}'^{(k)}=g(\boldsymbol{X}_{t_i}, \boldsymbol{\theta}^{(k)})=\boldsymbol{\theta}^{(k)}\circ \boldsymbol{X}_{t_i}, \label{eq:informedprior}
\end{equation}
where
\begin{equation}
	\theta^{(jk)}=c+(1-c)W^{(jk)},\, j=1,2,...,n \label{eq:w_dependence}
\end{equation}
and $c<1$.

$\boldsymbol{X}_{t_i}'^{(k)}$ is then the datapoint that is used in place of $\boldsymbol{X}_{t_i}$ in the expressions of the parameter dynamics, and the one
that attracts the $k$-th centroid, so that the Hebbian update of Eq. \ref{eq:F_sup} becomes
\begin{gather}
	\boldsymbol{F}^{(k)}_t\leftarrow
	\boldsymbol{F}^{(k)}_t+
	\gamma\boldsymbol{X}_t'^{(k)}
	Q^{(k)}_t= \boldsymbol{F}^{(k)}_t+
	\gamma\boldsymbol{\theta}^{(k)}\circ \boldsymbol{X}_{t_i}
	Q^{(k)}_t\implies
\end{gather}
\begin{gather}
	F^{(jk)}_t\leftarrow
	F^{(jk)}_t+
	\gamma^{(jk)}\boldsymbol{X}_{t_i}^{(k)}
	Q^{(k)}_t,
\end{gather}
where $\gamma^{(jk)}=\gamma\theta^{(jk)}$
That is, the effect of incorporating such a prior knowledge is simply to let each synapse $j$ have a different strength of the ST-STDP effect, equating $\gamma\theta^{(jk)}$, which is dependent on the fixed weight $W^{(jk)}$, i.e. the effect is to render ST-STDP weight-dependent.
In data with transformations where novel features are as likely to be repeated, as typical features of an object type $k$, which would be the case if the objects' morphing is not biased by any prior distribution, then in Eq. \ref{eq:w_dependence}, $c=1$, and thus $\theta^{(jk)}=c$ which is equivalent to the original uninformed prior with a mean equal to the last observation of the object, i.e. from Eq. \ref{eq:informedprior}, $c=1\implies\boldsymbol{X}_{t_i}'^{(k)}=\boldsymbol{X}_{t_i}$, and there is no weight-dependence.

Eq. \ref{eq:informedprior} and Eq. \ref{eq:w_dependence} are the forms that we used in the model of elastic clustering through ST-STDP on the OMNIST
video, and they imply a simple weight-dependence involved in the short-term efficacy updates: a synapse with a stronger fixed (i.e. long-term) weight
$W^{(jk)}$ is more strongly potentiated by the Hebbian (short-term) updates of the ST-STDP rule.

\section{Simulation differences from theory}
\label{sec:sim_diffs} To test the model, we chose to simulate it by using LIF neurons, which are a convenient model suitable for simulations, or
implementations with efficient electronic circuits. The theoretically optimal model requires normalized inputs, continuous normalization of the
synaptic efficacies, an additional inherent neuronal plasticity continuously updating the neuronal threshold (Eq. \ref{eq:thr_cos}), and divisive
inhibition in the case of LIF neurons, as opposed to subtractive. Nevertheless, our simulations used a simplified model without these particular
idealities and demonstrated that the model can be robust to their absence. Omitting the neuronal memory mechanism that the intrinsic neuronal
plasticity would introduce, not only simplified the simulation, but also allowed us to perform a more direct contrast of ST-STDP with other models
such as RNNs or LSTMs which do include an analogous decaying neuronal memory, and therefore allows us to attribute the demonstrated benefits of the
simulated model to a plasticity mechanism that is specific to synapses.

\section{ST-STDP temporarily reshapes the internal representation of neurons}
\label{sec:reshapes}
\begin{figure}[h!]
	\centering
	\begin{tabular}{c}
		\includegraphics[width = 89mm]{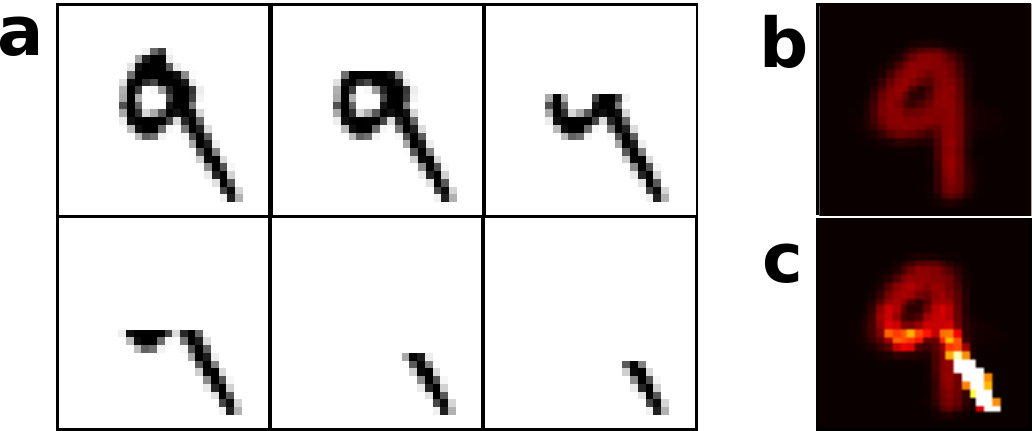}
	\end{tabular}
	\caption{\textbf{Example of temporary expansion of a neuron's internal representation.} \textbf{a}, A sequence involving a handwritten digit "9" being progressively occluded. \textbf{b}, The weights of a neuron that has been trained to specialize in a handwritten form of digit 9. It can be seen that the learned pattern is different from the input digit (first frame of panel A). \textbf{c}, The efficacies of the same neuron during the last frame of A. Through ST-STDP, the neuron adapts its representation of digit 9 to match the different form of digit 9, and to recognize even the highly occluded last frame despite its very small overlap with the original learned weights. This expansion is short-term and, accoding to ST-STDP, the neuron's efficacy vector will relax to its original form after the specific input digit is replaced by a different one, or after the neuron stops recognizing it, i.e. stops firing in response to it.} \label{fig:weirdtail}
\end{figure}
Equipping neurons with ST-STDP allows them to temporarily change their internal representation, to learn the currently most accurate representation
of the category they are already representing. This lets a neuron temporarily strengthen the most relevant of its permanent features it already has
stored in its previously learned long-term weights, but it also allows the neuron to acquire new features that a current version of an object may
have, but the neuron has not previously learned. An example of this is shown in (Fig. \ref{fig:weirdtail}), where a neuron temporarily expands its
internal representation (Fig. \ref{fig:weirdtail}\textbf{c}) to be able to recognize a digit (digit 9) through a feature ('tail' of digit, last frame in Fig.
\ref{fig:weirdtail}\textbf{a}) that is not included in its long-term weights (Fig. \ref{fig:weirdtail}\textbf{b} vs \textbf{c}). This possibility can be exploited to apply
ST-STDP to other types of transformations than occlusions.

These become possible by virtue of the synapse-specific nature of the short-term memory in the case of ST-STDP, as opposed to a neuron-specific memory, as is implemented by the recurrency in the case of RNN and LSTM. 

\onecolumn
\section{ST-STDP realizes elastic clustering}
\begin{figure*}[h!]
	\centering
	\includegraphics[width=155mm]{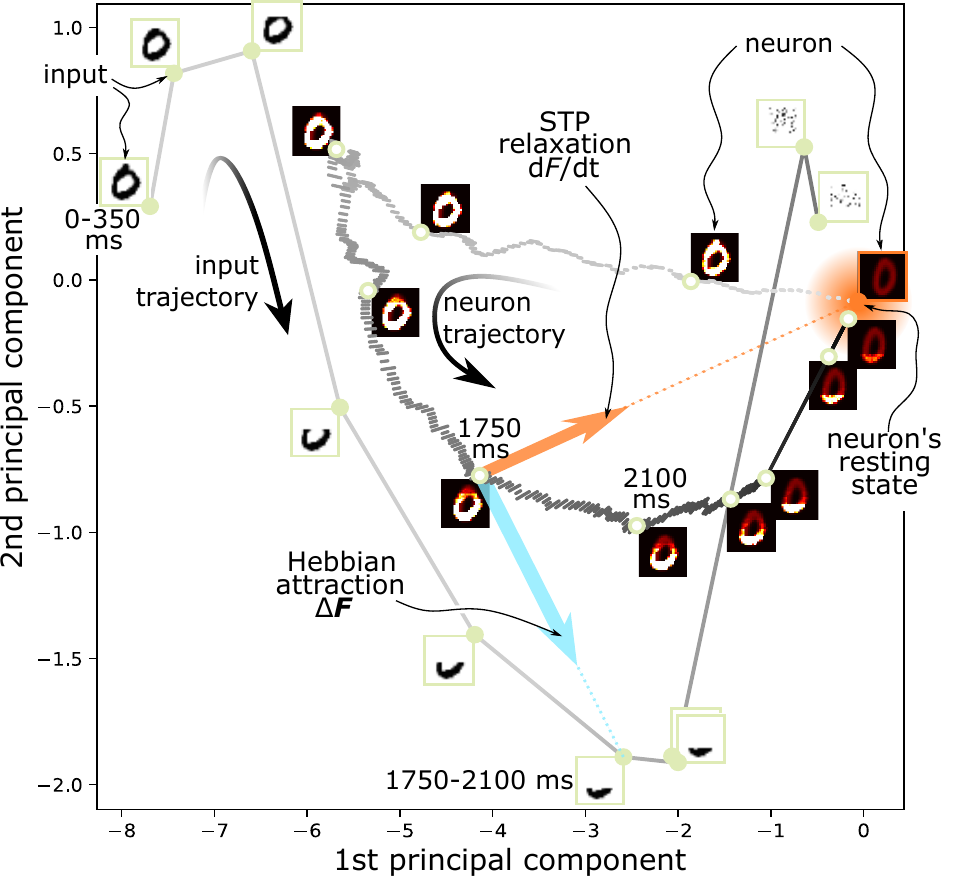}
	\caption{\textbf{Experimental evidence that ST-STDP realizes elastic clustering.}  (cf. Fig. \ref{fig:clustering}\textbf{a} and Fig. \ref{fig:wta}.) The trajectory of a neuron's efficacy vector on the plane of the first two principal components, as it follows the input sequence and relaxes to its resting position, demonstrating qualitatively that the elastic clustering algorithm is indeed implemented by the simulated SNN. The input and neuron here are those of Fig. \ref{fig:wta}\textbf{a} and \textbf{d}. Green-filled circles beginning from the datapoint on the left-hand side (arrow labelled "input trajectory") correspond to the inputs of Fig. \ref{fig:wta}\textbf{a}. The neuron's initial and resting state is the orange highlighted point (right-hand side). The thick grey-scale sequence of points depicts the neuron's path (arrow labelled "neuron trajectory"). White-filled green circles correspond to Fig. \ref{fig:wta}\textbf{d}'s snapshots. Grey points are the in-between states, plotted every \unit[1]{ms}. Each time the neuron fires, it is attracted by the recent input (cyan arrow) through Hebbian plasticity, and thus follows the input's path. The STP (orange arrow) continuously tends to relax the neuron towards its resting state, which ultimately happens as the neuron stops firing (straight last segment of the neuron's trajectory).} \label{fig:pca}
\end{figure*}

\onecolumn
\section{Application of elastic clustering to other types of transformations}
\label{sec:head}
OMNIST provides a demonstrative example of changing transformations that can be encountered in temporal data, and which can be
tackled successfully by our elastic clustering and its implementation with spiking neurons. Our theoretical analysis shows that the strength of the
approach is quite more general. As an example, it can be applied to the continuous recognition of patterns that morph with time, or whose observation
changes because the sensor changes its properties with time. Here (Fig. \ref{fig:rotatinghead}) we show a proof-of-concept toy example where a
neuron has been trained to recognize the side view of a person's face, but has never encountered its frontal view. As the head rotates though, the
neuron, equipped with ST-STDP, expands its internal representation to match the rotating head, and to ultimately still recognize the frontal view,
despite its lack of prior training on this pattern. This is a simple form of transfer learning, where the neuron used its previous knowledge of one
view of the person, to learn a different one, by exploiting the common features between temporally contiguous input samples.

\begin{figure*}[h!]
	\centering
	\includegraphics[width=160mm]{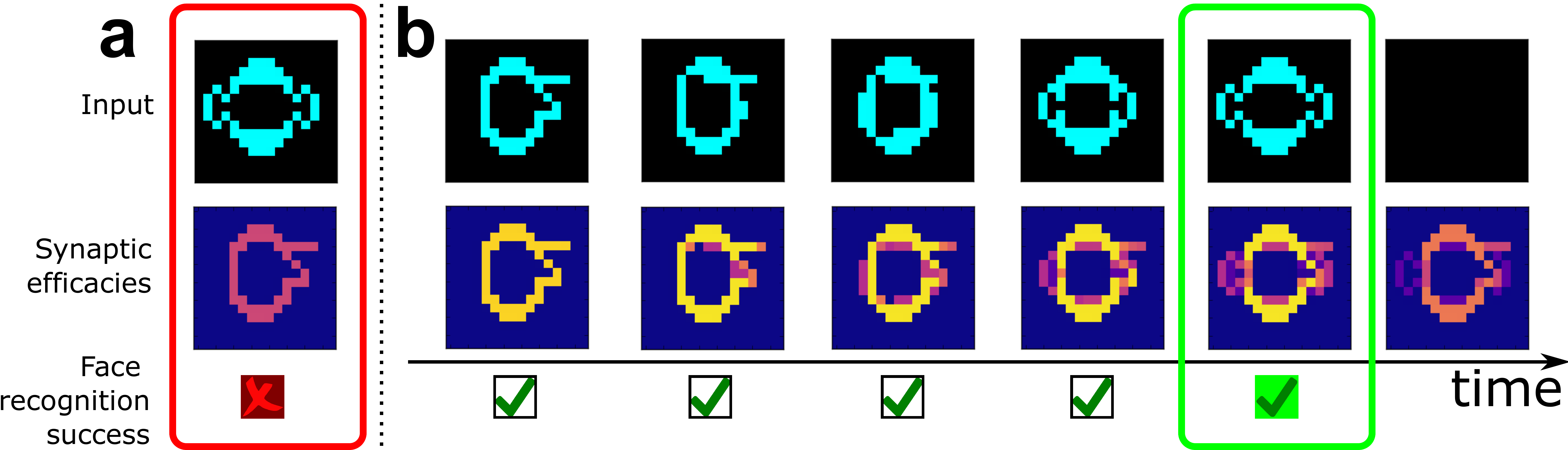}
	\caption{\textbf{Concept demonstration results of ST-STDP applied to morphing patterns.} \textbf{a}, One neuron has been pre-trained to recognize the
		side view of a person's face (middle), but the front view (top) can't be recognized (bottom) when it is presented out of context. \textbf{b}, Use of
		the temporal context through STP enables recognition. The side view is successfully recognized, and as the person rotates his head in subsequent
		frames, the face continues being recognized, even at its front view (green rectangle) that could not be recognized in panel A. This is possible
		thanks to the ST-STDP type of STP at the synapses. ST-STDP facilitates synapses whose inputs contribute to a successful face recognition (middle
		row). The facilitation is short-term, and the synapses that don't contribute to the postsynaptic neuron's activation transiently recover to their
		relaxed state (last frame).} \label{fig:rotatinghead}
\end{figure*}

\twocolumn
\section{Short-term memory in ST-STDP vs RNN and LSTM}
\label{sec:LSTM}
The SNN's
advantage compared to the RNN and LSTM in the OMNIST task stems partly from the fact that its short-term memory is independent for each synapse (see Supplementary Information
\ref{sec:reshapes}), as opposed to the neuronal-level memory implemented by the recurrency in RNN and LSTM networks.
This feat enables the spiking neurons to keep in short-term memory recent specific features of digits (as opposed to a whole digit memorized by the neuron's soma), so that the subsequent noise cannot be mis-recognized as a digit, unless it closely matches these specific features, i.e. pixels, of the prior digit's recent form. This gives the SNN a very high accuracy on noisy frames (Fig. \ref{fig:ANNs}\textbf{b}, right-most data-point). 
This requires a short time constant for the short-term component of the synaptic efficacies, to maintain a memory of only the most recent frame's visible features. This short time constant however causes some additional errors on highly occluded frames (Fig. \ref{fig:ANNs}\textbf{b}, 80-100\% occlusion). As the noise frames are
frequent, and the SNN’s performance in highly occluded frames remains reasonable, the SNN
reaches a better performance than the ANNs in the overall dataset.

The LSTM's and the RNN's better accuracy in highly occluded frames
is due to their choice of a longer time constant. A longer time constant is of course possible for
the SNN too, if necessary. Therefore, if the highly occluded frames were the priority, the SNN
does have the option to prioritize those rather than the noisy frames, by setting its time constant
to a longer value, as expected, and as we did observe while tuning the SNN’s parameters. The
alternative option, i.e. trading off some accuracy in highly occluded frames to further increase
overall performance as in the SNN, is not available to the ANNs, because a synapse-specific
short-term memory mechanism is not present. Their performance on the overall dataset is already optimized
by backpropagation through time.

Note that the theoretically optimal SNN also includes an additional short-term memory at the neuron's soma, implemented as an adaptive excitability of the neuron. We did not include this in our simulations, in order to focus on the advantages of the synaptic short-term memory which is the unique mechanism here, compared to the ANNs. Including the neuronal-level short-term memory mechanism as well may further improve the SNN's accuracy.

In addition, the short-term memory parameters, i.e. the recurrent weights, of the RNN and LSTM units were optimized through supervised training, and were flexible enough to be different for each unit, while the SNN's temporal parameter of STP was only hand-tuned, and it was a single one shared over the whole network's synapses. An optimization procedure for the SNN too, and an independent temporal parameter for each synapse is expected to further increase the SNN's accuracy. 

In summary, the key advantage of the SNN compared to LSTM does not lie in the choice of its time constant parameter, but in the qualitative difference of the model, which enables multiple independent short-term memories per unit, where each memory corresponds to one synapse.

\end{document}